\pdfoutput=1

\documentclass[11pt]{article}

\usepackage[preprint]{acl}

\usepackage{times}
\usepackage{latexsym}

\usepackage[T1]{fontenc}

\usepackage[utf8]{inputenc}

\usepackage{microtype}

\usepackage{inconsolata}

\usepackage{graphicx}

\title{Multi-Drafter Speculative Decoding with Alignment Feedback}

\author{
  Taehyeon~Kim\textsuperscript{\dag,}\Thanks{Equal contribution.}\quad
  Hojung~Jung\textsuperscript{\ddag,}$^*$\quad
  Se-Young~Yun\textsuperscript{\ddag}\\[6pt]
  \textsuperscript{\dag}LG AI Research\quad
  \textsuperscript{\ddag}KAIST AI\quad
}

\usepackage[ruled,vlined]{algorithm2e}
\usepackage{algorithmic}

\usepackage{hyperref}
\usepackage{url}
\usepackage{amsmath, amssymb, amsthm}
\usepackage{amsfonts}
\usepackage{booktabs}
\usepackage{nicefrac}
\usepackage[dvipsnames,svgnames,table]{xcolor}
\usepackage{colortbl}
\usepackage[utf8]{inputenc} 
\usepackage[T1]{fontenc}    
\usepackage{afterpage}
\usepackage{multirow}

\usepackage{booktabs}
\usepackage{ulem}

\usepackage{graphicx}
\usepackage{wasysym}

\usepackage{subcaption}
\usepackage{multirow}
\usepackage{wrapfig}

\usepackage{amsmath,amsfonts,bm}

\def\eqref#1{equation~\ref{#1}}

\def\1{\bm{1}}

\DeclareMathAlphabet{\mathsfit}{\encodingdefault}{\sfdefault}{m}{sl}
\SetMathAlphabet{\mathsfit}{bold}{\encodingdefault}{\sfdefault}{bx}{n}

\DeclareMathOperator*{\argmax}{arg\,max}
\DeclareMathOperator*{\argmin}{arg\,min}

\definecolor{myred}{HTML}{B22222}
\hypersetup{colorlinks, 
            citecolor=myred}

\newcommand{\Autoref}[1]{%
  \begingroup%
  \def\algorithmautorefname{Algorithm}%
  \def\chapterautorefname{Chapter}%
  \def\sectionautorefname{Section}%
  \def\subsectionautorefname{Section}%
  \def\theoremautorefname{Theorem}%
  \def\assumptionautorefname{Assumption}%
  \def\equationautorefname{eq.}%
  \def\definitionautorefname{Definition}%
  \def\lemmaautorefname{Lemma}%
  \def\corollaryautorefname{Corollary}%
  \autoref{#1}%
  \endgroup%
}

\newcommand{\CAutoref}[1]{%
  \begingroup%
  \def\sectionautorefname{Section}%
  \def\subsectionautorefname{Section}%
  \autoref{#1}%
  \endgroup%
}

\newcommand{\Tableautoref}[1]{%
  \begingroup%
  \def\figureautorefname{Table}%
  \autoref{#1}%
  \endgroup%
}

\newcommand{\Algautoref}[1]{%
  \begingroup%
  \def\figureautorefname{Algorithm}%
  \autoref{#1}%
  \endgroup%
}

\usepackage{color}

\usepackage[symbol*]{footmisc}
\usepackage[capitalize,noabbrev]{cleveref}
\usepackage{enumitem}

\theoremstyle{plain}
\newtheorem{assumption}{Assumption}

\newtheorem{theorem}{Theorem}
\newtheorem{lemma}{Lemma}
\newtheorem{corollary}{Corollary}
\newtheorem{definition}{Definition}
\theoremstyle{remark}

\usepackage{bbding}
\usepackage{booktabs}
\usepackage[dvipsnames]{xcolor}

\newcommand{\cmark}{\textcolor{ForestGreen}{\Checkmark}} 
\newcommand{\xmark}{\textcolor{BrickRed}{\XSolidBrush}}

\usepackage{makecell}

\begin{document}
\maketitle

\begin{abstract}
Speculative decoding (SD) accelerates large language model (LLM) inference by using a smaller model to draft future tokens, which are then verified by the target LLM. This preserves generation quality by accepting only aligned tokens. However, individual drafters, often trained for specific tasks or domains, exhibit limited effectiveness across diverse applications. To address this, we introduce \textsc{MetaSD}, a unified framework that integrates multiple drafters into the SD process. \textsc{MetaSD} dynamically allocates computational resources to heterogeneous drafters by leveraging alignment feedback and framing drafter selection as a multi-armed bandit problem. Extensive experiments show \textsc{MetaSD} consistently outperforms single-drafter approaches.
\end{abstract}
\section{Introduction}

Large language models (LLMs) like GPT-4 \citep{gpt4}, Gemini \citep{gemini}, and Llama \citep{llama} have significantly advanced applications such as search \citep{reid2024gemini}, coding assistance, and virtual assistants. However, their token-by-token generation process often results in substantial inference times, primarily due to memory bandwidth limitations \citep{memorybound, memorybound2}. Speculative decoding (SD) has emerged as a key technique to mitigate this, employing a smaller `drafter' model to predict future tokens for parallel verification by the target LLM \citep{speculative_decode, speculative_decode2}. This reduces latency by accepting only tokens aligned with the target model's predictions, thereby preserving output quality.

Existing SD methods have improved draft acceptance rates via architectural and training innovations \citep{onspec, zhou2023distillspec, medusa, specinfer, spectr}, including multi-path exploration through batched inference or tree verification \citep{spectr, specinfer, medusa} and better drafter-target alignment via knowledge distillation \citep{zhou2023distillspec, onspec}. However, these methods predominantly rely on a single drafter, a design choice that curtails versatility. This limitation has become particularly significant with the growing trend towards deploying ensembles of specialized models, a paradigm actively explored in LLM routing~\citep{hu2024routerbench,jitkrittum2025universal}. Consequently, SD with a single drafter can fail on out-of-distribution inputs or dynamic user queries where its inherent biases become a bottleneck~\citep{onspec, yi2024towards}. 

\begin{table}[t]

    \centering

    \caption{Comparison of \textsc{MetaSD} with existing single- and multi-drafter SD approaches. \textsc{MetaSD} is training-free without extra compute for drafting and verification cost than standard, offers theoretical regret guarantees and Robustness to Non-Stationarity (RNS).}

    \label{tab:comparison_methods}

    \resizebox{\linewidth}{!}{

    \begin{tabular}{lccccc}

        \toprule

        \textbf{Method} & \textbf{Multi-drafter} & \textbf{Training-free} & \textbf{No extra compute} & \textbf{Regret bound} & \textbf{RNS} \\

        \midrule

        Standard SD & \xmark & \cmark & --- & --- & --- \\

        Static Ensemble SD & \cmark & \cmark & \xmark & \xmark & \cmark \\

        Learned Routing & \cmark & \xmark & \cmark & \xmark & \xmark \\

        \midrule

        \textbf{\textsc{MetaSD} (Ours)} & \textbf{\cmark} & \textbf{\cmark} & \textbf{\cmark} & \textbf{\cmark} & \textbf{\cmark}\\

        \bottomrule

    \end{tabular}
    }
\end{table}




\begin{figure*}
    \centering
    \includegraphics[width=0.90\linewidth]{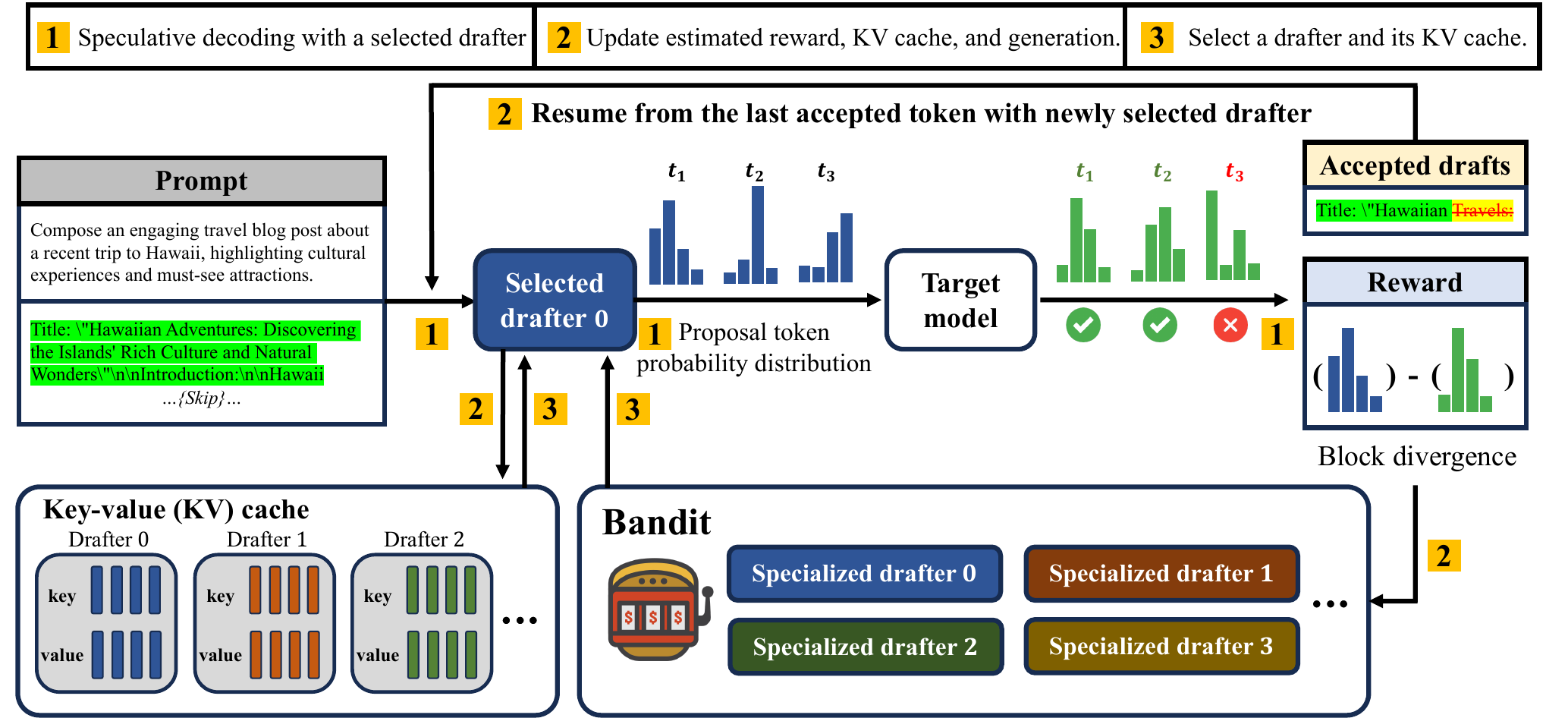}
    \caption{Overview of speculative decoding with multiple drafters in multi-armed bandit (MAB) framework. Drafters are selected and updated based on alignment feedback through the block divergence reward. The example in this figure is from an instance in MT-Bench dataset\,\citep{llmjudge}.}
    \vspace{-15pt}
    \label{fig:overview_fig}
\end{figure*}


The limitations of single-drafter systems necessitate an adaptive, multi-drafter framework capable of generalizing across diverse and evolving user queries. This is compounded by operational constraints, including the need for low computational overhead, scalability across variable user loads, and robustness to traffic fluctuations without manual intervention or pre-tuning. To meet these requirements, SD systems must incorporate dynamic adaptation mechanisms that can efficiently allocate computational resources while optimizing performance for diverse input conditions.

This paper addresses these challenges by introducing \textsc{MetaSD}, a novel framework that integrates multiple drafters into SD and dynamically meta-drafts the optimal drafter during inference. 
At the core of our framework is the concept of alignment feedback, which measures the compatibility between a drafter's predictions and the target LLM. 
Leveraging this feedback, \textsc{MetaSD} formulates drafter selection as an exploration-exploitation tradeoff \citep{gittins2011multi}, dynamically learning a policy that selects the best drafter based on the current observation at every state. 
Inspired by multi-armed bandit (MAB) algorithms widely used in recommendation systems to adaptively optimize decisions based on feedback \citep{silva2022multi}, our approach employs a similar mechanism to dynamically moderate between drafters. Unlike static configurations, \textsc{MetaSD} enables the system to adapt to both evolving input contexts and varying task distributions, achieving robust performance across scenarios (\Autoref{fig:overview_fig}). Our contributions are:

\begin{itemize}[leftmargin=12pt, itemsep=2pt]
    \item 
    \textbf{\textsc{MetaSD}}: A unified framework integrating multiple drafters into SD. MetaSD utilizes block divergence (BD) as a reward signal, derived from alignment feedback between drafter predictions and the target LLM, to enable dynamic drafter selection. It supports both black-box (independent drafters) and white-box (drafters using target LLM latent representations) configurations (\textbf{\Autoref{sec2}}).
    
    \item 
    \textbf{Theoretical analysis}: We establish regret upper bounds for \textsc{MetaSD}, offering insights into its convergence, scalability, and the role of BD-based alignment feedback in optimizing drafter selection (\textbf{\Autoref{sec3:method}}).
    
    \item 
    \textbf{Performance improvement}: Extensive experiments demonstrate \textsc{MetaSD} achieves superior inference speedups over existing single-drafter and static multi-drafter methods (\textbf{\Autoref{sec4:exp}}).
\end{itemize}

\section{Problem statement}  \label{sec2}

\subsection{Motivation} \label{sec:motivation}
Speculative decoding (SD) employs a \textit{draft-verify-accept} paradigm for faster inference. A drafter $\mathcal{M}_q$, which is smaller than the target LLM $\mathcal{M}_p$, drafts the future tokens $\{x^{l+1:l+N_{max}}\}$ based on the input sequence $x^{1:l}$. The target LLM assesses each token $x^{l+j}$ ($j=1,\dots,N_{max}$) to determine whether $p(\cdot | x^{1:l+j-1})$ is aligned with its own predictions $q(\cdot | x^{1:l+j-1})$. Only the tokens aligned with the LLM's own predictions are accepted, ensuring the lossless generation (detailed in \Autoref{sec:prelim}).
\vspace{-2pt}

Despite advancements, existing SD often employs a single drafter, leading to critical limitations in dynamic, multi-task settings. Task-specific drafters, though effective in-domain, generalize poorly to unseen tasks due to inherent training data biases \citep{yi2024towards, onspec}. SD's efficacy hinges on the alignment between the drafter \(\mathcal{M}_q\) and the target LLM \(\mathcal{M}_p\); strong alignment yields higher acceptance rates.
This alignment provides crucial real-time feedback for system adaptation. Consequently, addressing dynamic and multi-task challenges requires a framework that uses alignment feedback to adapt its drafter selection policy based on the current input context.
\vspace{-2pt}

Motivated by these limitations, our \textsc{MetaSD} framework continuously evaluates alignment feedback to identify the optimal drafter for the current context. We frame this as a decision-making problem under uncertainty, leveraging the MAB framework. MAB offers a principled approach to balance exploration (evaluating different drafters) and exploitation (using the historically best-performing drafter) (\Autoref{alg:generic_mab}).

\begin{algorithm}[t] 
\DontPrintSemicolon
    \caption{MetaSD}
    \label{alg:generic_mab}
    \begin{algorithmic}[1]
    \SetKwInOut{Input}{Input}
    \INPUT {\small : Drafters $[K]$, initial prompt  $x^{1:l}$, \\ target model $\mathcal{M}_p$, target sequence length $B$.}
    \\
    \STATE $t \leftarrow 0$
    \WHILE{$l < B$} 
        \STATE Meta-draft the drafter $i$ in drafter pool \\ $[K]$ following the bandit \\
        \STATE Execute one SD step with drafter $i$ and target model  given $x^{1:l}$\\
        \STATE Compute the BD between drafter $i$'s prediction and $\mathcal{M}_p$'s prediction as the reward (\Autoref{sec:2.3})\\
        \vspace{3pt}
        \STATE Update the sequence length with the number of accepted tokens from the draft \\
        \vspace{2pt}
        \small $N_{acc}(i,t)$: $l, t \leftarrow l + N_{acc}(i,t) +1, t + 1$ \\
        \STATE \normalsize Update the bandit \\
    \ENDWHILE
\end{algorithmic}
\end{algorithm}

\subsection{Problem formulation}

\paragraph{Multi-armed bandit}
MAB framework addresses an online learning scenario where, at each round $t$, an agent takes an action by choosing an arm $a_t\in[K]$ and receives a reward $r_t$ from the environment. The goal of MAB is to design an algorithm that maximizes the expectation of cumulative reward $\mathbb{E}[\sum_{t=1}^{T}r_t]$ throughout a total of $T$ rounds. To achieve this, one can aim to design an optimal policy $\pi^{\star}$ to minimize the pseudo-regret, defined as: $\textsc{Reg}(\pi,T)=\sum_{t=1}^{T}\mathbb{E}[r_{a_t^{\star}}]-\mathbb{E}[r_{a_t}]$. Here, $a_t$ denotes the action chosen in round $t$ by the policy $\pi$ and $a_t^{\star}$ represents the optimal action in round $t$ which yields the highest expected reward. For a more comprehensive review, we refer the reader to \citet{lattimore2020bandit}.

\paragraph{\textsc{MetaSD}: Multi-drafter SD}

We formalize multi-drafter integration in SD as a MAB problem, where each SD cycle (drafting, verifying, accepting tokens) constitutes one MAB round (\Autoref{alg:generic_mab}). At each round $t$, a drafter a $a_t$ is chosen from a pool of $K$ heterogeneous drafters, and an SD step is performed, yielding $N_{acc}(a_t,t)$ accepted tokens. 
The process ends once $B$ tokens are generated. Unlike standard MAB settings with a fixed number of total rounds $T$, \textsc{MetaSD}'s total rounds is stochastic, contingent on the chosen drafter's alignment efficiency and token acceptance rates. Interestingly, conventional regret minimization does not directly optimize SD efficiency (formal proof in \Autoref{appendix:N_acc and stopping time relationship}). To address this, we introduce a novel objective in the following.  

\begin{definition}[Stopping time regret]
\label{def:stopping_time_regret}
Denote $\tau(\pi,B)$ as the number of total rounds of bandit policy $\pi$ with  target sequence length $B$ and $\pi^{\star}$ as the optimal policy which satisfies $\pi^{\star}=\argmin_{\pi}\mathbb{E}[\tau(\pi,B)]$. Then, regret objective of MetaSD with policy $\pi$ becomes:
\vspace{-3pt}
\begin{equation}
\label{eq:Regret_Define_stoppingtime}
\textsc{Reg}(\pi,B) =\mathbb{E}\left[\tau(\pi,B)]-\mathbb{E}[\tau(\pi^{\star},B)\right].
\end{equation}
\end{definition}
\vspace{-3pt}
The following lemma establishes this new regret objective as a valid performance metric for SD.
\begin{lemma}
    Minimizing~ $\textsc{Reg}(\pi,B)$ is equivalent to maximizing expected number of accepted tokens, which aligns with the objective of SD.
    \vspace{-5pt}
\end{lemma}
\begin{proof}
$B$ is sum of the total rounds and accepted tokens,
$B = \tau(\pi,B) + \sum_{t=1}^{\tau(\pi,B)} N_{acc}(a_t,t)$.
\end{proof} 

\vspace{-7pt}
Consequently, minimizing the regret in~\Autoref{eq:Regret_Define_stoppingtime} directly maximizes the expected number of accepted tokens, aligning with the SD objective. Thus, the goal of MetaSD is to find a policy $\pi$ that minimizes the regret in \Autoref{def:stopping_time_regret}.

Note that our algorithm operates in a lossless framework; \textbf{therefore, our evaluation focuses exclusively on inference speedups and acceptance rates.}


\subsection{Alignment feedback as a reward} \label{sec:2.3}
 
To dynamically select the best-aligned drafter in \textsc{MetaSD} without prior knowledge, the MAB algorithm requires informative reward signals reflecting the alignment quality between a given drafter and the target model.
Intuitively, a well-aligned drafter will produce next-token probability distributions closely matching those of the target model for the same context. The following assumption formalizes this concept.

\begin{assumption}
\label{assumption:main_model_alignment}
Denoting $d_{TV}$ as total variation distance(TV distance), define $\alpha_{i,t}$ to be:
\begin{equation}
\begin{aligned}
\alpha_{i,t}=1-d_{TV}\left(p^{t}(\cdot|x^{1:t-1}),q_i^{t}(\cdot|x^{1:t-1})\right),
\end{aligned}
\end{equation}
Then, for any instance of $x^{1:B}$ generated by the target model, $\alpha_{i,t}$'s are drawn i.i.d. from a distribution $\nu_i$ with expectation $\alpha_i$. In other words, following holds for all $i\in[K]$.
\begin{equation}
\mathbb{E}[{\alpha_{i,t}}] = \alpha_i, \;\; \alpha_{i,t}\overset{\text{i.i.d.}}{\sim}  \nu_i.
\end{equation}
\end{assumption}
Although seemingly strong, above assumption generalizes prior work in the speculative decoding literature, which typically assumes a fixed TV distance at every decoding step~\citep{speculative_decode, li2024eagle}. 
Furthermore, we empirically demonstrate in \Autoref{app:assumption_on_model_alignment} that the alignment between the drafter and verifier shows stationary behavior in most practical scenarios.


\begin{definition}[Block divergence reward] 
\label{def:BD_reward}
Let $t$ be the current round, $i$ be the drafter index, and $l(t)$ be the number of input tokens for the target model at round $t$. Given context $x^{1:l(t)}$, denote $d_{TV}(p^{l(t)},q_i^{l(t)})=\frac{1}{2}\|p^{l(t)}-q_i^{l(t)}\|_1$ as the total variation (TV) distance between probability measures of target model ($p^{l(t)}$) and the drafter ($q^{l(t)}_i$). Then, we define the BD reward as: 
\begin{equation}
\label{eq:BDreward_definition}
\begin{aligned}
{r}_{i,t}^{BD} =
\frac{1}{N_{max}}\sum_{j=0}^{N_{max}-1}\left(1-d_{TV}\left(p^{l(t)+j},q_i^{l(t)+j}\right)\right).
\end{aligned}
\end{equation}
\end{definition}

Another straightforward option for feedback signal is the block efficiency (BE), which quantifies the number of mean accepted tokens until a given round \citep{spectr, speculative_decode2, kim2024exploring}. Formally, for drafter $i$ in round $t$, it is defined as: $r_{i,t}^{BE}:=N_{acc}(i,t)/{N_{max}}$, where $N_{max}$ is predefined maximum draft length and $N_{acc}(i,t)$ is number of accepted tokens. 

Both BE and BD rewards act as alignment feedback in MetaSD, as shown by the following lemma:
\begin{lemma}
\label{lemma:BE_BD_equation}
BE and BD rewards are related as:
\begin{equation}
      \mathbb{E}[r_{i,t}^{BE}] = \frac{1-\alpha_{i}^{N_{max}}}{N_{max}(1-\alpha_{i})}\mathbb{E}[r_{i,t}^{BD}].  
\end{equation}
Consequently, maximizing BD reward is equivalent to maximizing number of accepted tokens.
\end{lemma}
The proof is in~\Autoref{appendix:proof_of_theorem1}.
While both rewards are viable, we show that BD reward is superior in SD, both theoretically and empirically. The following theorem formalizes this comparison:

\begin{theorem}[Informal]
\label{theorem:BD,BE reward comparison}
Under the stationary environment, for any reward design $r_i$ with $\mu_i=\mathbb{E}[r_i]$, $i^{\star}=\argmax{\alpha_i}$, and $\Delta_i=\mu_i^{\star}-\mu_i$, we define the feedback signal for each suboptimal arm $i\neq i^{\star}$ as 
\begin{equation}
\label{eq:bandit_signal}
R(r_i):=\frac{\Delta_i^2}{\max({\mathrm{Var}[r_i],\mathrm{Var}[r_{i^{\star}}]})}.
\end{equation}
Then, $R(r_i^{BD}) > R(r_i^{BE})$ for most cases.
\end{theorem} 

\begin{table}[t]
\centering
\caption{Comparison of reward statistics on the Japanese queries with multilingual specialized drafters. Detailed descriptions are in \autoref{tab:motivation}.}
\label{tab:reward_stats}
\resizebox{0.98\linewidth}{!}{%
\begin{tabular}{c|cc|cc|c}
\toprule
\multirow{2}{*}{Drafter} & \multicolumn{2}{c|}{Mean reward} & \multicolumn{2}{c|}{variance} & Zero-reward ratio$^\dagger$ \\\cmidrule{2-6}
 & BE & BD & BE & BD & BE only \\ \midrule
Ja-drafter & 0.232 & 0.488 & 0.093 & 0.044 & 0.503 \\
Ru-drafter & 0.099 & 0.294 & 0.032 & 0.026 & 0.678 \\
De-drafter & 0.081 & 0.317 & 0.024 & 0.032 & 0.721 \\
Fr-drafter & 0.074 & 0.288 & 0.023 & 0.029 & 0.743 \\
Zh-drafter & 0.106 & 0.326 & 0.037 & 0.034 & 0.681 \\ \bottomrule
\end{tabular}%
}
\begin{flushleft}
\scriptsize$^\dagger$ Zero-reward ratio applies only to BE, as BD rewards are always nonzero.
\end{flushleft}
\vspace{-20pt}
\end{table}

\Autoref{theorem:BD,BE reward comparison} indicates that the BD reward offers a more informative feedback signal than BE. The feedback signal $R(r_i)$ (\Autoref{eq:bandit_signal}) is crucial for bandit algorithm performance: a stronger signal enables faster and more accurate identification of the optimal arm, thereby minimizing exploration of suboptimal arms. Consequently, employing the BD reward can lead to superior performance in bandit-based drafter selection. A formal statement of \Autoref{theorem:BD,BE reward comparison} with a statistical analysis of both rewards (\Autoref{lemma:BEreward_stats}, \Autoref{lemma:BD_stats_asymptotic}) is in \Autoref{appendix:proof_of_theorem1}.

We empirically validate \Autoref{theorem:BD,BE reward comparison} by analyzing the reward statistics gathered from the full speculative execution. In \Autoref{tab:reward_stats}, the BD reward exhibits significantly larger differences in mean rewards between the optimal and suboptimal drafters (larger $\Delta_i$), and consistently lower variance than BE. This provides a stronger feedback signal $R$, enabling more stable learning and faster convergence (details in \Autoref{app:reward_abl}). 

\paragraph{Concurrent works} \citet{hou2025banditspec} have leveraged the idea of MAB for SD, but different scope and contributions from ours:
first, while \citet{hou2025banditspec} modify the confidence range of the UCB algorithm to control the regret analysis estimating mean of number of accepted tokens, our theoretical analysis \textbf{can be applied to more general reward shaping scenarios}, resulting in stronger instance-dependent bound in \Autoref{theorem:regret_upperbound_SD-UCB}. Furthermore, we reveal that reward design is important and choice of BD reward improves the performance of SD both practically and theoretically (\Autoref{app:concurrent}).


\section{Method} \label{sec3:method}

This section presents MetaSD-UCB designed to optimize bandit-based drafter selection in MetaSD by leveraging alignment feedback via the BD reward.

\begin{algorithm}[t]
\DontPrintSemicolon  
\caption{MetaSD-UCB}  
\label{alg:ucb1}  
\begin{algorithmic}[1]  
\INPUT{\small: Drafter pool $[K]$, initial prompt $x^{1:l}$, \\ target sequence length $B$, exploration parameter $\beta$.}
\STATE $t \leftarrow 0$ \\
\FOR{$i \in [K]$}
    \STATE Do one round of SD with drafter $i$ and obtain \small{$N_{acc}(i,t)$, $r_{i,t}$} (by \Autoref{eq:BDreward_definition})
    {\small \STATE $\hat{\mu}_{i,t}, n_i, l,  t\leftarrow r_{i,t}, 1, l + N_{acc}(i,t) + 1,  t+1$} \\
\ENDFOR \\
\WHILE{$ l < B$} 
    \STATE $a_t \leftarrow \arg\max_{i \in [K]} \hat{\mu}_{i,t} + \beta \sqrt{\frac{2 \ln t}{n_i}}$ 

    \STATE Do one round of SD with drafter $a_t$ and obtain \small{$N_{acc}(a_t,t)$, $r_{a_t,t}$} (by \Autoref{eq:BDreward_definition})
    \STATE $\hat{\mu}_{a_t,t} \leftarrow \frac{\hat{\mu}_{a_t,t} * n_{a_t} + r_{a_t, t}}{n_{a_t}+1}$ \\
    $n_{a_t}, l, t \leftarrow n_{a_t} + 1, l + N_{acc}(a_t,t) + 1, t+1$
\ENDWHILE
\end{algorithmic}
\end{algorithm}
\vspace{-5pt}

\subsection{Algorithm}
\label{subsec:algorithm 3.1}
\paragraph{MetaSD-UCB}
MetaSD-UCB (\Algautoref{alg:ucb1}) adapts the Upper Confidence Bound (UCB) algorithm \citep{auer2002finite} for SD, leveraging the BD reward to enhance convergence speed. Following the UCB principle, MetaSD-UCB selects drafters based on their empirical mean reward augmented by an uncertainty term (Line 7 in \Autoref{alg:ucb1}), thus inherently balancing the exploration-exploitation trade-off.

\subsection{Regret upper bound for MetaSD-UCB}
While the UCB algorithm achieves optimal performance in standard stochastic bandit settings \citep{lattimore2020bandit}, extending this guarantee to MetaSD is non-trivial. MetaSD's unique structure, particularly its novel stopping time regret objective (\Autoref{def:stopping_time_regret}), requires a dedicated analysis. 
We demonstrate that MetaSD-UCB attains a logarithmic regret bound with respect to the target sequence length B, formalized as follows.

\begin{theorem}[Regret upper bound]
\label{theorem:regret_upperbound_SD-UCB}
Under Assumption~\ref{assumption:main_model_alignment} and MetaSD-UCB with the BD reward, there exists a constant $C,C'>0$ such that following bound holds:
\begin{equation}
\label{eq:regret_upper_bound}
\begin{aligned}
& \textsc{Reg}(\pi,B) < \sum_{i\neq i^{\star}}\frac{8}{(N_{max})\Delta(\alpha_i)^2} (\ln{B} \\ 
&\quad +{\ln{(\ln(\sum_{i\neq i^{\star}}\frac{1}{\Delta(\alpha_i)^2}))}+C'}) + C.
\end{aligned}
\end{equation}
\vspace{-15pt}
\end{theorem}
Here, we denote $\Delta(\alpha_i)=\alpha_{i^{\star}}-\alpha_i$, where $i^{\star}$ is the index of the optimal drafter.
In~\Autoref{appendix:proof of theorem2}, we show the above regret bound is tighter then the regret bound of using BE reward.
The improvement in~\Autoref{eq:regret_upper_bound} stems directly from adopting the BD reward in \Algautoref{alg:ucb1}. Since the number of observations within each round grows with $N_{max}$, the variance of the BD reward is effectively reduced by a factor of $N_{max}$. This, in turn, leads to a smaller constant term in the regret upper bound compared to using the BE reward. The following corollary captures this observation:
\begin{corollary}[Informal]
\label{corollary:sample_complexity_ratio}
In most scenarios, the regret upper bound in~\Autoref{eq:regret_upper_bound} is tighter than the regret upper bound obtained when using the BE reward with MetaSD-UCB.

\end{corollary}
A complete proof of \Autoref{theorem:regret_upperbound_SD-UCB} and a formal statement of~\Autoref{corollary:sample_complexity_ratio} with the proof are in~\Autoref{appendix:proof of theorem2}.
Note that above analysis incorporates both temperature sampling and greedy decoding scenarios in~\Autoref{app:assumption_on_model_alignment}. Note that \Autoref{theorem:regret_upperbound_SD-UCB} is token generation instance dependent bound, holding for any specific realization of the token generation process. This per-realization guarantee is stronger than a bound on the expected regret over the distribution of all possible generations  (\Autoref{app:B_randomness}).

\vspace{-5pt}
\subsection{Extensions of MetaSD framework}
\paragraph{Switching costs} 
Switching between drafters incurs a computational cost due to KV-cache recalculations. Unlike classical bandits with fixed per-switch costs, the overhead in MetaSD depends on the number of unprocessed tokens in the current block. For this case, we propose the variant using Sequential Halving (SH) (details in \Autoref{appendix:switching_costs}).

\vspace{-5pt}
\paragraph{Non-stationary environment}

Non-stationary environments, where the optimal drafter may change across queries or even within a single generation 
, present further challenges for SD. MetaSD inherently handles inter-query non-stationarity through its query-by-query re-initialization. For intra-query non-stationarity, where the best drafter shifts mid-generation, MetaSD can be extended with non-stationary bandit algorithms~\citep{kocsis2006discounted}. Such algorithms track changing reward distributions, allowing MetaSD to sustain performance in dynamic settings. \Autoref{appendix:non-stationary_environment} offers further discussion and experiments.

\begin{table*}[t]
\caption{(Black-box SD) Speedup ratio relative to standard autoregressive greedy decoding on various datasets, comparing single specialized independent drafters, other methods (PLD \citep{saxena2023prompt} and Lookahead \citep{fu2024break}), and bandit-based drafter selection (Rand (uniformly random), EXP3 \citep{auer2002nonstochastic}, SH \citep{Sequential_Halving}, UCB). Evaluations are conducted with a single NVIDIA A6000 GPU under greedy decoding settings. Drafter specializations: 1: Code, 2: Translation, 3: Summarization, 4: QA, 5: Math.
}\vspace{-5pt}
\label{tab:main}\footnotesize
\resizebox{1.0\textwidth}{!}{%
\begin{tabular}{@{}c|lllll|l|lc|llll@{}}
\toprule
\multirow{2}{*}{Speedup} & \multicolumn{5}{c|}{SpS with specialized drafters}        & \multicolumn{1}{c|}{\textcolor{black}{SpS}}                                                   & \multicolumn{2}{c|}{Other methods}                                                          & \multicolumn{4}{c}{Bandit in MetaSpS}                           \\ \cmidrule{2-13}
                         & {Drafter1} & {Drafter2} & {Drafter3} & {Drafter4} & {Drafter5} & \multicolumn{1}{c|}{\textcolor{black}{OFA}} &{PLD} & {Lookahead} & {Rand} & {EXP3} & {SH} & {UCB}   \\ \midrule
Code                     & \textbf{2.437} \textcolor{green!60!black}{$\CIRCLE$}     & 1.224             & 1.565             & 1.814             & 1.687             & \textbf{2.435} \textcolor{green!60!black}{$\CIRCLE$} & 1.923        & 1.542                     & 1.640           & 1.919         & 2.148       & {2.300}   \\
Trans                    & 0.991             & \textbf{2.076} \textcolor{green!60!black}{$\CIRCLE$}     & 1.000             & 1.019             & 0.950             & 1.032 & 1.076        & 1.133                   & 1.150           & 1.217         & 1.422       & \textbf{1.587} \textcolor{green!60!black}{$\CIRCLE$}  \\
Sum                      & 1.513             & 1.087             & \textbf{2.133} \textcolor{green!60!black}{$\CIRCLE$}     & 1.510             & 1.387             & 1.526 & \textbf{2.501} \textcolor{green!60!black}{$\CIRCLE$}         & 1.275                      & 1.429           & 1.606         & 1.812       & {1.971} \\
QA                       & 1.332             & 1.200             & 1.343             & \textbf{1.960} \textcolor{green!60!black}{$\CIRCLE$}     & 1.252             & 1.267 & 1.178        & 1.208                    & 1.294           & 1.437         & 1.599       & \textbf{1.711} \textcolor{green!60!black}{$\CIRCLE$}  \\
Math                     & 1.483             & 1.228             & 1.378             & 1.486             & \textbf{2.454} \textcolor{green!60!black}{$\CIRCLE$}     & 1.571 & 1.653        & 1.533                   & 1.471           & 1.690         & 2.144       & \textbf{2.280} \textcolor{green!60!black}{$\CIRCLE$}  \\ \bottomrule
\end{tabular}%
}
\end{table*}

\begin{table*}[t]
\caption{(White-box SD) Speedup ratio relative to standard autoregressive greedy decoding on various datasets, comparing single specialized drafters, other methods (blockwise parallel decoding (BPD) \citep{bpd}, Medusa, Rescored-BPD (R-BPD) and Rescored-Medusa \citep{kim2024exploring}), and bandit-based drafter selection. Evaluations are conducted with a single NVIDIA A100 GPU under greedy decoding settings. }\vspace{-5pt}
\label{tab:main-white}
\resizebox{1.0\textwidth}{!}{%
\begin{tabular}{@{}c|lllll|l|cccc|cccl@{}}
\toprule
\multirow{2}{*}{Speedup} & \multicolumn{5}{c|}{Specialized Eagle drafters} & \multicolumn{1}{c|}{\textcolor{black}{Eagle}} & \multicolumn{4}{c|}{Other methods}                                                                               & \multicolumn{4}{c}{Bandit in MetaEagle}                                                         \\ \cmidrule(l){2-15} 
                         & \multicolumn{1}{l}{Eagle1} & \multicolumn{1}{l}{Eagle2} & \multicolumn{1}{l}{Eagle3} & \multicolumn{1}{l}{Eagle4} & \multicolumn{1}{l|}{Eagle5} & \multicolumn{1}{c|}{\textcolor{black}{OFA}} & \multicolumn{1}{l}{BPD} & \multicolumn{1}{l}{R-BPD} & \multicolumn{1}{l}{Medusa} & \multicolumn{1}{l|}{R-Medusa} & \multicolumn{1}{l}{Rand} & \multicolumn{1}{l}{EXP3} & \multicolumn{1}{l}{SH} & \multicolumn{1}{l}{UCB} \\ \midrule
Code                     & \textbf{3.934} \textcolor{green!60!black}{$\CIRCLE$}                      & 1.303                      & 1.776                      & 2.150                      & 2.427                      & \textbf{3.776} \textcolor{green!60!black}{$\CIRCLE$} & 1.963                   & 2.146                     & 2.661                      & 2.822                        & 2.310                    & 2.858                    &  3.650                 & {3.724}                    \\
Trans                    & 1.750                      & \textbf{2.496}  \textcolor{green!60!black}{$\CIRCLE$}                     & 2.281                      & 2.131                      & 1.714                      & 2.143 & 1.626                   & 1.442                     & 1.909                      & 2.056                        & 2.036                    & 2.171                    & 2.225                  & \textbf{2.318} \textcolor{green!60!black}{$\CIRCLE$}                  \\
Sum                      & 1.707                      & 1.507                      & \textbf{3.382} \textcolor{green!60!black}{$\CIRCLE$}      & 2.005                      & 1.589                      & 2.640 & 1.509                   & 1.455                     & 1.723                      & 2.136                        & 2.261                    & 2.261                    & 2.801                  & \textbf{3.057} \textcolor{green!60!black}{$\CIRCLE$}                  \\
QA                       & 1.842                      & 1.579                      & 2.181                      & \textbf{2.916} \textcolor{green!60!black}{$\CIRCLE$}                       & 1.783                      & 2.446 & 1.489                   & 1.468                     & 1.817                      & 2.154                        & 2.006                    & 2.128                    & 2.466                  & \textbf{2.641} \textcolor{green!60!black}{$\CIRCLE$}                   \\
Math                     & 2.584                      & 1.618                      & 2.337                      & 2.433                      & \textbf{3.903} \textcolor{green!60!black}{$\CIRCLE$} & 3.049                       & 1.696                   & 1.696                     & 2.142                      & 2.519                        & 2.449                    & 2.811                    & 3.339                  & \textbf{3.520} \textcolor{green!60!black}{$\CIRCLE$}                    \\ \bottomrule
\end{tabular}%
}
\end{table*}

\begin{table}[t]
\caption{Speedup ratio relative to standard AR greedy decoding on various multilingual datasets, comparing single specialized drafters, all drafting method, and our bandit-based algorithm (EXP3, UCB). Evaluations are conducted with a single NVIDIA A5000 GPU under greedy decoding settings. Drafter specializations: 1: Ja \textrightarrow En, 2: Ru \textrightarrow En, 3: De \textrightarrow En, 4: Fr \textrightarrow En, 5: Zh \textrightarrow En.}
\label{tab:mlin}
\resizebox{1.0\linewidth}{!}{%
\begin{tabular}{@{}l|lllll|lll@{}}
\toprule
\multirow{2}{*}{Speedup} & \multicolumn{5}{c|}{Sps with specialized drafters}                                                           & \multicolumn{3}{c}{Bandit in MetaSps}         \\ \cmidrule(l){2-9} 
                         & \textbf{Drafter1} & \textbf{Drafter2} & \textbf{Drafter3} & \textbf{Drafter4} & \textbf{Drafter5} & \textbf{DraftAll} & \textbf{EXP3} & \textbf{UCB}   \\ \midrule
Ja \textrightarrow En   & \textbf{1.757}   \textcolor{green!60!black}{$\CIRCLE$}  & 1.109   & 1.012  & 1.018  & 1.154   &  0.352   & \textbf{1.260}   \textcolor{green!60!black}{$\CIRCLE$}     & {1.161} \\
Ru \textrightarrow En & 1.055 & \textbf{1.817}   \textcolor{green!60!black}{$\CIRCLE$}  & 0.995     & 0.963  & 1.036             & 0.396 & 1.259  & \textbf{1.503}  \textcolor{green!60!black}{$\CIRCLE$} \\
De 
\textrightarrow En & 1.098 & 1.369 & \textbf{2.360} \textcolor{green!60!black}{$\CIRCLE$}  & 1.036 & 1.099 & 0.544 & 1.472  & \textbf{1.693}  \textcolor{green!60!black}{$\CIRCLE$} 
\\
Fr \textrightarrow En  & 1.106   & 1.445             & 1.108   & \textbf{2.135}  \textcolor{green!60!black}{$\CIRCLE$}   & 1.122   & 0.484 & 1.506 & \textbf{1.775}  \textcolor{green!60!black}{$\CIRCLE$} \\
Zh 
\textrightarrow En    
& 1.198 & 1.086 & 1.021 & 1.023 & \textbf{1.516}   \textcolor{green!60!black}{$\CIRCLE$}  
& 0.367 & 1.204 & \textbf{1.369}  \textcolor{green!60!black}{$\CIRCLE$} \\ \bottomrule
\end{tabular}%
}
\end{table}
\vspace{-10pt}

\section{Experiment}
\label{sec4:exp}

\subsection{Experimental setup}

\paragraph{Models}
We use Vicuna 7B \citep{vicuna} as the target LLM for both black-box and white-box SD. The primary difference between these approaches lies in the drafter's access to the target LLM: black-box drafters operate independently using only the target's final logits, while white-box drafters can access its intermediate activations and hidden states. For black-box SD, Vicuna 68M \citep{vicuna68m} serves as the base drafter architecture; each such drafter is trained on distinct task-specific datasets (generated via self-distillation from the target LLM, following \cite{seqKD, zhou2023distillspec, medusa, yi2024towards}) to ensure heterogeneity. For white-box SD, we integrate Eagle \citep{li2024eagle} with Vicuna 7B. Multiple Eagle drafters, sharing the same base architecture, are fine-tuned on separate task-specific datasets, also produced through self-distillation. As a baseline, we include a One-Size-Fits-All (OFA) drafter trained on a mixed dataset spanning all tasks, enabling evaluation of whether a single generalist drafter can match task-specialized ones. Further details are in \Autoref{sec:further_exp}.


\vspace{-5pt}
\paragraph{Number of drafts $N_{max}$} 
For black-box SD, we employ speculative sampling (SpS) \citep{speculative_decode2}, generating one draft candidate per drafter, termed as MetaSpS. For multi-draft methods like Medusa \citep{medusa} and Eagle \citep{li2024eagle}, we adhere to their original settings with a tree-attention mechanism. We employ the same tree structure for multiple Eagle drafters described in \cite{li2024eagle}, termed as MetaEagle.
Unless otherwise stated, all approaches utilize a maximum of 5 drafts ($N_{max}=5$).

\vspace{-5pt}

\paragraph{Evaluation}

We conduct experiments on NVIDIA A5000, A6000, and A100 GPUs. To ensure adaptability, particularly in multi-turn conversations, the bandit algorithm is re-initialized for every new query. Our evaluation encompasses two primary scenarios: (1) a diverse task suite comprising coding (Code) from MT-Bench \citep{llmjudge}, summarization (Sum) on CNN/Daily \citep{cnn_sum}, German-English translation (Trans) on WMT16 \citep{wmt16}, natural question answering (QA) \citep{kwiatkowski2019natural}, and mathematical reasoning (Math) on GSM8K \citep{gsm8k}; and (2) multilingual translation tasks as detailed in \Autoref{tab:motivation}, following \citet{yi2024towards}. For all evaluations, test datasets are randomly shuffled to simulate a non-stationary environment.

\vspace{-5pt}

\subsection{Main result}

\Autoref{tab:main} presents speedup ratios for black-box SD across diverse tasks. As expected, specialized drafters perform best on their respective tasks but degrade significantly on unrelated ones, highlighting the limitations of static selection. Our MetaSpS-UCB consistently achieves competitive speedup across all tasks, often matching or surpassing both specialized drafters and state-of-the-art speculative decoding techniques. This demonstrates the effectiveness of adaptive drafter selection, allowing MetaSpS-UCB to dynamically leverage the most suitable drafter for a given input. While the OFA drafter performs well, MetaSD outperforms OFA in most cases, reinforcing the advantage of task-specialized selection over a single generalized model. Among bandit algorithms, MetaSpS-UCB demonstrates superior performance compared to other baselines (EXP3, SH), even when accounting for switching costs and non-stationarity. This empirically supports our theoretical analysis, confirming the advantages of UCB in SD. 

\Autoref{tab:main-white} presents results for white-box SD using MetaEagle, where EAGLE drafters are integrated into the target LLM. Similar to MetaSpS case, specialized drafters perform well in their designated tasks but fail to generalize. In contrast, MetaEagle-UCB consistently achieves high speedup ratios, adapting effectively to different task distributions. This results demonstrates that our MetaSD algorithm can be applied to different SD architectures. 

Speedup ratios on multilingual tasks, shown in \Autoref{tab:mlin}, further corroborate these findings. As a baseline, we include DraftALL method, where in each round, all drafters are utilized for drafting. While this method generates optimal candidate at every round, verification increases linearly, resulting in significant decrease in speed-up ratios. In contrast, MetaSpS-UCB significantly outperforms this naive baseline, outperforming alternative bandit strategy (EXP3).


\begin{figure}[t]
\vspace{-5pt}
    \centering
        \centering
        \includegraphics[width=0.8\linewidth]{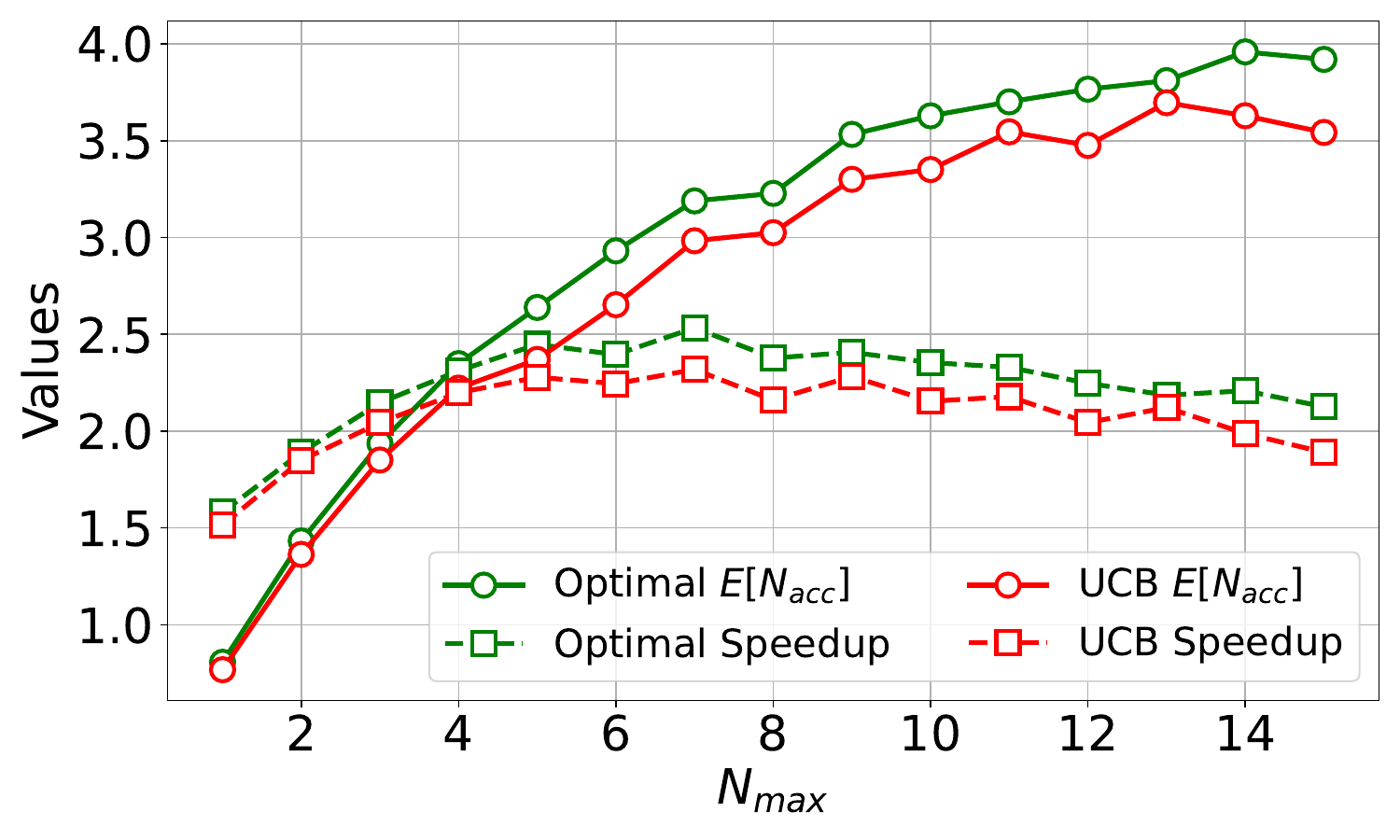}
        \vspace{-10pt}
        \caption{Ablations on $N_{max}$. `Optimal' represents the optimal drafter and UCB denotes MetaSps-UCB with BD.}
        \label{fig:ab_draft}
\vspace{-10pt}
\end{figure}

\begin{table}[t]
\centering
\caption{Speedup ratio comparison of MoE routing and MetaSD across different scenarios.}
\label{tab:routing_comparison}
\resizebox{1.0\linewidth}{!}{%
\begin{tabular}{l|ll|ll|cc}
\toprule
\multirow{2}{*}{\textbf{Routing Method}} & \multicolumn{2}{c|}{\textbf{In-domain}} & \multicolumn{2}{c|}{\textbf{Perturbed prompts}} & \multicolumn{2}{c}{\textbf{Out-of-Domain (Sps/Eagle)}} \\ 
\cmidrule(lr){2-3} \cmidrule(lr){4-5} \cmidrule(lr){6-7}
& \textbf{SpS} & \textbf{Eagle} & \textbf{SpS} & \textbf{Eagle} & \textbf{RAG} & \textbf{Finance} \\ \midrule

\multicolumn{7}{l}{\textit{MOE routing (BERT)}} \\ \midrule
Finetuned BERT-Base  & 1.772  & 2.759  & 1.456  & 2.245  & 1.645 / 2.132  & 1.410 / 2.413\\
Finetuned BERT-Large & 1.741  & 2.797  & 1.567  & 2.323  & 1.638 / 2.110  & 1.398 /2.426 \\ \midrule

\multicolumn{7}{l}{\textit{MetaSD (Bandit-based policy routing)}} \\ \midrule
MetaSpS-UCB           & \textbf{1.912} \textcolor{green!60!black}{$\CIRCLE$}  & —    & \textbf{1.820} \textcolor{green!60!black}{$\CIRCLE$}  & —    & \textbf{1.799} \textcolor{green!60!black}{$\CIRCLE$}  & \textbf{1.436} \textcolor{green!60!black}{$\CIRCLE$}  \\
MetaEagle-UCB         & —    & \textbf{3.052} \textcolor{green!60!black}{$\CIRCLE$}  & —    & \textbf{2.912} \textcolor{green!60!black}{$\CIRCLE$}  & \textbf{2.238} \textcolor{green!60!black}{$\CIRCLE$}  & \textbf{2.517} \textcolor{green!60!black}{$\CIRCLE$}  \\
\bottomrule
\end{tabular}%
}
\vspace{-5pt}
\end{table}

\begin{table}[t]
        \centering
        \captionof{table}{Average of speedup ratio comparing the BE and BD rewards for MetaSD-UCB with both SpS and EAGLE drafters over 3 different runs.}
        \vspace{-5pt}
        \resizebox{0.44\textwidth}{!}{%
        \begin{tabular}{@{}c|ll|ll@{}}
        \toprule
        \multirow{2}{*}{Task} & \multicolumn{2}{c|}{MetaSpS-UCB} & \multicolumn{2}{c}{MetaEagle-UCB} \\ \cmidrule{2-5}
                              & \multicolumn{1}{c}{{BE}}  & \multicolumn{1}{c|}{{BD}}  & \multicolumn{1}{c}{{BE}}   & \multicolumn{1}{c}{{BD}}   \\ \midrule
        Code                  & $2.052_{\pm 0.004}$        & $2.231_{\pm 0.006}$ \textcolor{green!60!black}{$\CIRCLE$}     & $3.590_{\pm 0.017}$      & $3.661_{\pm 0.003}$ \textcolor{green!60!black}{$\CIRCLE$}         \\
        Trans                 & $1.465_{\pm 0.004}$        & $1.554_{\pm 0.001}$ \textcolor{green!60!black}{$\CIRCLE$}     & $2.228_{\pm 0.009}$ \textcolor{green!60!black}{$\CIRCLE$}      & $2.201_{\pm 0.001}$         \\
        Sum                   & $1.770_{\pm 0.002}$        & $1.929_{\pm 0.001}$ \textcolor{green!60!black}{$\CIRCLE$}     & $3.038_{\pm 0.005}$      & $3.043_{\pm 0.001}$ \textcolor{green!60!black}{$\CIRCLE$}         \\
        QA                    & $1.591_{\pm 0.003}$        & $1.698_{\pm 0.001}$ \textcolor{green!60!black}{$\CIRCLE$}     & $2.629_{\pm 0.003}$ \textcolor{green!60!black}{$\CIRCLE$}      & $2.608_{\pm 0.001}$         \\
        Math                  & $1.992_{\pm 0.003}$        & $2.238_{\pm 0.002}$ \textcolor{green!60!black}{$\CIRCLE$}     & $3.461_{\pm 0.009}$      & $3.515_{\pm 0.001}$ \textcolor{green!60!black}{$\CIRCLE$}         \\ \bottomrule
        \end{tabular}%
        }
        \vspace{-10pt}
        \label{tab:reward}
\end{table}

\subsection{Comparing with MoE routing} 

Classification-based routing, analogous to Mixture-of-Experts (MoE) systems where experts (drafters) are selected via encoder outputs (e.g., BERT; \citet{bert}), serves as another baseline. \Autoref{tab:routing_comparison} shows that such fine-tuned classifiers achieve competitive speedups on in-domain tasks. However, their performance substantially degrades on perturbed prompts and out-of-domain (OOD) tasks, where MetaSD maintains robustness. This disparity arises because MoE routing relies on input-level embeddings for a static drafter assignment, struggling with prompt variations and novel task distributions. Conversely, MetaSD dynamically adapts drafter selection at each decoding step, utilizing the Block Divergence (BD) reward to capture real-time, token-level alignment. 
These findings demonstrates that MetaSD's dynamic, alignment-driven selection consistently surpasses static classification-based approaches in complex scenarios. 
\vspace{-7pt}

\begin{figure*}[t]
    \centering
    \includegraphics[width=0.98\linewidth]{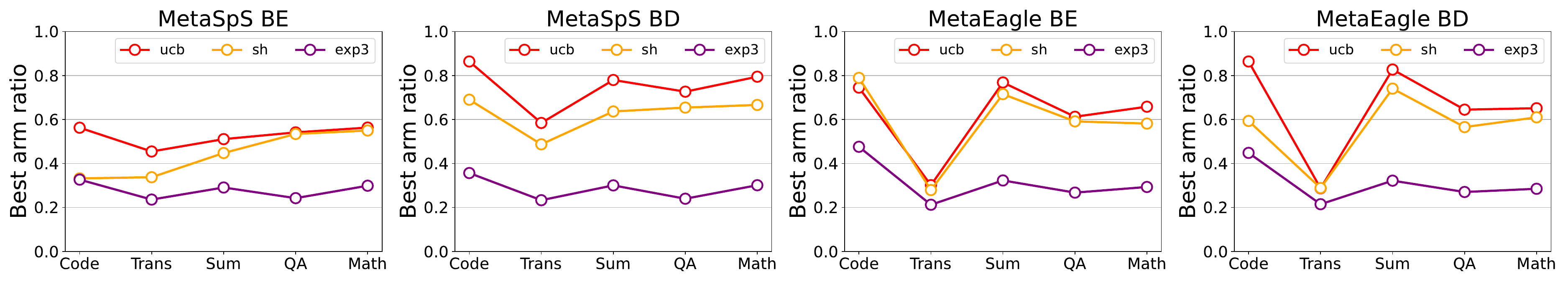}
    \vspace{-10pt}
    \caption{Best arm ratio over rounds for various configurations. (Left) MetaSpS (black-box SD) with BE and BD rewards. (Right) MetaEagle (white-box SD) with BE and BD rewards.}
    \vspace{-10pt}
    \label{fig:best_arm}
\end{figure*}

\subsection{Ablation study}

\paragraph{Draft length} 
To analyze the impact of draft length on the performance of MetaSps-UCB with the BD reward, we conduct experiments on the Code task using 5 drafters following the same setting in \Autoref{tab:main}, where we vary the draft length $N_{max}$.
\Autoref{fig:ab_draft} shows that increasing the draft length initially leads to higher $\mathbb{E}[N_{acc}]$ and speedup due to the increased parallelism in token generation. However, beyond a certain threshold, further increasing the draft length yields diminishing returns and can even decrease performance due to the higher probability of rejection and the associated overhead. 

\paragraph{Reward design}
We compare MetaSD with different types of rewards. As shown in \Tableautoref{tab:reward}, BD consistently outperforms BE in both MetaSps and MetaEagle scenarios, while the performance gap is less pronounced in MetaEagle-UCB. We attribute this difference to Eagle’s tree-attention mechanism, which explores multiple decoding paths and implicitly mitigates the limitations of BE.
\vspace{-5pt}




\begin{table}[t]
\centering
\caption{White-box performance with perturbed prompts (speedup relative to greedy decoding on A100).}
\vspace{-7pt}
\label{tab:whitebox_perturbed}
\resizebox{0.94\linewidth}{!}{%
\begin{tabular}{l|ccccc|l|l}
\toprule
\textbf{Task} & \textbf{Eagle1} & \textbf{Eagle2} & \textbf{Eagle3} & \textbf{Eagle4} & \textbf{Eagle5} & \textbf{OFA} & \textbf{MetaSD-UCB} \\ 
\midrule
\textbf{Code} & 3.748 & 1.335 & 1.697 & 2.030 & 2.451 & \textbf{3.626} \textcolor{green!60!black}{$\CIRCLE$} & {3.563} \\ 
\textbf{Trans} & 1.757 & 2.553 & 2.161 & 2.035 & 1.677 & 2.293 & \textbf{2.375} \textcolor{green!60!black}{$\CIRCLE$} \\ 
\textbf{Sum} & 1.671 & 1.529 & 3.084 & 1.939 & 1.639 & 2.648 & \textbf{2.742} \textcolor{green!60!black}{$\CIRCLE$} \\ 
\textbf{QA} & 1.837 & 1.616 & 2.094 & 2.932 & 1.722 & 2.439 & \textbf{2.516} \textcolor{green!60!black}{$\CIRCLE$} \\ 
\textbf{Math} & 2.511 & 1.664 & 2.207 & 2.913 & 3.844 & 3.136 & \textbf{3.366} \textcolor{green!60!black}{$\CIRCLE$} \\ 
\bottomrule
\end{tabular}%
}
\vspace{-10pt}
\end{table}

\paragraph{Robustness to rephrased prompts} 
To simulate real-world conditions, we assess MetaSD's performance on perturbed prompts: queries slightly rephrased by GPT-4o while retaining their original meaning. These GPT-4o generated variations introduce natural linguistic shifts both across and within tasks. For instance, a translation prompt like `Translate German to English' might become `Convert this text from German to English' and `Summarize:' could be rephrased as `Provide a concise overview of the following text:'.
While such prompt perturbations generally degrade the performance of all evaluated methods—underscoring the limitations of static drafter selection—MetaSD demonstrates robust performance, consistently outperforming baselines. This robustness stems from MetaSD's dynamic, token-level adaptation at inference, which mitigates the impact of prompt variations by not relying solely on initial prompt characteristics (details in \Autoref{tab:whitebox_perturbed} and \Autoref{tab:blackbox_perturbed}). 

\vspace{-5pt}
\paragraph{Best arm ratio} 
We evaluate the best arm ratio, which measures how frequently the optimal drafter is selected. \Autoref{fig:best_arm} tracks the evolution of this ratio across decoding rounds, comparing different reward functions and bandit algorithms for both MetaSpS and MetaEagle. Across all settings, UCB consistently converges to the optimal drafter faster than other bandit algorithms, with this effect being particularly pronounced in the MetaSpS setting. Additionally, the BD reward leads to a higher best arm ratio than BE. This aligns with our earlier findings that BD better captures the alignment dynamics in SD. 

Further long-context experiments and stochastic decoding (temperature $>0$) are in \Autoref{sec:further_exp}.

\vspace{-3pt}
\section{Discussion}

\paragraph{Memory bandwidth bound}
A potential concern with MetaSD is increased memory storage from loading multiple drafter models. However, our approach introduces minimal memory overhead. Storing all drafter weights in GPU DRAM mitigates frequent, slow system memory accesses—a key LLM bottleneck. For example, with a 7B target LLM and float16 precision, our MetaEagle framework uses at most 19GB of GPU DRAM, compared to 17GB for a single Eagle drafter. This minor increase in storage does not translate to higher memory bandwidth requirements during inference, as only one drafter is active and accessing VRAM at any time.

\paragraph{Throughput \& efficiency}
MetaSD sustains high throughput efficiency, even in distributed, batch-processing environments. \Autoref{tab:throughput_merged} shows MetaEagle-UCB surpasses the OFA Eagle baseline (e.g., 2.427 vs. 2.235 speedup). This performance is rooted in optimized drafter management: preloading parameters into DRAM minimizes memory transfers, ensuring consistent memory bandwidth irrespective of using single or multiple drafters. While MetaEagle-UCB remains competitive in heterogeneous batch settings (\Autoref{tab:throughput_merged}), a slight throughput reduction occurs due to increased I/O from drafter switching.
Nevertheless, MetaSD prioritizes latency reduction for multi-task SD over raw throughput, as faster individual responses, even at moderately lower batch efficiency, typically enhance user experience more than higher-latency processing on a fully utilized GPU. 
MetaSD also features efficient KV cache management with minimal switching overhead. Eagle drafters, for instance, compute only a single KV cache layer for new tokens, making prefilling lightweight (\Autoref{sec:stream}). 



\clearpage
\section*{Limitations}

While our framework supports general LLM architectures, broader evaluation is needed to evaluate its robustness across various models. The current single-query design can be extended to batched inference, enabling efficient instance-level drafter selection for mixed-task batches. Also, computation overhead might increase with a larger number of drafters. We provide a detailed discussion of these limitations and potential solutions in \Autoref{sec:lim}.

\bibliography{ref}

\clearpage
\appendix
\section{Overview of appendix}\label{sec:appendix}
This appendix provides supplementary material that expands on the main contents. Each section is designed to complement the research presented:

\begin{itemize}
    \item \textbf{\Autoref{sec:broad}}: Discusses the broader impact and further motivations of our work.

    \item \textbf{\Autoref{sec:lim}}: Acknowledges the limitations of our current approach and outlines promising directions for future research.

    \item \textbf{\Autoref{sec:prelim}}: Provides a prepliminary for speculative sampling (SpS).

    \item \textbf{\Autoref{sec:related}}: Provides a comprehensive review of related work, situating our contributions within the broader context of speculative decoding with LLMs and multi-armed bandit research.

    \item \textbf{\Autoref{sec:further_exp}}: Details additional experimental setups, offering further insights into the performance, behavior of our proposed method, and additional experimental results including long-context experiments, out-of-domain experiments, and evaluations with perturbed prompts.

    \item \textbf{\Autoref{sec:proof}}: Presents rigorous mathematical proofs for the theoretical guarantees established in the main paper.
    
    \item \textbf{\Autoref{app:extensions}}: Explores extensions to the MetaSD framework, addressing practical considerations such as switching costs and non-stationary environments.

    \item \textbf{\Autoref{sec:discuss}}: Offers further discussion and analysis of the results presented in the main paper, potentially including additional insights, interpretations, or comparisons.
\end{itemize}

\paragraph{Ethics statement} This work primarily focuses on improving the efficiency of LLMs through algorithmic advancements and does not directly involve sensitive data or applications that could raise immediate ethical concerns. 

\paragraph{Reproducibility statement}
To facilitate reproducibility, we provide a comprehensive exposition of the materials and experimental configurations within this paper and its accompanying appendices. The organization is as follows:

\begin{itemize}
    \item \textbf{\Autoref{sec2}} - This section presents the problem statement and pseudocode for the MetaSD framework.
    \item \textbf{\Autoref{sec3:method}} \& \textbf{\Autoref{appendix:other_extensions}} - This section provide detailed MAB algorithms for the MetaSD framework under various scenarios.
    \item \textbf{\Autoref{sec4:exp}} - This section elaborates on the implementation specifics, including the pre-trained models, datasets, and evaluation metrics.
    \item \textbf{\Autoref{sec:further_exp}} - This section delves into additional details of the experimental settings.
\end{itemize}

The source code for this paper will be made publicly available under the MIT License upon publication.

\begin{figure*}[t]
    \centering
  \includegraphics[width=0.6\linewidth]{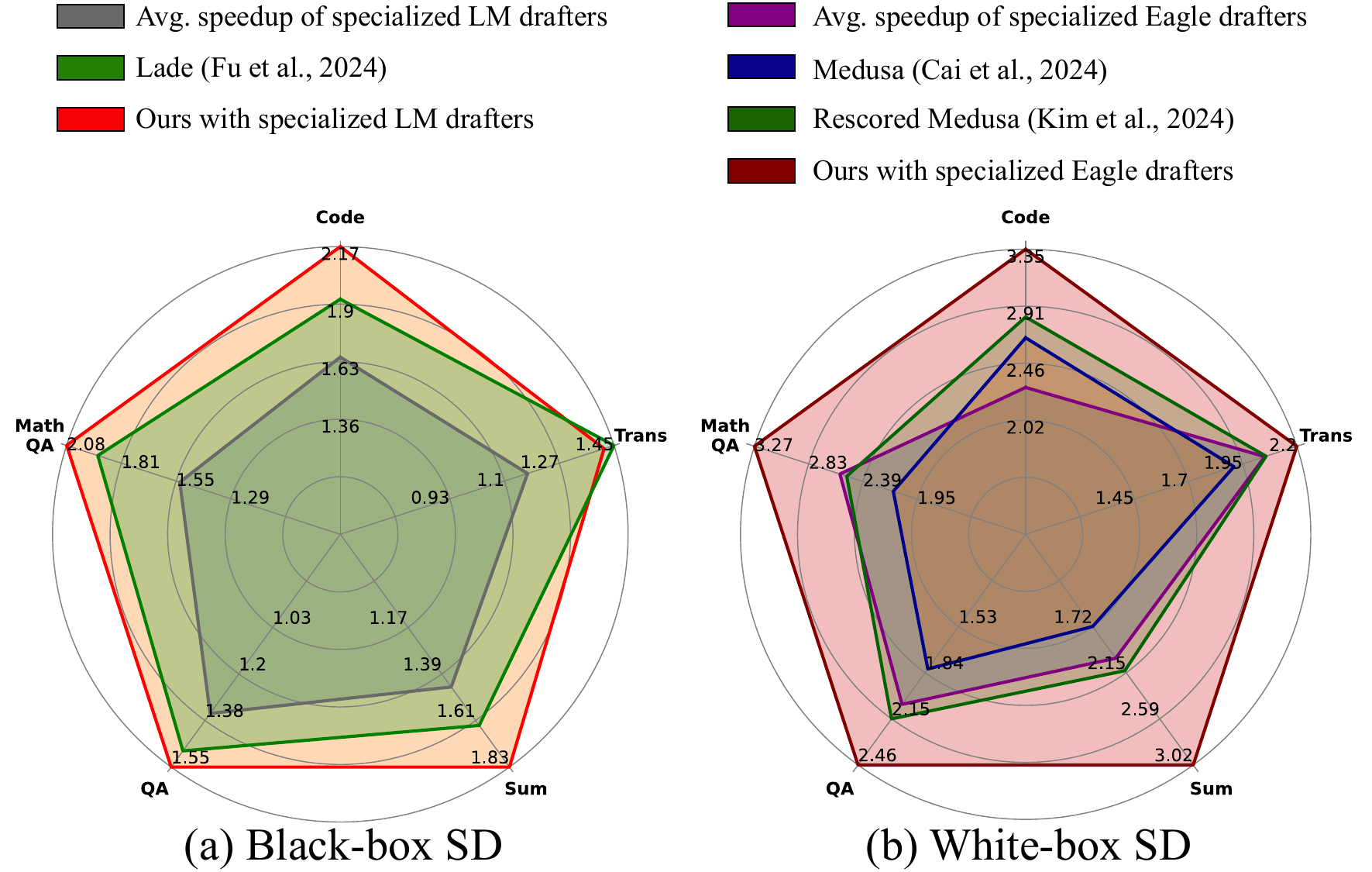}
  \caption{Comparison of average speedup ratios by various methods relative to standard autoregressive greedy decoding on a single NVIDIA A100 GPU. The target model is Vicuna 7B v1.3. (a) Results for black-box methods. (b) Results for white-box methods. Detailed settings are in \textbf{\CAutoref{sec4:exp}}.}
  \label{fig:result_overall}
\end{figure*} 

\section{Broader impact and further motivation} \label{sec:broad}

\subsection{Broader impact}

\paragraph{Generalized speedup} Our MetaSD framework for multi-drafter speculative decoding has the potential to enhance the robust speedup capabilities of LLMs (\Autoref{fig:result_overall}). By dynamically selecting from a diverse pool of drafters, the system can better adapt to a wider range of tasks and input contexts, potentially leading to reduced latency on unseen or less frequently encountered scenarios. This increased generalization could benefit various applications, such as machine translation, summarization, and creative writing, where models are often required to handle diverse and unpredictable inputs.

\paragraph{Efficiency} The primary goal of our framework is to accelerate the inference process of LLMs. By leveraging speculative decoding with multiple drafters, we aim to achieve significant speedup gains compared to traditional single-drafter approaches. This improved efficiency could enable the deployment of large language models in resource-constrained environments or real-time applications where latency is critical. Faster inference could also facilitate broader accessibility to powerful language models, making them more practical for a wider range of users and use cases.

\paragraph{Systematic impact}

Our work remains various potential societal impact. Faster and more efficient language models could lead to advancements in various domains, such as healthcare, education, and customer service, where natural language understanding and generation play crucial roles.

\begin{table*}[t]
    \centering
    \caption{Speedup ratio relative to the standard autoregressive greedy decoding on various multilingual datasets following \cite{yi2024towards} where target model is Vicuna 7B v1.3 and the drafter is decoder-only 68M language model:
    Japanese (Ja)\textrightarrow English (En) \citep{japanese}, Russian (Ru)\textrightarrow En, German (De)\textrightarrow En \citep{wmt16}, French (Fr)\textrightarrow En \citep{wmt14}, and Chinese (Zh)\textrightarrow En \citep{wmt19}. Evaluations are conducted with a NVIDIA A5000 GPU.}
    \label{tab:motivation}
    \resizebox{0.6\linewidth}{!}{%
    \begin{tabular}{@{}l|lllll@{}}
    \toprule
    Dataset & Ja-drafter & Ru-drafter & De-drafter & Fr-drafter & Zh-drafter \\ \midrule
    Ja \textrightarrow En      & \textbf{1.757} \textcolor{green!60!black}{$\CIRCLE$}     & 1.109      & 1.012      & 1.018      & 1.154      \\
    Ru \textrightarrow En      & 1.055      & \textbf{1.817} \textcolor{green!60!black}{$\CIRCLE$}      & 0.995      & 0.963      & 1.036      \\
    De \textrightarrow En      & 1.098      & 1.369      & \textbf{2.360} \textcolor{green!60!black}{$\CIRCLE$}      & 1.036      & 1.099      \\
    Fr \textrightarrow En      & 1.106      & 1.445      & 1.108      & \textbf{2.135} \textcolor{green!60!black}{$\CIRCLE$}      & 1.122      \\
    Zh \textrightarrow En      & 1.198      & 1.086      & 1.021      & 1.023      & \textbf{1.516} \textcolor{green!60!black}{$\CIRCLE$}      \\ \bottomrule
    \end{tabular}%
    }
\end{table*} 

\subsection{Further motivation}

This subsection provides another line of research motivation in \Autoref{sec:motivation}. MetaSD addresses the practical challenge of managing diverse and heterogeneous drafters often found in real-world systems (e.g., HuggingFace, Google Cloud, Azure, AWS, etc..). These drafters, pre-trained with varying objectives and frequently lacking detailed training documentation, pose significant obstacles to deployment frameworks that assume uniformity or rely on static selection strategies (e.g., rule-based strategies).

MetaSD provides a robust and adaptive mechanism for optimizing performance in environments characterized by task variability and drafter heterogeneity. By operating dynamically at the token level, it ensures task-specific efficacy without requiring retraining or fine-tuning of existing drafters. This flexibility allows MetaSD to excel in scenarios where traditional methods struggle, such as managing pre-trained drafters with black-box environment regarding the information for the use of drafters such as incomplete training histories or handling tasks with unpredictable distributions. Unlike frameworks that depend on rigid assumptions or predefined similarity metrics, MetaSD makes serving system particularly well-suited for organizations leveraging public repositories or heterogeneous resources.

\paragraph{Multilingual example} In scenarios where the drafter's strengths do not align well with the task at hand, its predictions may be less accurate, leading to fewer accepted tokens and diminished speedup benefits of SD. As \Autoref{tab:motivation} shows, a drafter trained on a specific language pair exhibits significantly higher speedup on that pair compared to others, highlighting the need for a more adaptive approach. Therefore, integrating multiple heterogeneous drafters into the SD framework can potentially address this limitation.

\section{Limitation \& future work} \label{sec:lim}

\subsection{Limitation}

\paragraph{Scalability} It is important to acknowledge that the scalability of our approach may be challenged when dealing with an extremely large number of drafters. In such scenarios, the computational overhead associated with evaluating multiple drafters at each step could potentially outweigh the speedup benefits. To address this limitation, future work could explore strategies for pre-selecting a smaller subset of promising drafters based on initial query analysis or other heuristics, before applying the MetaSD framework. This would help to maintain the efficiency and scalability of our approach even in the presence of a vast pool of potential drafters.

\paragraph{Diverse target LLMs} While our framework is designed to be agnostic to the target LLM architecture, extensive empirical evaluation across a wider range of LLMs is needed. Future work will assess the generalizability of our approach across different LLM architectures and sizes.


\paragraph{Batched inference} Our current implementation primarily focuses on single-query scenarios. However, adapting the MetaSD framework to batched inference—where different tasks are mixed within a single batch—presents an opportunity for significant efficiency gains. Unlike static single-drafter-based SD, which can suffer from suboptimal performance when handling diverse tasks in a batch, MetaSD dynamically optimizes drafter selection at the instance level. This ensures consistently high throughput, even in high-throughput batched settings.

\subsection{Future work}

\paragraph{Reward design and exploration-exploitation balance} The choice of reward function and the exploration-exploitation tradeoff significantly impact the performance of MetaSD. Exploring alternative reward designs and adaptive exploration strategies could lead to further improvements in speedup and adaptability.

\paragraph{Non-stationarity} While we briefly discuss handling non-stationarity in~\Autoref{app:extensions}, more sophisticated techniques could be investigated. This could involve incorporating change detection mechanisms or developing MAB algorithms specifically tailored to the non-stationary nature of language generation.

\paragraph{Contextual bandits} Our current framework primarily relies on observed rewards for drafter selection. Incorporating additional contextual information, such as the query type, user history, or drafter metadata, could lead to more informed decisions. Integrating contextual bandit algorithms into the MetaSD framework is a promising direction for future research.

\paragraph{Reinforcement learning (RL) formulation} The MetaSD framework could also be formulated as an RL problem, where the agent learns to select the optimal drafter based on the current state (input context and generated text) to maximize a long-term reward (e.g., overall speedup). Exploring RL-based approaches could potentially uncover novel strategies for adaptive drafter selection.

\paragraph{MAB framework over different SD algorithms} Our current work focuses on applying the MAB framework to select among heterogeneous drafters sharing the same SD algorithm (e.g., SpS or EAGLE). While this approach demonstrates significant benefits, it is worth noting that the MAB framework could potentially be extended to encompass a more diverse set of SD algorithms (e.g., Sps, PLD, Lookahead, EAGLE, and others). This would involve designing a reward function and selection strategy that can effectively compare and choose between fundamentally different SD approaches, each with its own strengths and weaknesses. Exploring this broader application of the MAB framework in speculative decoding is an interesting direction for future research.
\section{Preliminary: speculative sampling} \label{sec:prelim}

Speculative decoding accelerates LLM inference by employing a smaller draft model to predict future tokens, which are then verified by the target LLM. This parallel token generation can significantly reduce latency, especially when the draft model's predictions align well with the target LLM's output distribution.

\begin{algorithm*}[!t] 
\DontPrintSemicolon
    \caption{Speculative sampling (SpS)}
    \label{alg:sps}
\begin{algorithmic}[1]
    \SetKwInOut{Input}{Input}
    \INPUT: Target LLM $\mathcal{M}_p$, a small drafter $\mathcal{M}_q$, initial prompt sequence $x_1, \ldots, x_l$ and target sequence length $B$.\\
    \WHILE{$l < B$} 
        \FOR{$e \leftarrow 1, \ldots, E$}
            \STATE $x_{l_e} \sim \mathcal{M}_q(x|x_1, \ldots, x_l, x_{l_1}, \ldots, x_{l_{e-1}})$
        \ENDFOR
        \STATE In parallel, compute $E+1$ sets of logits drafts $x_{l_1}, \ldots, x_{l_E}$ with the target LLM $\mathcal{M}_p$: \\
         \quad \quad $\mathcal{M}_p(x|x_1, \ldots, x_l), \mathcal{M}_p(x|x_1, \ldots, x_l, x_{l_1}), \ldots, \mathcal{M}_p(x|x_1, \ldots, x_l, x_{l_1}, \ldots, x_{l_E})$
        \FOR{$j \leftarrow 1, \ldots, E$}
            \STATE Sample $r \sim U[0,1]$ from a uniform distribution
            \IF{$r < \min(1, \frac{\mathcal{M}_p(x|x_1, \ldots, x_{l+j-1})}{\mathcal{M}_q(x|
            x_1, \ldots, x_{l+j-1})})$}
                \STATE Set $x_{l+j} \leftarrow x_{l_{j}}$ and $l \leftarrow l +1$
            \ELSE
                \STATE Sample $x_{l+j} \sim (\mathcal{M}_p(x|x_1, \ldots, x_{l+j-1}) - \mathcal{M}_q(x|x_1, \ldots, x_{l+j-1}))_{+}$ and exit for loop.
            \ENDIF
        \ENDFOR
        \STATE If all tokens $x_{l+1}, \ldots, x_{l+E}$ are accepted, sample extra token\\  $x_{l+E+1} \sim \mathcal{M}_p(x|x_1, \ldots, x_l, x_{l+E})$ and set $l \leftarrow l+1$
    \ENDWHILE
\end{algorithmic}
\end{algorithm*}

\Algautoref{alg:sps} outlines the speculative sampling procedure \citep{speculative_decode, speculative_decode2}. Given an initial prompt sequence, the draft model generates $E$ potential future tokens. Concurrently, the target LLM computes the probabilities of these tokens, as well as the probability of its own prediction for each subsequent token position. 
A drafted token is accepted if its probability, according to the target LLM, exceeds a certain threshold. This threshold is determined by comparing the target LLM's probability for the drafted token to both the draft model's prediction and a random sample, ensuring only high-confidence drafts are accepted. If a drafted token is rejected, the target LLM samples a token from the residual distribution, which represents the difference between its own prediction and the draft model's. This process iterates until the desired sequence length is reached.

Speculative sampling allows the target LLM to process multiple tokens in parallel by drafting them in advance, reducing the overall generation time. When the draft model's predictions are accurate, a significant portion of the generated tokens are accepted, leading to substantial speedup. The verification step and residual sampling ensure that the final generated sequence remains consistent with the target LLM's distribution, preserving generation quality. Speculative sampling provides a foundation for our proposed framework, where we extend this approach to incorporate multiple drafters and dynamically select the optimal one using MAB algorithms.

\section{Related work}\label{sec:related}

\subsection{Speculative decoding}

Speculative decoding employs a draft-then-verify paradigm to enhance LLM inference speed. This approach tackles the latency bottleneck in autoregressive decoding, where extensive memory transfers for each token generation lead to underutilized compute resources \citep{memorybound}. Pioneering works by \cite{speculative_decode, speculative_decode2} introduced speculative decoding and sampling, enabling lossless acceleration of diverse sampling methods. These methods leverage smaller models within the same model family (e.g., T5-small for T5-XXL) without additional training. Models like Eagle \citep{li2024eagle} and Medusa \citep{medusa} integrate lightweight feedforward neural network heads into the LLM architecture, enabling early drafting of token sequences and improving throughput. Recent advancements have further refined speculative decoding by adopting judge model framework for accepting the draft tokens~\citep{llmjudge} or combining with reward guided decoding in reasoning tasks~\citep{liao2025reward}.  

Despite their efficacy, these methods often rely on a single drafter or a fixed set, limiting adaptability to diverse tasks and input contexts. \citet{yi2024towards} propose specialized drafters based on the self-distilled dataset training, but dynamically selecting among heterogeneous drafters remains an open challenge. \citet{onspec} suggest online training of specialized drafters, but their reliance on query-based classification and limited speedup gains highlight the need for a more comprehensive solution.

\subsection{Bandit algorithms}

\paragraph{Multi-armed bandit}
Multi-armed bandit (MAB) problem has been extensively studied for decades with various settings. For stochastic MAB setting,~\citet{lai1985asymptotically} and \citet{agrawal1995sample} provided asymptotic optimal regret bounds that is logarithmic to the total round $T$ and \citet{auer2002finite,audibert2007tuning} and \citet{honda2010asymptotically} proved 
this result also holds when $T$ is finite. For another variant, EXP3 algorithm~\citep{auer2002nonstochastic} proves the optimal regret bound in adversarial environment where reward distribution of each arm can change by adversary in every round.

\paragraph{Budgeted bandit}
The budgeted MAB problem address a bandit scenario where each arm pull yields both a reward and a cost drawn from individual distributions. Here, the goal is to maximize the cumulative reward until sum of the cost reaches the budget. Then, the optimal arm would be the one with the highest reward-to-cost ratio. $\epsilon$-First policies~\citep{tran2010epsilon} and KUBE~\citep{tran2012knapsack} assumed a non-stochastic fixed cost for each arm pull.~\citet{ding2013multi} provided UCB-BV algorithm where cost for each arm is assumed to be a bounded discrete random variable.

\paragraph{Bandits with switching costs}
In real-world scenarios, a cost may be incurred whenever switching arms. This is related to the MAB problem with switching costs.~\citep{dekel2014bandits,gao2019batched,rouyer2021algorithm,esfandiari2021regret,amir2022better}. For stochastic MAB,~\citet{gao2019batched} and \citet{esfandiari2021regret}
assume a fixed cost is incurred whenever switching arms. They proved an instance-dependent regret bound ${O}(\log{T})$ which does not depend on the unit switching cost value.

\paragraph{Pure exploration}
Pure exploration or best arm identification (BAI) problems \citep{even2002pac,even2006action,Successive_rejects} aim to explore as much as possible throughout the round to obtain the best arm at the end of the round. This contrasts with the traditional MAB objective which is maximizing cumulative reward. \citet{even2002pac, mannor2004sample, even2006action} investigated pure exploration in MAB under the PAC learning framework. 
BAI problems are primarily categorized into two settings. First, in the fixed budget setting \citep{Successive_rejects,Sequential_Halving,carpentier2016tight}, the goal is to minimize the chance of selecting sub-optimal arms within a fixed number of rounds. The other problem targets fixed confidence setting \citep{Sequential_Halving,jamieson2014lil,garivier2016optimal,chen2017towards} whose objective is to minimize number of rounds required to achieve a desired confidence level.

\paragraph{Non-stationary bandit}
Non-stationary bandit problems assume that reward distribution of each arm changes over time. The goal in non-stationary bandit problems is to find a balance between exploration and exploitation while carefully managing past information to adapt to the dynamic environment. Among the earliest works, \citet{gittins1974dynamic} assumed that only the best arm changes over time. This assumption was later relaxed in~\citet{whittle1988restless}, where the authors allow the mean reward for each arm to change at every round. \citet{slivkins2008adapting} assumed reward distribution follows a Brownian motion and established a regret upper bound that grows linear in rounds. Another line of works quantifies the degree of non-stationarity in the bandit instance by assuming a fixed value of $L$ which represents a number of times reward distributions change. \citet{auer2002nonstochastic} suggested EXP3.S algorithm and proved regret upper bound with given $L$ but slightly worse when $L$ is not given. \cite{kocsis2006discounted} suggested Discounted-UCB, where they obtain reward estimates with discounting factor over time. \citet{garivier2011upper} introduced Sliding-window UCB, where they used fixed-size window to retain information of the rounds within the window for estimating mean reward. ADSWITCH in~\citet{auer2019adaptively} is proven to be nearly minimax optimal, achieving the state-of-the art regret bound without any prior knowledge of $L$.

\subsection{Large language models and bandits}\label{app:concurrent}
Recently, several works have made connections between LLMs with bandits using the emergent abilities of LLMs. One side of works utilize LLM as an agent to solve decision making problems combining with bandit framework~\citep{baheri2023llms,felicioni2024importance,xia2024beyond,park2024llm, ong2024routellm}. Several works use bandit algorithms for improve the performance guarantee of LLMs with certain tasks such as for efficient prompt optimization \citep{shi2024efficient} and online model selection \citep{xia2024llm}. The latter concept has been further developed into LLM routing~\citep{hu2024routerbench, panda2025adaptive, jitkrittum2025universal}.

\textbf{Concurrent works} have leveraged the idea of bandit framework for SD. \citet{liu2024speculative} used Thomson sampling algorithm (which is one of the most popular bandit algorithm) to adaptively choose maximum candidate length $N_{max}$ combining with early-exit framework. 
\citet{huang2024context} assume existence of multiple drafters and formulate SD as a contextual bandit problem. However, they rely on collecting offline samples for the policy learning which can be costly. Furthermore, their approach is regarded as a classification problem that the selected drafter is fixed in a single query. 

Most relevant to ours, \citet{hou2025banditspec} propose BANDITSPEC, which also leverage bandit algorithms for SD with multi-drafter scenario. 
While \citet{hou2025banditspec} reported similar observations to ours but our work having key differences from them. First, while \citet{hou2025banditspec} modify the confidence range of the UCB algorithm to control the regret analysis estimating mean accepted tokens, our theoretical analysis is built upon the budgeted bandit setting~\citep{heyden2024budgeted} \textbf{which can be applied to more general reward shaping scenarios.} Moreover, our novel analysis on reward design shows that using estimation of mean acceptance rate as a reward (BD reward) which is obtained from the fine-grained model logits (not just the number of accepted tokens), improves the performance of the speculative decoding both practically and theoretically. Finally, \Autoref{theorem:regret_upperbound_SD-UCB} holds for any token generation instance (i.e, any realization of verified token instances) which is stronger than the expected regret over all token generation instances as done in \citet{hou2025banditspec}.

\section{Experiment detail} \label{sec:further_exp}
Since our experiment dataset covers broad and large dataset, we conduct single run experiment for the result unless specified otherwise.
\subsection{Training specialized drafters with self-distilled data}

Following the \citet{yi2024towards}, we use their training strategy consisting of two steps:

\begin{enumerate}
    \item Pretraining drafters on a portion of C4 dataset \citep{c4} and ShareGPT dataset \citep{sharegpt}.

    \item Finetuning the models with self distilled data having the target task with templates.
\end{enumerate}

\paragraph{Self-distilled data} Following prior work \citep{seqKD, zhou2023distillspec, medusa, yi2024towards}, we generate the training data for specialized drafters through self-distillation from the target LLM. To capture the full spectrum of its output variability, we generate multiple responses at various temperatures—\{0.0, 0.3, 0.7, 1.0\}.  We utilize this self-distilled dataset for training both independent small drafter models and dependent Eagle drafters. For Eagle-specific training details, we adhere to the settings outlined in the original Eagle paper \citep{li2024eagle}.

\subsection{Drafter details}
All independent drafters are based on a decoder-only Llama transformer model with 68M parameters. The model configuration includes 2 hidden layers, 768 hidden size, 12 attention heads, and a vocabulary size of 32,000. Other key settings are: silu activation function, 0.0 attention dropout, and no weight decay. The training recipe involves pretraining on a subset of the C4 and ShareGPT datasets, followed by fine-tuning on task-specific data generated through self-distillation from the target LLM. We employ 4 NVIDIA A100 GPUs with 80GB memory, utilizing techniques like FSDP (Fully Sharded Data Parallelism), gradient checkpointing, and lazy preprocessing to optimize training efficiency. Hyperparameters include a batch size of 8, 3 training epochs, a learning rate of 2e-5, and a cosine learning rate scheduler with a warmup ratio of 0.03. We maintain consistent architecture and training procedures across all white-box drafters, ensuring their heterogeneity stems solely from the diverse task-specific datasets they are fine-tuned on. For further specifics on Eagle drafter training, we refer readers to the original Eagle paper \citep{li2024eagle}.

\subsection{Datasets}
All of the datasets for our experiments are used within the range of the intended use for the original paper, not 
\paragraph{Training dataset}
We utilize a diverse collection of datasets to train our specialized drafters, ensuring their proficiency across various tasks and languages:

\begin{itemize}
    \item ShareGPT \citep{sharegpt}: A dataset of approximately 58,000 conversations scraped. These conversations include both user prompts and responses from OpenAI's ChatGPT.

    \item WMT16 De\textrightarrow En \citep{wmt16}: A dataset for German-to-English machine translation, providing high-quality parallel text data.

    \item JparaCrawl-v3.0 \citep{japanese}: A large-scale Japanese web corpus, enabling training of a drafter specialized in Japanese-to-English translation.

    \item WMT16 Ru\textrightarrow En \citep{wmt16}: A parallel corpus for Russian-to-English machine translation, similar to the WMT16 De→En dataset but focusing on the Russian language.

    \item WMT14 Fr\textrightarrow En \citep{wmt14}: A dataset for French-to-English machine translation, providing additional multilingual training data.

    \item WMT19 Zh\textrightarrow En \citep{wmt19}: A dataset for Chinese-to-English machine translation, further expanding the language coverage of our drafter pool.

    \item Code alpaca \citep{codealpaca}: A dataset of code generation instructions and corresponding outputs, facilitating the training of a drafter specialized in code-related tasks.

    \item CNN/Daily mail \citep{cnn_sum}: A dataset for summarization, comprising news articles and their corresponding summaries.

    \item Natural question answering \citep{kwiatkowski2019natural}: A large-scale question answering dataset based on real user queries and Wikipedia passages, aiding in training a drafter for question answering tasks.

    \item Meta math question answering \citep{yu2023metamath}: A dataset focusing on mathematical question answering, providing specialized training data for a math-oriented drafter.
\end{itemize}

\paragraph{Evaluation dataset}

\begin{itemize}
    \item Multilingual translation: Ja to En \citep{japanese}, Ru to En, De to En \citep{wmt16}, Fr to En \citep{wmt14}, and Zh to En \citep{wmt19}.

    \item Code generation: Code tasks from the MT-Bench dataset \citep{llmjudge}.

    \item Summarization: CNN/Daily summarization dataset \citep{cnn_sum}.

    \item Question answering: Natural Questions dataset \citep{kwiatkowski2019natural}.

    \item Math reasoning: GSM8K mathematical reasoning dataset \citep{gsm8k}.
\end{itemize}

\paragraph{Templates} We employ specific prompt templates during model evaluation to guide the behavior of the target LLM and drafters, ensuring consistency and clarity in task execution. These templates are carefully designed to elicit desired responses and provide relevant context for each task category. Before the data templates, system prompts of LLMs are positioned at the front to provide additional context or instructions.

\begin{itemize}
    \item Multilingual translation: `Translate this sentence from [source language] to English: [source sentence]'.
    
    \item Code generation: Its instruction depends on the query.
    
    \item Summarization: `Summarize: [article text]'.

\end{itemize}

\subsection{MAB settings}
\label{subsec:mab_settings}
In our experiments, we set the exploration strength $\beta$ for MetaSD-UCB to 0.01, balancing exploration and exploitation. For MetaSD-EXP3, we use a gamma value of 0.4 to control the degree of exploration. In the SH algorithm, we set the period to 1, ensuring frequent elimination of underperforming drafters. 

\subsection{Baseline}

We conduct several SD methods, ensuring their open-source availability and robust performance. Each method embodies a distinct strategy for accelerating LLM inference:

\begin{itemize}
    \item SpS \citep{speculative_decode2}: SpS employs a smaller LM from the same model series as the drafter. In the verification stage, if a token is rejected, SpS corrects it using residual probability to maintain generation quality.

    \item BPD, Medusa, and Eagle \citep{bpd, medusa, li2024eagle}: These methods enhance the target LLM by incorporating additional lightweight FFN heads. These heads draft potential token sequences based on the penultimate layer representations from the target LLM.

    \item PLD \citep{saxena2023prompt}: Implementing the ideas of \citep{yang2023inference}, PLD selects text spans directly from the input to serve as drafts, aiming for relevant and accurate initial predictions.

    \item R-BPD (Rescored blockwise parallel decoding) and R-Medusa (Rescored Medusa) \citep{kim2024exploring}: This method enhances BPD by rescoring the drafts at test-time, aiming to increase the number of accepted tokens.
\end{itemize}

\subsection{Reward distribution} \label{app:reward_abl}

\begin{figure}
    \begin{subfigure}{0.49\textwidth}
        \centering
        \includegraphics[width=\linewidth]{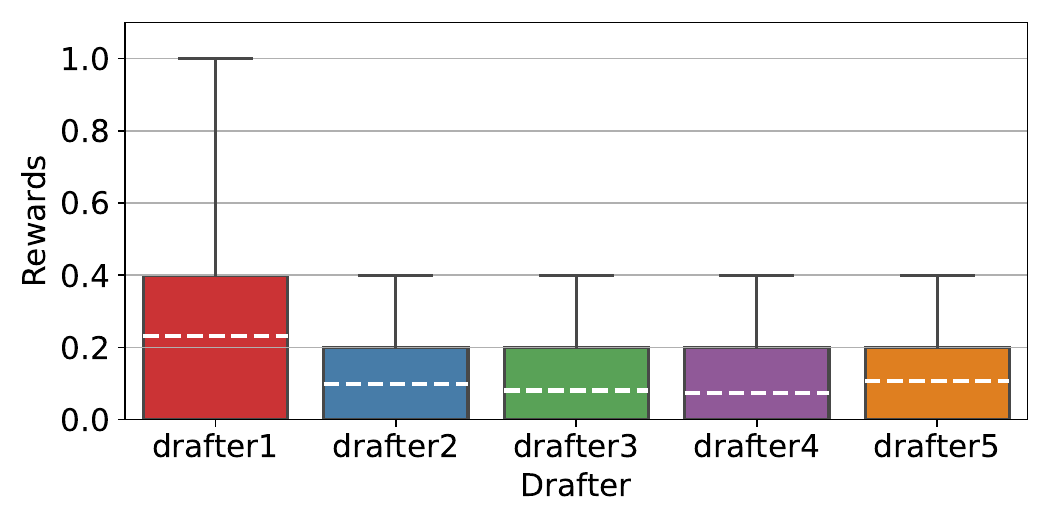}
        \caption{BE reward}
    \end{subfigure}
    \hfill
    \begin{subfigure}{0.49\textwidth}
        \centering
        \includegraphics[width=\linewidth]{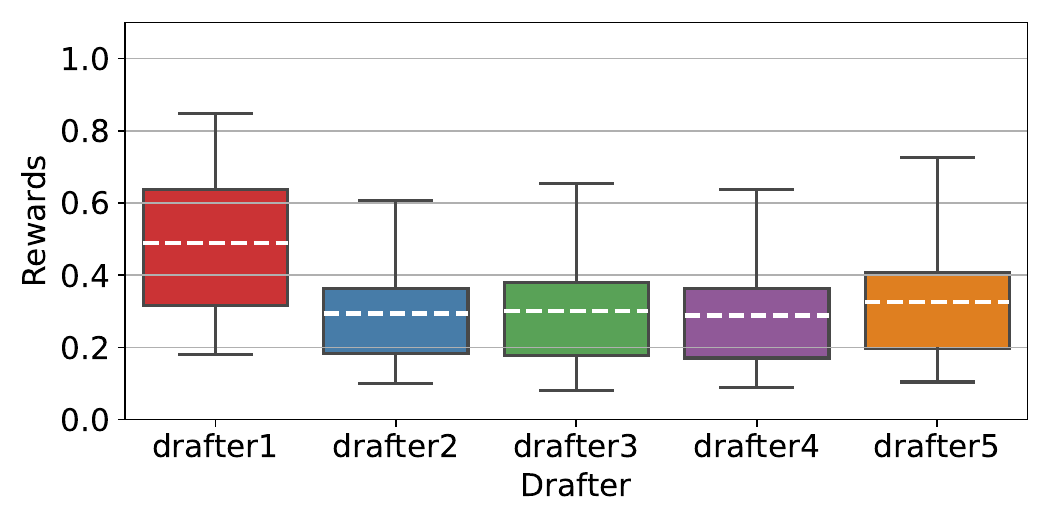}
        \caption{BD reward}
    \end{subfigure}
    \caption{Comparison of rewards on the Ja\textrightarrow En dataset across different drafters in two scenarios: (a) BE and (b) BD. Box plots show the distribution of rewards, with whiskers extending to the 5th and 95th percentiles. Drafter specializations: 1: Ja \textrightarrow En, 2: Ru \textrightarrow En, 3: De \textrightarrow En, 4: Fr \textrightarrow En, 5: Zh \textrightarrow En.}
    \label{fig:reward}
\end{figure}

\Autoref{fig:reward} and \Autoref{tab:reward_stats} present a statistical analysis of the BE and BD reward distributions, collected using autoregressive decoding with the same Japanese dataset and drafter configurations as in \Autoref{tab:motivation}. Several key observations emerge:

\begin{itemize}
    \item Lower variance: The BD reward exhibits lower variance compared to the BE reward across all drafters. This suggests that BD provides a more stable and consistent feedback signal, leading to faster convergence with less sample complexity.
    
    \item Improved discrimination: The difference in mean reward between the optimal drafter (Drafter 1; Ja-drafter) and the suboptimal drafters is more pronounced with the BD reward. This improved discrimination between drafters can facilitate quicker identification of the optimal drafter by the MAB algorithm.
    
    \item Reduced sparsity:  A significant portion of the BE rewards are zero, particularly for the suboptimal drafters. This sparsity can hinder the learning process of the MAB algorithm. In contrast, the BD reward consistently provides non-zero feedback, enabling continuous learning and adaptation.
\end{itemize}

These observations collectively suggest that the BD reward offers several advantages over the BE reward in the context of MetaSD. Its lower variance, improved discrimination between drafters, and reduced sparsity contribute to a more informative and efficient learning signal for the MAB algorithm, potentially leading to faster convergence and better overall performance.

For the experiment, we use the same Japanese dataset and drafter configurations as in \Autoref{tab:motivation}, employing autoregressive decoding to collect BE and BD rewards at each step without actual speculative execution.

\begin{table*}[ht!]
\centering
\caption{Speedup ratio on long-context De\textrightarrow En translation with the same settings in \Tableautoref{tab:mlin}.}
\label{tab:long-deen}
\resizebox{0.7\textwidth}{!}{%
\begin{tabular}{@{}c|ccccc|c@{}}
\toprule
Dataset & Drafter1 & Drafter2 & Drafter3 & Drafter4 & Drafter5 & UCB \\ \midrule
Long De→En & 1.238 & 1.316 & 2.044 & 0.970 & 1.187 & 2.031 \\ \bottomrule
\end{tabular}%
}
\end{table*}

\subsection{Long-context De \textrightarrow En translation}

While our results in \Autoref{tab:main} and \Autoref{tab:mlin} have the relatively less effectiveness of MetaSpS on the WMT16 De\textrightarrow En translation task than other tasks, it is worth noting that this dataset primarily consists of relatively short sentences with an average length of fewer than 100 tokens. To assess the performance of our framework in a more challenging long-context scenario, we evaluate it on a new De\textrightarrow En translation dataset with an average context length of 500 tokens generated by GPT-4o. As shown in \Autoref{tab:long-deen}, MetaSpS-UCB achieves a speedup ratio of 2.031 on this long-context dataset, approaching the performance of the optimal drafter (Drafter3).

\subsection{Evaluations on out-of-domain datasets}\label{app:ood}

To evaluate the adaptability and performance of our MetaSD framework in out-of-domain settings, we conduct additional experiments using the Alpaca-Finance \citep{financealpaca} and RAG datasets \citep{xia2024unlocking}. These datasets fall outside the domains of the specialized drafters used in our main experiments, providing a robust test of MetaSD's ability to generalize. The results in \Autoref{tab:ood_results}, measured using an NVIDIA A100 GPU, are presented.

\begin{table*}[t]
\centering
\caption{\textcolor{black}{Performance of MetaSpS, MetaEagle, and baselines on out-of-domain datasets (measured on A100 GPU).}}
\label{tab:ood_results}
\resizebox{\textwidth}{!}{%
\begin{tabular}{l|ccccc|c|c|ccccc|c|c}
\toprule
\textbf{Dataset} & \textbf{Drafter1} & \textbf{Drafter2} & \textbf{Drafter3} & \textbf{Drafter4} & \textbf{Drafter5} & \textbf{OFA Drafter} & \textbf{MetaSpS-UCB} & \textbf{EAGLE1} & \textbf{EAGLE2} & \textbf{EAGLE3} & \textbf{EAGLE4} & \textbf{EAGLE5} & \textbf{OFA Eagle} & \textbf{MetaEagle-UCB} \\ 
\midrule
\textbf{RAG} & 1.720 & 1.373 & 1.752 & 1.944 & 1.552 & 1.638 & \textbf{1.799} & 1.844 & 1.568 & 2.566 & 2.535 & 1.793 & 2.175 & \textbf{2.238} \\ 
\textbf{Finance} & 1.416 & 1.284 & 1.414 & 1.550 & 1.397 & 1.367 & \textbf{1.436} & 2.432 & 2.175 & 2.494 & 2.826 & 2.175 & 2.435 & \textbf{2.517} \\ \bottomrule
\end{tabular}%
}
\end{table*}

\paragraph{Superior adaptability} The results indicate that MetaSD consistently outperforms both OFA drafters and most of individual specialized drafters in out-of-domain scenarios. This highlights its ability to dynamically adapt to new tasks without relying on prior assumptions about domain similarity. The following provides the limitations of similarity-based selection:

\begin{itemize}
    \item Computing similarity between sentence embeddings requires encoding the context to generate embeddings. For inputs exceeding 128 tokens, this process can significantly increase inference time. For example, with over 100 tokens, similarity computation becomes slower than MetaSD's dynamic drafter selection.

    \item High accuracy in selecting the correct drafter based on embeddings is challenging, leading to potential misclassifications. Errors in this step can result in suboptimal drafter performance. For example, as input lengths increase, the performance gap between static Math drafters and MetaSD-UCB narrows, reducing the benefits of static drafter selection.
\end{itemize}

\paragraph{Intractability with heterogeneous drafters} In practical scenarios, heterogeneous drafters often lack complete or uniform training descriptions. Under such conditions, similarity-based selection becomes infeasible. MetaSD’s dynamic and adaptive approach offers a scalable alternative, ensuring robust performance even with limited information about drafter specialization.

\subsection{Additional evaluations with perturbed prompts}
In addition to the results in \Autoref{sec4:exp}, \Autoref{tab:blackbox_perturbed} also shows the strength of MetaSD’s dynamic token-level selection mechanism, which adapts to the token distributions during inference rather than relying solely on the characteristics of the input prompt. 

\begin{table*}[t]
\centering
\caption{Black-box performance with perturbed prompts (speedup relative to greedy decoding, measured on A100 GPU).}
\label{tab:blackbox_perturbed}
\resizebox{0.9\textwidth}{!}{%
\begin{tabular}{l|ccccc|c|c}
\toprule
\textbf{Task} & \textbf{Drafter1} & \textbf{Drafter2} & \textbf{Drafter3} & \textbf{Drafter4} & \textbf{Drafter5} & \textbf{OFA Drafter} & \textbf{MetaSpS - UCB} \\ 
\midrule
\textbf{Code} & 2.368 & 1.158 & 1.521 & 1.763 & 1.633 & 1.937 & \textbf{2.139} \\ 
\textbf{Translation} & 0.997 & 1.986 & 0.973 & 1.036 & 0.935 & 0.969 & \textbf{1.422} \\ 
\textbf{CNN} & 1.458 & 1.016 & 1.895 & 1.458 & 1.318 & 1.521 & \textbf{1.779} \\ 
\textbf{NQA} & 1.297 & 1.158 & 1.285 & 1.907 & 1.237 & 1.387 & \textbf{1.610} \\ 
\textbf{MathQA} & 1.482 & 1.184 & 1.357 & 1.470 & 2.346 & 1.895 & \textbf{2.149} \\ 
\bottomrule
\end{tabular}%
}
\end{table*}

\subsection{Throughput over Eagle drafters}

To evaluate the throughput efficiency of our proposed method, particularly in distributed system deployments where batch processing plays a critical role, we conduct experiments under the same settings described in the original Eagle paper \citep{li2024eagle}. Using an RTX 3090 (24GB) with the Vicuna 7B model, we measured throughput across a diverse set of tasks. The results demonstrate that MetaEagle-UCB achieves superior throughput compared to the single OFA Eagle, with a speedup factor of \textbf{2.427} versus \textbf{2.235} for single drafters.

A key strength of our drafter management mechanism lies in its ability to maintain throughput efficiency comparable to single-drafter methods. This is facilitated by preloading drafter parameters into DRAM, thereby avoiding frequent memory transfers to VRAM during computation. As a result, both the number of memory movements and the overall memory bandwidth requirements remain consistent with those of single-drafter configurations, even in scenarios involving multiple drafters. Additionally, the computational structure of MetaSD is designed to scale effectively across batches. Performance gains observed in single-batch scenarios carry over seamlessly to multi-batch settings, ensuring throughput efficiency in real-world distributed environments.

\subsection{MetaEagle-UCB with Efficient KV Cache Strategies}\label{sec:stream}

In our framework, the KV cache is recalculated for the previous context whenever a drafter switch occurs. Despite this recalculation, the computational overhead is negligible, even for relatively long contexts. This efficiency arises from the minimal cost of prefilling the KV cache for a small drafter. For instance, in the Eagle drafter, only one layer of KV cache is computed for the unseen context, ensuring computational efficiency.

To further validate the framework's efficiency, we conducted additional experiments incorporating StreamingLLM techniques \citep{xiao2023efficient}. These techniques circumvent the need for full KV cache recalculation, offering an alternative method for reducing computational costs. The results, summarized in \Autoref{tab:stream}, demonstrate that StreamingLLM achieves comparable performance to the default approach of KV cache recalculation, highlighting the robustness of MetaSD.

\begin{table*}[ht!]
\centering \caption{Performance comparison of MetaSD-UCB with different KV cache strategies (speedup relative to standard greedy decoding, measured on A100 GPU).} 
\label{tab:stream}
\resizebox{0.7\textwidth}{!}{%
\begin{tabular}{l|cc} 
\toprule 
\textbf{Task} & \textbf{MetaEagle-UCB (Recomputing KV)} & \textbf{MetaEagle-UCB with StreamingLLM} \\ \midrule 
\textbf{Code} & 3.724 & 3.624 \\
\textbf{Trans} & 2.318 & 2.352 \\
\textbf{Sum} & 3.057 & 2.986 \\
\textbf{NQA} & 2.641 & 2.654 \\
\textbf{Math} & 3.520 & 3.338 \\
\bottomrule \end{tabular}%
} \end{table*}

These results confirm two key observations. First, the computational overhead introduced by full KV cache recalculation is minimal, as evidenced by MetaEagle-UCB maintaining high performance across tasks. This demonstrates that recalculating the KV cache is not a significant bottleneck. Second, Streaming Decode techniques provide an effective alternative, yielding similar overall performance with slight improvements observed in specific cases such as Translation and QA. These findings underscore the flexibility and efficiency of MetaSD in managing KV cache strategies.

\subsection{Temperature sampling}\label{sec:stochastic_decoding}

\begin{table*}
\centering
\caption{Speedup ratio with temperature sampling as temperature is set to 0.7 over a NVIDIA A6000 GPU.}
\vspace{-5pt}
\label{tab:temperature}
\resizebox{0.45\textwidth}{!}{%
\begin{tabular}{@{}c|ccccc|c@{}}
\toprule
\multirow{2}{*}{Dataset} & \multicolumn{5}{c|}{SpS with specialized drafters}                                & \multicolumn{1}{l}{Bandit} \\ \cmidrule(l){2-7} 
                         & \textbf{Drafter1} & \textbf{Drafter2} & \textbf{Drafter3} & \textbf{Drafter4} & \textbf{Drafter5} & \textbf{UCB}               \\ \midrule
Code        & \textbf{2.250} & 1.215          & 1.379          & 1.532          & 1.513          & 1.896 \\
Trans       & 1.086          & \textbf{1.886} & 1.096          & 1.130          & 1.078          & 1.431 \\
Sum         & 1.461          & 1.165          & \textbf{1.874} & 1.463          & 1.353          & 1.744 \\
QA          & 1.316          & 1.193          & 1.324          & \textbf{1.776} & 1.272          & 1.534 \\
Math        & 1.450          & 1.258          & 1.355          & 1.616          & \textbf{2.379} & 2.046 \\ \bottomrule
\end{tabular}%
}
\vspace{-5pt}
\end{table*}

\begin{table*}[t]
\centering
\caption{Performance of MetaSpS under stochastic decoding (temperature = 1.0). Speedup ratios are reported relative to standard autoregressive greedy decoding.}
\label{tab:stochastic_blackbox}
\resizebox{0.7\linewidth}{!}{%
\begin{tabular}{l|ccccc|cc}
\toprule
\textbf{Task} & \textbf{Drafter1} & \textbf{Drafter2} & \textbf{Drafter3} & \textbf{Drafter4} & \textbf{Drafter5} & \textbf{OFA} & \textbf{MetaSpS-UCB} \\ \midrule
Code         & 1.781 & 0.963 & 1.168 & 1.260 & 1.178 & 1.501 & \textbf{1.596} \\
Translation  & 0.856 & 1.695 & 0.897 & 0.880 & 0.838 & 0.861 & \textbf{1.197} \\
CNN          & 1.201 & 0.918 & 1.629 & 1.223 & 1.092 & 1.230 & \textbf{1.439} \\
NQA          & 1.073 & 0.961 & 1.132 & 1.510 & 1.031 & 1.123 & \textbf{1.322} \\
MathQA       & 1.220 & 1.026 & 1.200 & 1.360 & 1.968 & 1.512 & \textbf{1.673} \\ 
\bottomrule
\end{tabular}%
}
\end{table*}

\begin{table*}[t!]
\centering
\caption{Performance of MetaEagle under stochastic decoding (temperature = 1.0). Speedup ratios are reported relative to standard autoregressive greedy decoding.}
\label{tab:stochastic_whitebox}
\resizebox{0.7\linewidth}{!}{%
\begin{tabular}{l|ccccc|cc}
\toprule
\textbf{Task} & \textbf{EAGLE1} & \textbf{EAGLE2} & \textbf{EAGLE3} & \textbf{EAGLE4} & \textbf{EAGLE5} & \textbf{OFA Eagle} & \textbf{MetaEagle-UCB} \\ \midrule
Code         & 3.019 & 0.872 & 1.012 & 1.348 & 1.809 & 2.926 & \textbf{2.765} \\
Translation  & 1.030 & 1.817 & 1.373 & 1.278 & 0.997 & 1.578 & \textbf{1.668} \\
CNN          & 0.998 & 0.834 & 2.289 & 1.267 & 0.864 & 1.749 & \textbf{1.935} \\
NQA          & 1.179 & 0.922 & 1.269 & 2.181 & 1.011 & 1.680 & \textbf{1.756} \\
MathQA       & 1.739 & 0.966 & 1.462 & 2.141 & 3.099 & 2.289 & \textbf{2.539} \\ 
\bottomrule
\end{tabular}%
}
\end{table*}

We investigate the impact of temperature sampling on MetaSpS performance.  \Autoref{tab:temperature} presents the speedup ratios achieved with temperature sampling with temperature 0.7 on an NVIDIA A6000 GPU.  Consistent with the trends observed in our main experiments with greedy decoding, MetaSD continues to achieve competitive speedup.

To further assess the robustness of MetaSD under stochastic decoding conditions, we conduct experiments with temperature sampling ($T=1.0$). The results, presented in \Autoref{tab:stochastic_blackbox} and \Autoref{tab:stochastic_whitebox}, align with our theoretical assumptions, demonstrating the adaptability of MetaSD even in non-deterministic decoding settings. These results further support the theoretical alignment of MetaSD with stochastic decoding conditions.

\subsection{Throughput evaluation in single-batch and heterogeneous settings}  
\label{sec:throughput_exp}

We evaluate its throughput performance in both single-batch and heterogeneous batch settings. As shown in \Autoref{tab:throughput_merged}, we conduct experiments on out-of-distribution tasks—Physics QA \citep{physics}, Hotpot QA \citep{hotpotqa}, and MMLU-CoT \cite{mmlu1, mmlu2}(average across 57 tasks). These datasets provide a robust evaluation of MetaSD’s ability to generalize beyond training domains.

In single-batch settings, MetaEagle-UCB consistently outperforms OFA Eagle, achieving higher throughput across all evaluated tasks. This improvement stems from its adaptive selection mechanism, which efficiently routes queries to the most suitable drafter, even in OOD scenarios.  

For heterogeneous batch throughput, where different tasks are mixed within a batch, MetaEagle-UCB remains competitive. While throughput slightly decreases in heterogeneous settings due to increased I/O from drafter switching, its performance remains comparable to single-drafter methods in small-batch inference. These findings highlight MetaSD’s ability to balance adaptive specialization with high-throughput efficiency, making it well-suited for real-world deployment in diverse task distributions.

\begin{table*}[ht!]
\centering
\caption{Throughput in single-batch and heterogeneous batch settings.}
\label{tab:throughput_merged}
\resizebox{0.95\linewidth}{!}{%
\begin{tabular}{l|cc|ccc}
\toprule
\multirow{2}{*}{\textbf{Task / Batch Size}} & \multicolumn{2}{c|}{\textbf{Single Batch Throughput}} & \multicolumn{3}{c}{\textbf{Heterogeneous Batch Throughput}} \\ 
\cmidrule(lr){2-3} \cmidrule(lr){4-6} 
 & \textbf{OFA Eagle} & \textbf{MetaEagle-UCB} & \textbf{Single OFA Eagle} & \textbf{Homogeneous (MetaEagle-UCB)} & \textbf{Heterogeneous (MetaEagle-UCB)} \\ \midrule
Physics       & 2.424  & 2.573  & ---  & ---  & ---  \\
Hotpot QA     & 2.262  & 2.270  & ---  & ---  & ---  \\
MMLU-CoT (Avg. 57 tasks) & 2.466  & 2.529  & ---  & ---  & ---  \\ \midrule
Batch Size = 1  & ---  & ---  & 2.803  & 3.045  & 3.045  \\
Batch Size = 2  & ---  & ---  & 2.751  & 2.933  & 2.518  \\
Batch Size = 4  & ---  & ---  & 2.563  & 2.701  & 1.931  \\ \bottomrule
\end{tabular}%
}
\end{table*}

\subsection{Statistical verification}
All reported values in our experiments are averaged over three independent runs to support statistical reliability.
\begin{table*}[t!]\small
\centering
\caption{Mathematical terms and notations in our work. \label{tab:notation}}
\begin{tabular}{l|l}
\toprule
Notation & Descriptions \\ \midrule\midrule
$K$ & Number of drafters \\ \midrule
$[K]$ & For a given integer $K$, denotes the set $\{1, ..., K\}$ \\ \midrule
$i$ & Drafter index $i\in[K]$ \\ \midrule
$\alpha_i$ & True mean of acceptance rate when using drafter $i$ \\ \midrule
$i^{\star}$ & Drafter index with the highest $\alpha_i$ \\ \midrule
$t$ & Number of current round \\ \midrule
$B$ & Total number of tokens to generate \\ \midrule
$l(t)$ & Number of input tokens at round $t$ \\ \midrule
$x^{1:l}$ & Token sequence of first $l$ tokens \\ \midrule
$\mathcal{M}_q$ & Target model \\ \midrule
$\mathcal{M}_{q_i}$ &  The i-th drafter \\ \midrule
$p^{l}$ & Probability distribution of target model output given token sequence $x^{1:l}$\\ \midrule
$q_i^{l}$ & Probability distribution of output of drafter $i$ given token sequence $x^{1:l}$\\ \midrule
$N_{acc}(i,t)$ & Number of accepted tokens using drafter $i$ in round $t$ \\ \midrule
$N_{max}$ & Number of candidate tokens \\ \midrule
$r$ & Arbitrary reward distribution with bounded support [0,1] \\ \midrule
$r_{i,t}$ & General reward feedback using drafter $i$ in round $t$ \\ \midrule
$r_{i,t}^{BE}$ & BE reward using drafter $i$ in round $t$ \\ \midrule
$r_{i,t}^{BD}$ & BD reward using drafter $i$ in round $t$ \\ \midrule
$n_i(t)$ & Number of selecting drafter $i$ until round $t$ \\ \midrule
$a_t$ & Index of selected drafter in round $t$ \\ \midrule
$\beta$ & The exploration strength hyperparamter in UCB \\ \midrule
$\gamma$ & The exploration hyperparameter used in EXP3 \\ \midrule
$\mu_i$ & Expectation of the reward distribution of drafter $i$ \\ \midrule
$\pi$ & Bandit policy (algorithm) \\ \midrule
$\tau(\pi,B)$ & Stopping time for the policy $\pi$ with given total number of tokens $B$ \\ \midrule
$\lambda$ & Switching cost constant factor \\ \midrule
$\Delta_i$ & Suboptimality gap for the arbitrary reward distribution $r$: $\mu_i^{\star}-\mu_i$  \\ \midrule
$\Delta(\alpha_i)$ & Suboptimality gap for the BD reward: $\alpha_i^{\star}-\alpha_i$ \\ \midrule
$\Delta_i^{BE}$ & Suboptimality gap for the BE reward \\ \midrule
$R(r_i)$ & Feedback signal for reward distribution when using drafter $i$ (\Autoref{theorem:BD,BE reward comparison}) \\\midrule
$d_{TV}(\cdot, \cdot)$ & The total variation distance between probability measures \\ \midrule
$\mathbb{I}$ & Indicator function \\ \midrule
$ O(\cdot)$ & Big O notation
\\
\bottomrule
\end{tabular}
\end{table*}
\clearpage

\section{Proofs} \label{sec:proof}
To begin, we provide the mathematical terms and notations in \Autoref{tab:notation}.

\subsection{Basic lemmas}
First, we provide basic concentration inequalities which will be used to prove our theoretical results.
\begin{lemma}[Chernoff-Hoeffding bound]
\label{lemma:Chernoff-Hoeffding_bound}
Suppose there are $n$ random variables $X_1,X_2,\dots,X_n$ whose value is bounded in $[0,1]$ and $\mathbb{E}[X_t|X_1,\dots,X_{t-1}]=\mu$ for $2\leq t\leq n$. Then, for $S_n=\sum_{i=1}^{n}X_{i}$ and $a\geq0$, following inequalities holds:
\vspace{-5pt}
\begin{equation*}
\begin{aligned}
& \mathbb{P}(S_n\geq n\mu+a)\leq e^{-2a^2/n},
\\ & \mathbb{P}(S_n\leq n\mu-a)\leq e^{-2a^2/n}.
\end{aligned}
\end{equation*}
\end{lemma}

\begin{lemma}[Bernstein inequality]
\label{lemma:Bernstein_inequality}
Suppose there are $n$ random variables $X_1,X_2,\dots,X_n$ whose value is bounded in $[0,1]$ and $\sum_{t=1}^{n}\mathrm{Var}[X_t|X_{t-1},\dots,X_1]=\sigma^2$. Then, for $S_n=\sum_{i=1}^{n}X_{i}$ and $t\geq0$, following inequalities holds:
\begin{equation*}
\mathbb{P}(S_n\geq \mathbb{E}[S_n]+t)\leq \exp({-\frac{t^2}{\sigma^2+t/2}}).
\end{equation*}
\end{lemma}
\vspace{-5pt}

\subsection{Proof of Theorem~\ref{theorem:BD,BE reward comparison}}
\label{appendix:proof_of_theorem1}

In order to prove the theorem, we first provide statistics for the BE and BD rewards by the following lemmas.
\paragraph{BE reward statistics}

Here, we explicitly calculate expectation and variance of the BE reward in one round of speculative decoding. The result is presented in the following lemma.
\begin{lemma} [BE reward statistics]
\label{lemma:BEreward_stats}
The expectation and variance of the number of accepted tokens is as follows:
\begin{equation}
\begin{aligned}
&\mathbb{E}[r_{i,t}^{BE}]=\frac{\alpha_{i}-\alpha_{i}^{N_{max}+1}}{N_{max}(1-\alpha_{i})}, \\
&\mathrm{Var}[r_{i,t}^{BE}] = \frac{1}{( 
N_{max})^2(1-\alpha_i)^2}
\\ & \cdot \alpha_i
(1-(2N_{max}+1)\alpha_i^{N_{max}}
\\ & +(2N_{max}+1)\alpha_i^{N_{max}+1}-\alpha_i^{2N_{max}+1}).
\end{aligned}
\end{equation}
\end{lemma}
\paragraph{Proof of Lemma~\ref{lemma:BEreward_stats}} We first start with calculating the expectation and variance of $N_{acc}$ which can be obtained in a closed form. Suppose we conduct one round of speculative decoding for candidate token indices $l+j$ for $j=1,\dots,N_{max}$. Now, define $E_{l+j}^{i}$ as the event of $(l\!+\!j)$-th token generated by drafter $i$ is accepted in the verification stage. Also, define random variable $X_{l+j}^i$ to be 1 when $E_{l+j}^{i}$ occurs and 0 otherwise.
With the stationary assumption, one can observe  $X_{l+j}^i$ follows Bernoulli distribution with mean $\alpha_{i}$. Now, expectation can be obtained as: 
\begin{equation}
\label{eq:N_acc_expectation}
\begin{aligned}
& \mathbb{E}[N_{acc}(i,t)]=\sum_{l=1}^{N_{max}}\mathbb{E}[X_{l+j}^i]
\\ & =\sum_{l=1}^{N_{max}}{\alpha_i^{l}}=\frac{\alpha_{i}-\alpha_{i}^{N_{max}+1}}{1-\alpha_i}.
\end{aligned}
\end{equation}

To obtain variance, from $X_{L+l}^{i}\sim Ber(\alpha_i^l)$, following holds:
\[\mathrm{Var}(X_{L+l}^{i}) = (\alpha_i^{l}-\alpha_i^{2l})\]
Now, we can directly obtain a closed form of the variance by,
\begin{equation}
\begin{aligned}
& \mathrm{Var}(N_{acc}(i,t)) = \mathrm{Var}(\sum_{l=1}^{N_{max}}X_{L+l}^{i})
\\
& =  \sum_{l=1}^{N_{max}}\mathrm{Var}(X_{L+l}^{i}) + 2\cdot\sum_{l<m}\mathrm{Cov}(X_{L+l}^{i},X_{L+m}^{i})
\\
& = 2\cdot\sum_{l=1}^{N_{max}}\sum_{m=l}^{N_{max}}\mathrm{Cov}(X_{L+l}^{i},X_{L+m}^{i}) 
\\ & -  \sum_{l=1}^{N_{max}}\mathrm{Var}(X_{L+l}^{i})
\\
& = 2\cdot \sum_{l=1}^{N_{max}}\sum_{m=l}^{N_{max}}(\alpha_i^m - \alpha_i^{m+l}) 
- \sum_{l=1}^{N_{max}}(\alpha_i^{l}-\alpha_i^{2l})
\end{aligned}
\end{equation}
\begin{equation*}
\begin{aligned}
& = 2 \cdot \sum_{l=1}^{N_{max}}\sum_{m=l}^{N_{max}}\alpha_i^m - 2\cdot\sum_{l=1}^{N_{max}}\sum_{m=l}^{N_{max}}\alpha_i^{m+l}
\\ & - \frac{\alpha_{i}(1-\alpha_{i}^{N_{max}})(1-\alpha_{i}^{N_{max}+1})}{1-\alpha_{i}^2}
\\
& = 2\cdot \sum_{l=1}^{N_{max}}l\cdot\alpha_i^l 
\\ & - 2\cdot \sum_{l=1}^{N_{max}}\alpha_i^l\left(\frac{\alpha_i^l-\alpha_i^{N_{max}+1}}{1-\alpha_i}\right) 
\\ & - \frac{\alpha_{i}(1-\alpha_{i}^{N_{max}})(1-\alpha_{i}^{N_{max}+1})}{1-\alpha_{i}^2}
\\
& = 2\cdot \sum_{l=1}^{N_{max}}l\cdot\alpha_i^l - 2\cdot \frac{1}{1-\alpha_i}\sum_{l=1}^{N_{max}}\alpha_i^{2l} 
\\ & + 2\cdot \frac{\alpha_i^{N_{max}+1}}{1-\alpha_i}\sum_{l=1}^{N_{max}}(\alpha_i^l 
\\ & - \frac{\alpha_{i}(1-\alpha_{i}^{N_{max}})(1-\alpha_{i}^{N_{max}+1})}{1-\alpha_{i}^2})
\\
& = \frac{2\alpha_i(N_{max}\cdot\alpha_i^{N_{max}+1}-(N_{max}+1)\alpha_i^{N_{max}}+1)}{(1-\alpha_i)^2} 
\\ & -\frac{2\alpha_i^2(1-\alpha_i^{2N_{max}})}{(1-\alpha_i)(1-\alpha_i^2)} 
\\ & + \frac{2\alpha_i^{N_{max}+2}(1-\alpha_i^{N_{max}})}{(1-\alpha_i)^2} 
\\ & - \frac{\alpha_{i}(1-\alpha_{i}^{N_{max}})(1-\alpha_{i}^{N_{max}+1})}{1-\alpha_{i}^2}.
\end{aligned}
\end{equation*}

The second equality comes from the basic property of variance, the fourth equality is from observing $\mathrm{Cov}(X_{L+l}^{i},X_{L+m}^{i})=\mathbb{E}[X_{L+l}^{i}X_{L+m}^{i}]-\mathbb{E}[X_{L+l}^{i}]\mathbb{E}[X_{L+m}^{i}]=\alpha_i^{m}-\alpha_i^{l+m}$. After rearranging the terms, we can obtain closed form of the variance as follows.
\begin{equation}
\label{eq:N_acc_variance}
\begin{aligned}
& \mathrm{Var}(N_{acc}(i,t)) =\frac{\alpha_i}{(1-\alpha_i)^2} \cdot 
\\ & (1-(2N_{max}+1)\alpha_i^{N_{max}}
\\ & + (2N_{max}+1)\alpha_i^{N_{max}+1}-\alpha_i^{2N_{max}+1})
\end{aligned}
\end{equation}
Since $r_{i,t}^{BE}=\frac{1}{N
_{max}}N_{acc}(i,t)$ by definition, plugging this into~\Autoref{eq:N_acc_expectation} and~\Autoref{eq:N_acc_variance} concludes the proof.

\hfill \qedsymbol

\paragraph{BD reward statistics}
Next, we obtain the expectation and variance of the BD reward by following lemma. 
\begin{lemma}
\label{lemma:BD_stats_asymptotic}
Following the relationships hold for $r_{i,t}^{BD}$ for all $i,t$:

\begin{equation}
\mathbb{E}[r_{i,t}^{BD}]=\alpha_i,  \mathrm{Var}[r_{i,t}^{BD}]\leq \frac{1}{4N_{max}}
\end{equation}
\end{lemma}
\paragraph{Proof of Lemma~\ref{lemma:BD_stats_asymptotic}}
Under stationary assumption, any random variable which is bounded in $[0,1]$ has variance less than $\frac{1}{4}$. Since in~\Autoref{eq:BDreward_definition}, $r_{i,t}^{BD}$ is constructed by empirical mean of $N_{max}$ numbers of samples under stationary assumption, following holds:
\begin{equation*}
\begin{aligned}
& \mathrm{Var}[r_{i,t}] 
\\ & = \mathrm{Var}\left[\frac{1}{N_{max}}\sum_{j=0}^{N_{max}-1}(1-d_{TV}(p^{l(t)+j},q_i^{l(t)+j})\right]
\\ & \leq \frac{1}{4N_{max}},   
\end{aligned} 
\end{equation*}
and this concludes the proof. 
\hfill\qedsymbol

\paragraph{Relationship between expectations of two rewards.}
Combining above lemmas, one can show that the expectation of the BD reward is proportional to the BE reward.
\begin{lemma}
\label{lemma:relationship_between_expectations}
Following relationship holds between the expectation of the BE reward and the expectation of the BD reward:
\begin{equation}
\mathbb{E}[r_{i,t}^{BE}] = \frac{1-\alpha_i^{N_{max}}}{N_{max}(1-\alpha_i)}\mathbb{E}[r_{i,t}^{BD}].   
\end{equation}
\end{lemma}
\begin{proof}
    Combining~\Autoref{lemma:BEreward_stats} and~\Autoref{lemma:BD_stats_asymptotic} directly gives the result.
\end{proof}
Note that $\alpha_{i,t}$ can be interpreted as the acceptance rate for the $t$-th token generated by the $i$-th model~\cite{speculative_decode}.

\paragraph{Bandit feedback signal}
Next, we formally define the feedback signal with following definition.

\begin{definition}[Feedback signal]
Under stationary environment, any reward design $r_i$ with $\mu_i=\mathbb{E}[r_i]$, $i^{\star}=\argmax{\mu_i}$, and $\Delta_i=\mu_i^{\star}-\mu_i$, we define feedback signal for each suboptimal arm $i\neq i^{\star}$ as follows. 

\label{def:variance_to_deltasquare}
\begin{equation*}
R(r_i):=\frac{\Delta_i^2}{\max({\mathrm{Var}[r_i],\mathrm{Var}[r_{i^{\star}}]})}
\end{equation*}
\end{definition}
As we will see, $R$ become crucial factor that governs regret upper bound of our MetaSD-UCB algorithm. Specifically, the lower $R(r_i)$ guarantees smaller amount of regret by picking suboptimal arm $i$.

Then, we provide a formal version of \Autoref{theorem:BD,BE reward comparison} which states the BD reward actually has lower feedback signal compared to the BE reward.
\begin{theorem}[Formal version of \Autoref{theorem:BD,BE reward comparison}]
\label{theorem:bandit_signal_ratio}
Denote $\Delta(\alpha_i):= \alpha_{i^{\star}}-\alpha_i$ for any suboptimal arm $i$ and $n:=N_{max}$ for notational convenience. For any $n\in\mathcal{N}$, define functions $f_n,g_n,h_n$ on $(0,1)$ by $f_n(x)=\frac{x-x^{n+1}}{1-x}$, $g_n(x)=f_n'(x)=\sum_{s=1}^{n}sx^{s-1}$, and $h_n(x)=\sum_{s=1}^{n}s(x^{s-1}-x^{2n-s})$. Then following holds:

\begin{equation}
R(r_i^{BD}) \geq {4(\Delta(\alpha_i))^2N_{max}}.
\end{equation}

Also, following holds for any drafter configuration satisfying $h_n(\alpha_{i^{\star}})\geq \frac{g_n(\alpha_{i^{\star}})^2}{4n\alpha_{i^{\star}}}$ and
$\mathrm{Var}[r_i^{BE}]<\mathrm{Var}[r_{i^{\star}}^{BE}]$:

\begin{equation}
\label{eq:Bandit_signal_BE_lowerbound}
R(r_i^{BD}) > R(r_i^{BE}).    
\end{equation}
\end{theorem}
\begin{proof}
Upper bound for the BD reward can be directly obtained from~\Autoref{lemma:BD_stats_asymptotic}. To prove~\Autoref{eq:Bandit_signal_BE_lowerbound}, denote $N_{max}=n$ for notational convenience. Then, by directly applying~\Autoref{lemma:BEreward_stats}, it is observed that

\begin{equation}
\label{eq:proof_BEreward_bandit_signal}
\begin{aligned}
& \frac{1}{R(r_i^{BE})} = \frac{\max({\mathrm{Var}[r_i^{BE}],\mathrm{Var}[r_{i^{\star}}^{BE}]})}{\Delta_i^2}
\\
& =\frac{\alpha_{i^{\star}}(1-(2n+1)\alpha_{i^{\star}}^{n}+(2n+1)\alpha_{i^{\star}}^{n+1}-\alpha_{i^{\star}}^{2n+1})}{(f_n(\alpha_i^{\star})-f_n(\alpha_i))^2(1-\alpha_{i^{\star}})^2}
\\
& > \frac{\alpha_{i^{\star}}(1-(2n+1)\alpha_{i^{\star}}^{n}+(2n+1)\alpha_j^{n+1}-\alpha_{i^{\star}}^{2n+1})}{(g_n\left(\alpha_{i^{\star}})\Delta(\alpha_i))^2(1-\alpha_{i^{\star}}\right)^2}
\\
& = \frac{\alpha_{i^{\star}} h_n(\alpha_{i^{\star}})}{(g_n(\alpha_{i^{\star}})\Delta(\alpha_i))^2}
\\
& \geq \frac{1}{4(\Delta(\alpha_i))^2N_{max}},
\end{aligned}
\end{equation}

where the first inequality is from $f_n$ is a convex function, the second equality comes from~\Autoref{lemma:BEreward_stats}, and the last line comes from the assumption.

\end{proof}

\paragraph{Practical considerations}

While \Autoref{theorem:bandit_signal_ratio} provides a general scenario, the inequalities used in its derivation can be quite loose in certain cases.
In practice, the BD reward often exhibits a significantly smaller feedback signal $R(r_i)$ than the BE reward. For example, consider the case where $N_{max}=5$, which is the setting used in our main experiments. The condition $h_n(\alpha_{i^{\star}})>\frac{g_n(\alpha_{i^{\star}})}{4n\alpha_{i^{\star}}}$ holds for $0.06<\alpha_{i^{\star}}<0.8$, which covers most of the practical range of $\alpha_{i^{\star}}$. This implies that, in many realistic scenarios, the BD reward leads to a substantially tighter regret bound compared to the BE reward, further supporting its effectiveness in the MetaSD framework.
Moreover, assumption of $\mathrm{Var}[r_i^{BE}]<\mathrm{Var}[r_{i^{\star}}^{BE}]$ covers most of the practical scenarios. As an example, if $n=5$, $\mathrm{Var}[r_i^{BE}]$ is monotonically increasing until $\alpha_i=0.815$. Consequently, for any drafter set with $\alpha_{i^{\star}}<0.815$, $\mathrm{Var}[r_i^{BE}]<\mathrm{Var}[r_{i^{\star}}^{BE}]$ holds for all suboptimal drafters.

\subsection{Stopping time regret}
\label{appendix:N_acc and stopping time relationship}
In this subsection, we provide the equivalence relation between two objectives, maximizing the reward and minimizing the stopping time. First, we define the regret of MetaSD in terms of the stopping time. Denote $\tau(\pi,B)$ as the stopping time for any policy $\pi$ with target sequence length $B$ and $\pi^{\star}$ as the optimal policy. In~\Autoref{def:stopping_time_regret}, stopping time regret of policy $\pi$ with $B$ is defined as:
\begin{equation}
\label{eq:app_stoppingtime_regret}
\textsc{Reg}^{s}(\pi,B) = \mathbb{E}[\tau(\pi,B)] - \mathbb{E}[\tau(\pi^{\star},B)] .
\end{equation}
Intuitively, minimizing $\textsc{Reg}^{s}(\pi,B)$ should guarantee optimal speedup since minimizing $\tau(\pi,B)$ implies minimizing the number of total SD round. The following lemma proves that our reward design is well aligned with such objective.
\begin{lemma}[BE reward original regret]
\label{lemma:reward_stoppingtime_equivalence}
For any policy $\pi$ with the target sequence length $B$, denote the original regret objective using the BE reward as 
\begin{equation}
    \textsc{Reg}^{o,BE}(\pi,T)=\sum_{t=1}^{T}(\mathbb{E}[r_{i^{\star}}]-\mathbb{E}[r_{a_t}]).
\end{equation}

Then, the following equation holds:
\begin{equation}
    \textsc{Reg}^{o,BE}(\pi,T) = \frac{1}{N_{max}}\textsc{Reg}^{s}(\pi,B).
\end{equation}
Consequently, minimizing the regret in terms of accepted tokens is equivalent to minimizing $\textsc{Reg}^{(s)}(\pi,B)$.
\end{lemma}
\begin{proof}
It is observed that
\begin{equation*}
\begin{aligned}
B & = \sum_{t=1}^{\tau(B)}(N_{acc}(i,t)+1) 
\\ & = \tau(B) + \sum_{t=1}^{\tau(B)}N_{acc}(i,t)
\\ & = \tau(B) + N_{max}\sum_{t=1}^{\tau(B)}r_{a_t,t}.
\end{aligned}
\end{equation*}
which leads to,
\begin{equation}
\label{eq:reward_stoppingtime_relation}
\tau(\pi,B) - \tau(\pi^{\star},B)  = N_{max}\sum_{t=1}^{\tau(B)}(r_{a_t^{\star},t}-r_{a_t,t}),
\end{equation}
where $a_t^{\star}$ is the action from the optimal policy $\pi^{\star}$ in round $t$. By taking the expectation on both sides, we get the result.
\end{proof}

However, above result does not hold in every reward design as can be seen in the following Lemma.

\begin{lemma}[BD reward original regret]
\label{lemma:BDreward_regret_compare}
For any policy $\pi$ with the fixed target sequence length $B$, denote the original regret objective using the BE reward as $\textsc{Reg}^{o,BD}(\pi,T)=\sum_{t=1}^{T}(\mathbb{E}[r_{i^{\star}}]-\mathbb{E}[r_{a_t}])$ . Then, there exists a bandit instance with the two different policies $\pi_1,\pi_2$ such that:
\begin{equation}
\begin{aligned}
&\mathbb{E}[Reg^{o,BD}(\pi_1,B)] < \mathbb{E}[Reg^{o,BD}(\pi_2,B)],
\\
& \mathbb{E}[Reg^s(\pi_1,B)] > \mathbb{E}[Reg^s(\pi_2,B)].
\end{aligned}    
\end{equation}
\end{lemma}

\begin{proof}
Suppose we have three drafters with $\alpha_1=0.1,\alpha_2=0.5,\alpha_3=0.8$ with $N_{max}=2$. Consider $\pi_1$ as the deterministic policy where it picks the drafter $1$ for the first round and pick the drafter $3$ rest of the rounds. Also, $\pi_2$ be the policy which picks drafter $2$ for the first two rounds and drafter $3$ for the rest of the rounds. For the original regret objective, $\pi_1$ has expected regret of $0.7$ while $\pi_2$ has expected regret $0.6$. However, expectation of number of accepted tokens until first two rounds becomes $(0.1+0.1^2) + (0.8+0.8^2) = 1.55$ for $\pi_1$ and $2(0.5+0.5^2) = 1.50$ for $\pi_2$. Since policy for the rest of the rounds are the same, we can conclude that the expected stopping time of policy $\pi_1$ is less than the expected stopping time of the policy $\pi_2$. As a result, $\pi_2$ is better in terms of original regret objective and $\pi_1$ is better with stopping time regret objective.  
\end{proof}

The above lemma indicates that if we take BD reward design with original bandit regret objective in~\citep{lattimore2020bandit}, the optimal bandit algorithm may not be optimal in MetaSD algorithm. This necessitates us to define a new regret objective in~\Autoref{def:stopping_time_regret} with proper reward design.

\subsection{MetaSD-UCB with general reward}
\label{appendix:UCB_general_reward}

In this subsection, we provide a generic theorem which is stated as follows. 

\begin{theorem}[Generic regret upper bound]
\label{theorem:regret_upperbound_SD-UCB(general_ver)}
For any reward design $r$, denote $\mu_i=\mathbb{E}[r_{i,t}]$, $\Delta_i=\mu_{i^{\star}}-\mu_i$, and $i^{\star}=\argmax{\alpha_i}$. If $i^{\star}=\argmax{\mu_i}$, then there exists a constant $C',C'''>0$ such that following bound holds:
\begin{equation}
\label{eq:regret_upper_bound_general_reward}
\begin{aligned}
& \textsc{Reg}(\pi,B) <
\\ & 
\sum_{i\neq i^{\star}}\frac{8}{\Delta_i^2}(\ln{B}+\ln{(\ln(\sum_{i\neq i^{\star}}\frac{1}{\Delta_i^2}))}+C')+ C'''.   
\end{aligned}
\end{equation}
\end{theorem}

Above theorem holds for any reward design as long as the drafter with the maximum expected reward $\mathbb{E}[r_{i,t}]$ also has the highest acceptance rate $\alpha_i$. 
Since both the BD and BE rewards satisfy this condition,~\Autoref{theorem:regret_upperbound_SD-UCB(general_ver)} applies to both of the reward designs. The proof of~\Autoref{theorem:regret_upperbound_SD-UCB(general_ver)} consists of two main parts. First, given total round, we can bound the expected number of selecting suboptimal arms using the same anlysis in~\cite{auer2002finite}. Next, we get the upper bound on expected stopping time of MetaSD-UCB algorithm.

\paragraph{Bounding suboptimal selection}
Given fixed stopping time, we can bound the expectation of number of selecting suboptimal arms as follows:
\begin{lemma}[Theorem 1 from~\cite{auer2002finite}] 
\label{lemma:bounding_suboptimal_nums}
Let $n_i(t)$ be the number of pulling sub-optimal drafter ($i\neq i^{\star}$) by the MetaSD-UCB until round $t$. Also, denote $\Delta_{i}:=\mu_{i^{\star}}^r-\mu_{i}^r$ be the sub-optimal gap. Then, following inequality holds for $\beta=1:$ 
\begin{equation}
\mathbb{E}[n_i(\tau(B))|\tau(B)]\leq \frac{8\ln{\tau(B)}}{\Delta_{i}^2}+1+\frac{\pi^2}{3}.
\end{equation}
\end{lemma}

\paragraph{Proof of Lemma~\ref{lemma:bounding_suboptimal_nums}}
For the analysis, we restate the proof in \cite{auer2002finite} for MetaSD-UCB algorithm with our notations. One can observe $n_i(\tau(B))$, the number of times drafter $i$ is chosen for the one round of speculative decoding until the end of generation, can be bounded as follows:
\begin{equation}
\label{eq:n_i_bounding_proof}
\begin{aligned}
& n_i(\tau(B)) = 1+\sum_{t=K+1}^{\tau(B)}\mathbb{I}[a_t=i]
\\
& \leq l + \sum_{t=K+1}^{\tau(B)}\mathbb{I}[a_t=i, n_i(t-1)\geq l]
\\
& \leq l + \sum_{t=K+1}^{\tau(B)}\mathbb{I}
[\hat{\mu}_{i,{t-1}}+\sqrt{\frac{2\ln{(t-1)}}{n_i(t-1)}}\geq \hat{\mu}_{i^{\star},{t-1}}
\\ & +\sqrt{\frac{2\ln{(t-1)}}{n_{i^{\star}}(t-1)}},n_i(t-1)\geq l]
\\
& \leq l + \sum_{t=1}^{\tau(B)}\sum_{s=1}^{t-1}\sum_{{n_i}=l}^{t-1} \mathbb{I}
\\ &\left[\hat{\mu}_{i,{n_{i}}}+\sqrt{\frac{2\ln{(t-1)}}{n_i}}\geq \hat{\mu}_{i^{\star},{s}}+\sqrt{\frac{2\ln{(t-1)}}{s}} \right].
\end{aligned}
\end{equation}
Here, $\mathbb{I}$ is an indicator function and $l$ is a positive integer. Now, one can see following holds:
\begin{equation*}
\begin{aligned}
& 
\mathbb{P}\left(\hat{\mu}_{i,{n_{i}}}+\sqrt{\frac{2\ln{t}}{n_i}}\geq \hat{\mu}_{i^{\star},s}+\sqrt{\frac{2\ln{t}}s}\right)
\\
& \leq  \mathbb{P}\left(\hat{\mu}_{i^{\star},s} \leq \mu_{i^{\star}}-\sqrt{\frac{2\ln{t}}{s}}\right) 
\\ & + \mathbb{P}\left(\hat{\mu}_{i.n_i} \geq \mu_i+\sqrt{\frac{2\ln{t}}{n_i}} \right) 
\\ & + \mathbb{P}\left(\mu_{i^{\star}} < \mu_i + 2\cdot\sqrt{\frac{2\ln{t}}{n_i}} \right).
\end{aligned}
\end{equation*}
First term and the second term in the above equation is bounded by~\Autoref{lemma:Chernoff-Hoeffding_bound} as:
\begin{equation}
\label{eq:ucbproof_auer_concentration}
\begin{aligned}
& \mathbb{P}\left(\hat{\mu}_{i^{\star},s} \leq \mu_{i^{\star}}-\sqrt{\frac{2\ln{t}}{s}}\right)\leq \exp(-4\ln{t})=t^{-4},
\\
& \mathbb{P}\left(\hat{\mu}_{i.n_i} \geq \mu_i+\sqrt{\frac{2\ln{t}}{n_i}} \right) \leq \exp(-4\ln{t})=t^{-4}.
\end{aligned}
\end{equation}
By choosing $l=\lceil\frac{8\ln{\tau(B)}}{\Delta_i^2}\rceil$, one can see that the last term is 0 since,
\begin{equation}
\label{eq:ucbproof_auer_concentration_2}
2\cdot\sqrt{\frac{2\ln{t}}{n_i}} \leq 2\cdot \sqrt{\frac{2\ln{t}}{(\frac{8\ln{\tau(B)}}{\Delta_i^2})}} \leq \Delta_i. 
\end{equation}

Finally, taking expectation of~\Autoref{eq:n_i_bounding_proof} and put the above result, one can see that:
\begin{equation}
\label{eq:proof_expected_number_bound}
\begin{aligned}
& \mathbb{E}[n_i(\tau(B))|\tau(B)]
\\ & \leq \lceil\frac{8\ln{\tau(B)}}{\Delta_i^2}\rceil 
+ 2\sum_{t=1}^{\tau(B)}\sum_{s=1}^{t-1}\sum_{{n_i}=l}^{t-1}2t^{-4}
\\
& \leq \lceil\frac{8\ln{\tau(B)}}{\Delta_i^2}\rceil + 2\sum_{t=1}^{\infty}\sum_{s=1}^{t-1}\sum_{{n_i}=l}^{t-1}2t^{-4}
\\
& \leq \frac{8\ln{\tau(B)}}{\Delta_i^2} +  1 + \frac{\pi^2}{3}.
\end{aligned}
\end{equation}
\hfill\qedsymbol

\paragraph{Bounding stopping time}
\label{appendix:stoppingtime_bound}
The overall structure of the proof in bounding the stopping time is based on the proof of \Autoref{lemma:Bernstein_inequality} in~\cite{ding2013multi} while we provide additional details that suits with our problem formulation. First, we obtain upper bound on stopping time by following lemma:
\begin{lemma} 
\label{lemma:stopping_time_upper_bound}
Following inequalities holds for some constants $C',C''>0:$
\begin{equation*}
\label{eq:lemma3_stopping_upperbound}
\begin{aligned}
& \mathbb{E}[\tau(\pi,B)] \leq \frac{B(1-\alpha_{i^{\star}})}{1-\alpha_{i^{\star}}^{N_{max}+1}} 
\\ & + \sum_{i\neq i^{\star}} \frac{8}{\Delta_i^2}(\ln{B}+\ln{(\ln(\sum_{i\neq i^{\star}}\frac{1}{\Delta_i^2}))}+C')+ C''.
\end{aligned}
\end{equation*}
\end{lemma}

In order to prove~\Autoref{lemma:stopping_time_upper_bound}, we first present two lemmas for bounding stopping time for a single armed bandit process i.e., we play only the single arm consecutively until the end of the round. Then, we provide how can we decouple stopping time of multi-armed bandit process of UCB policy.

\begin{lemma}
\label{lemma:single_process_stopping_time}
Let $\tau(\pi^{i},B)$ be a stopping time for the single armed bandit process $\pi^{i}$ which chooses only same drafter $i$ throughout the generation (i.e. $a_t=i$ for all $t$). Then the stopping time can be bounded as:
\begin{equation}
\label{eq:append_proof_singlebandit_process}
\frac{B(1-\alpha_i)}{1-\alpha_i^{N_{max}+1}} - 1 < \mathbb{E}[\tau(\pi^{i},B)] \leq \frac{(B+1)(1-\alpha_i)}{1-\alpha_i^{N_{max}+1}}.
\end{equation}
\end{lemma}
\begin{proof}
One can see the expected number of generated tokens in each round is $\mu_i^c=\frac{1-\alpha_i^{N_{max+1}}}{1-\alpha_i}$ and the remaining number of tokens in the last round is contained in $\{1,2,\cdots,N_{max}\}$. Now, suppose~\Autoref{eq:append_proof_singlebandit_process} holds for all $B<B_0$. Then one can observe: 
\begin{equation*}
\begin{aligned}
& \mathbb{E}[\tau(\pi^{i},B_0)] 
\\ & = \mathbb{E}\left[\sum_{j=0}^{N_{max}}\left(\tau(\pi^{i},B_0-1-j)+1\right)\mathbb{P}[r_{i}^{BE}=j]\right] 
\\
& \leq \sum_{j=0}^{N_{max}}\frac{(B_0-j)(1-\alpha_i)}{1-\alpha_i^{N_{max}+1}}\mathbb{P}[r_{i}^{BE}=j]  + 1
\\
& \leq \sum_{j=0}^{N_{max}}\frac{(B_0+1)(1-\alpha_i)}{1-\alpha_i^{N_{max}+1}}\mathbb{P}[r_{i}^{BE}=j] 
\\ & -\frac{(1-\alpha_i)}{1-\alpha_i^{N_{max}+1}}\mathbb{E}[r_{i}^{BE}=j]  + 1
\\
& = \sum_{j=0}^{N_{max}}\frac{(B_0+1)(1-\alpha_i)}{1-\alpha_i^{N_{max}+1}}\mathbb{P}[r_{i}^{BE}=j].
\end{aligned}
\end{equation*}
Since it is trivial to see that~\Autoref{eq:append_proof_singlebandit_process} holds for $B=1$, by mathematical induction, one can conclude the proof. The lower bound can be proved by the exactly same manner as in the upper bound.

\end{proof}
Now, we propose a lemma which provides an upper bound on expected stopping time.

\begin{lemma}
\label{lemma:bounding_stopping_time}
For MetaSD-UCB algorithm $\pi$ with given token  target sequence length $B$, expectation of stopping time $\tau(B)$ can be bounded as follows:
\begin{equation}
\label{eq:stopping_time_upperbound}
\begin{aligned}
\mathbb{E}[\tau(B)] \leq \mathbb{E}[\tau(\pi^{i^{\star}},B)] +  \sum_{i\neq i^{\star}}\mathbb{E}[n_{i}(\pi,B)],
\end{aligned}
\end{equation}

where, $n_i(\pi,B)$ is number of selecting drafter $i$ by policy $\pi$ during the generation.
\end{lemma}

\begin{proof}
We first prove the upper bound (\Autoref{eq:stopping_time_upperbound}). For policy $\pi$ with the target sequence length $B$, define a corresponding process $\pi^{u}$ which is defined by extending the process with the new stopping time, which is:
\begin{equation}
    \begin{aligned}
        & \tau^{u}(\pi^{u}
,B) = min\{\tau > 0 
\\ & \mid \sum_{t=1}^{\tau}(N_{acc}(a_t^{u},t)+1)\cdot\mathbb{I}[a_t^{u}=i^{\star}] \geq B \}.
    \end{aligned}
\end{equation}

where, $a_t^u = a_t$ for $t\leq \tau(B)$ and $a_t^u = i^{\star}$ for $\tau(B) < t \leq \tau^u(\pi^u,B)$.
In other words, $\tau^{u}(\pi^{u},B)$ is the time where total number of generated tokens by optimal drafter exceeds $B$. Then, one can see from the construction of $\pi^{u}$ and by observing that $\tau^{u}$ does not depend on thenumber of tokens generated by suboptimal drafters, 
\begin{equation}
\mathbb{E}[n_{i^{\star}}(\pi,B)] \leq \mathbb{E}[n_{i^{\star}}(\pi^{u},B)] = \mathbb{E}[\tau(\pi^{i^{\star}},B)]. 
\end{equation}

\end{proof}

\paragraph{Proof of Lemma~\ref{lemma:stopping_time_upper_bound}}
To prove the upper bound, from~\Autoref{lemma:bounding_suboptimal_nums} and~\Autoref{lemma:bounding_stopping_time}, it can be shown that

\begin{equation}
\label{eq:append_proof_boundingsuboptimal_upperbound}
\begin{aligned}
& \mathbb{E}[\tau(B)]  \leq \mathbb{E}[\tau(\pi^{i^{\star}},B)] +  \sum_{i\neq i^{\star}}\mathbb{E}[n_{i}(\pi,B)]
\\
& \leq \frac{(B+1)(1-\alpha_{i^{\star}})}{1-\alpha_{i^{\star}}^{N_{max}+1}} 
+ \sum_{i\neq i^{\star}}\mathbb{E}[n_{i}(\pi,B)]
\\
& \leq \frac{(B+1)(1-\alpha_{i^{\star}})}{1-\alpha_{i^{\star}}^{N_{max}+1}} 
+\sum_{i\neq i^{\star}}\frac{8}{\Delta_i^2}\mathbb{E}[\ln{\tau(B)}] 
\\ & +  (K-1)(1 + \frac{\pi^2}{3}),
\\
& \leq \frac{(B+1)(1-\alpha_{i^{\star}})}{1-\alpha_{i^{\star}}^{N_{max}+1}} 
+\sum_{i\neq i^{\star}}\frac{8}{\Delta_i^2}\ln{\mathbb{E}[\tau(B)]} 
\\ & +  (K-1)(1 + \frac{\pi^2}{3}),
\end{aligned}
\end{equation}
where the second inequality holds from~\Autoref{lemma:single_process_stopping_time}, the third inequality holds by~\Autoref{lemma:bounding_suboptimal_nums},  and the last inequality holds from Jensen's inequality. Now, using $\ln(x)\leq \frac{x}{\epsilon} + \ln(\epsilon) - 1$ and taking $\epsilon = \sum_{i\neq i^{\star}}\frac{16}{\Delta_i^2}$, one can obtain:
\begin{equation*}
\begin{aligned}
\mathbb{E}[\tau(B)] & \leq \frac{(2B+2)\cdot(1-\alpha_{i^{\star}})}{1-\alpha_{i^{\star}}^{N_{max}+1}} \\
 &+ 2\ln(\sum_{i\neq i^{\star}}\frac{16}{\Delta_i^2})-2 +
 \\ & (2K-2)(1+\frac{\pi^2}{3}).
\\
\end{aligned}
\end{equation*}
If we again put the above equation into the~\Autoref{eq:append_proof_boundingsuboptimal_upperbound}, one can obtain:
\begin{equation*}
\label{eq:stoppingtime_upperbound}
\begin{aligned}
& \mathbb{E}[\tau(B)]
 \leq \frac{(B+1)(1-\alpha_{i^{\star}})}{1-\alpha_{i^{\star}}^{N_{max}+1}} 
 \\ + & \sum_{i\neq i^{\star}} \frac{8}{\Delta_i^2}\ln\frac{(2B+2)\cdot(1-\alpha_{i^{\star}})}{1-\alpha_{i^{\star}}^{N_{max}+1}} 
\\ & +{2\ln(\sum_{i\neq i^{\star}}\frac{1}{\Delta_i^2})}+C_1 + C_2
\\
& \leq \frac{B(1-\alpha_{i^{\star}})}{1-\alpha_{i^{\star}}^{N_{max}+1}} + \sum_{i\neq i^{\star}} \frac{8}{\Delta_i^2}\ln{B}
\\ & +\ln{(\ln(\sum_{i\neq i^{\star}}\frac{1}{\Delta_i^2}))}+C'+C'',
\end{aligned}
\end{equation*}
where $C_1,C_2,C',C''>0$ are constants that are independent of $B$ and $\Delta_i$.

\hfill\qedsymbol

\paragraph{Proof of Theorem~\ref{theorem:regret_upperbound_SD-UCB(general_ver)}}
The theorem is proved by observing:
\begin{equation*}
\begin{aligned}
& \mathbb{E}[\tau(\pi,B)]-\mathbb{E}[\tau(\pi^{\star},B)])
\\ & 
= \mathbb{E}[\tau(\pi,B)]-\mathbb{E}[\tau(\pi^{i^{\star}},B)]
\\
& \leq \frac{B(1-\alpha_{i^{\star}})}{1-\alpha_{i^{\star}}^{N_{max}+1}} 
\\ & + \sum_{i\neq i^{\star}} \frac{8}{\Delta_i^2}(\ln{B}+\ln{(\ln(\sum_{i\neq i^{\star}}\frac{1}{\Delta_i^2}))}+C')
\\ & + C'' - \mathbb{E}[\tau(\pi^{i^{\star}},B)]) 
\\
& < \frac{B(1-\alpha_{i^{\star}})}{1-\alpha_{i^{\star}}^{N_{max}+1}} 
\\ & + \sum_{i\neq i^{\star}} \frac{8}{\Delta_i^2}(\ln{B}+\ln{(\ln(\sum_{i\neq i^{\star}}\frac{1}{\Delta_i^2}))}+C')
\\ & + C'' -
\frac{B(1-\alpha_{i^{\star}})}{1-\alpha_{i^{\star}}^{N_{max}+1}} - 1 
\\
& < \sum_{i\neq i^{\star}} \frac{8}{\Delta_i^2}(\ln{B}+\ln{(\ln(\sum_{i\neq i^{\star}}\frac{1}{\Delta_i^2}))}+C')+ C'''.
\end{aligned}
\end{equation*}
Here, $C',C'''>0$ are constants independent of $B$ and $\Delta_i$. The first equality comes from~\Autoref{lemma:reward_stoppingtime_equivalence}, the first inequality is from~\Autoref{lemma:stopping_time_upper_bound}, and the second inequality holds by putting $i^{\star}$ to the lower bound of~\Autoref{lemma:single_process_stopping_time}.

\hfill\qedsymbol

Note that above analysis holds for every $\beta>0$ in~\Algautoref{alg:ucb1}. However, when the target sequence length $B$ is finite, constant terms in the regret bound becomes important which makes the performance of the algorithm dependent on $\beta$. We empirically found the optimal $\beta$ in our experiments. We provide further discussion on using different $\beta$ in~\Autoref{appendix:beta_analysis}.

\subsection{Proof of Theorem~\ref{theorem:regret_upperbound_SD-UCB}}
\label{appendix:proof of theorem2}

\paragraph{Concentration inequality}

Denote empirical mean of the BD and BE rewards as follows.
\begin{equation*}
\begin{aligned}
& \mu_{i,t}^{BD}=\frac{1}{n_i(t)}\sum_{\tau=1}^{t}r_{i,\tau}\cdot\mathbb{I}[a_\tau=i],
\\ & \mu_{i,t}^{BE}=\frac{1}{n_i(t)N_{max}}\sum_{\tau=1}^{t}N_{acc}(i,t)\cdot\mathbb{I}[a_\tau=i],
\end{aligned}
\end{equation*}
where $n_i(t)$ is number of times drafter $i$ is selected until round $t$ and $\mathbb{I}$ is indicator function.

Then, following inequalities can be derived for $\epsilon>0$:
\begin{equation}
\label{eq:BE_reward_concentration}
\begin{aligned}
&\mathbb{P}\left(\hat{\mu}_i^{BE} 
\geq \frac{\alpha_{i}-\alpha_{i}^{N_{max}+1}}{N_{max}(1-\alpha_{i})}+\epsilon\right) 
\\ & \leq \exp\left(-\frac{n_i(t)\epsilon^2}{2Var[r_{i}^{BE}]+\epsilon}\right),    
\end{aligned}   
\end{equation}

\begin{equation}
\label{eq:BD_reward_concentration_Hoeffding}
\mathbb{P}\left(\hat{\mu}_i^{BD} \geq \alpha_i+\epsilon\right) \leq \exp\left(-2(N_{max})n_i(t)\epsilon^2\right).    
\end{equation}

\Autoref{eq:BE_reward_concentration} comes from combining Bernstein's inequality (\Autoref{lemma:Bernstein_inequality}) with~\Autoref{lemma:BEreward_stats} and~\Autoref{eq:BD_reward_concentration_Hoeffding} is from combining Hoeffding's inequality (\Autoref{lemma:Chernoff-Hoeffding_bound}) with~\Autoref{lemma:BD_stats_asymptotic}.

\paragraph{Bandit algorithm guarantee}
Using concentration inequalities for both rewards, we provide how the bandit signal defined in~\Autoref{eq:bandit_signal} directly related to our algorithm~\Autoref{alg:ucb1}. In the proof of~\Autoref{theorem:regret_upperbound_SD-UCB(general_ver)}, one can observe that bounding number of suboptimal arm selection (\Autoref{lemma:bounding_suboptimal_nums}) directly related to the regret under the new regret object defined by stopping time (\Autoref{def:stopping_time_regret}). Leveraging above results, the regret upper bound for MetaSD-UCB algorithm with the BD and BE rewards can be proved.

\paragraph{Proof of Theorem~\ref{theorem:regret_upperbound_SD-UCB}}
For the BD reward, by putting $\beta=\frac{1}{\sqrt{N_{max}}}$ in the UCB algorithm and apply~\Autoref{eq:BD_reward_concentration_Hoeffding}, one can directly observe~\Autoref{eq:ucbproof_auer_concentration} becomes:

\begin{equation}
\begin{aligned}
& \mathbb{P}\left(\hat{\mu}_{i^{\star},s} \leq \mu_{i^{\star}}-\frac{1}{\sqrt{N_{max}}}\cdot\sqrt{\frac{2\ln{t}}{s}}\right)
\\ & \leq \exp(-4\ln{t})=t^{-4},
\\
& \mathbb{P}\left(\hat{\mu}_{i.n_i} \geq \mu_i+\frac{1}{\sqrt{N_{max}}}\cdot\sqrt{\frac{2\ln{t}}{n_i}} \right) 
\\ & \leq \exp(-4\ln{t})=t^{-4}.
\end{aligned}
\end{equation}
By choosing $l=\lceil\frac{8\ln{\tau(B)}}{(N_{max})\Delta(\alpha_i)^2}\rceil$, one can see for $n_i \geq l:$
\begin{equation*}
\frac{2}{\sqrt{N_{max}}}\cdot\sqrt{\frac{2\ln{t}}{n_i}} \leq \Delta_i. 
\end{equation*}
Rest of the proof is same as in~\Autoref{theorem:regret_upperbound_SD-UCB(general_ver)} and we can obtain:
\begin{equation*}
\begin{aligned}
& \textsc{Reg}(B) \leq \sum_{i\neq i^{\star}} \frac{8}{(N_{max})\Delta(\alpha_i)^2}\cdot 
\\ & \left(\ln{B}+\ln{(\ln(\sum_{i\neq i^{\star}}\frac{1}{\Delta_i^2}))}+C'\right)+ C,   
\end{aligned}
\end{equation*}
for some constants $C>0$ and this concludes the proof of~\Autoref{theorem:regret_upperbound_SD-UCB}.

\hfill\qedsymbol

\paragraph{BE reward regret}
For MetaSD-UCB algorithm with BE reward, we can obtain regret upper bound by the following theorem.

\begin{theorem}
\label{theorem:BEreward_Regret_Bound}
Define $\Delta_i^{BE}:= \mu_{i^{\star}}^{BE}-\mu_i^{BE}$ where $\mu_i^{BE}=\mathbb{E}[r_i^{BE}]$. If $\mathrm{Var}[r_i^{BE}]<\mathrm{Var}[r_{i^{\star}}^{BE}]$, we can obtain the following regret upper bound for the MetaSD-UCB algorithm using BE reward:
\begin{equation}
\label{eq:regret_upperbound_BEreward}
\begin{aligned}
& \textsc{Reg}(\pi^{BE},B) \leq \sum_{i\neq i^{\star}}\frac{(32\mathrm{Var}[r_{i^{\star}}^{BE}]+16)}{(\Delta_i^{BE})^2} \cdot 
\\ & \left(\ln{B}+\ln{(\ln(\sum_{i\neq i^{\star}}\frac{1}{\Delta_i^2}))}+C'\right) + C,   
\end{aligned}
\end{equation}
where $C,C'>0$ are constants independent of $B, \Delta_i^{BE}$.
\end{theorem}

\begin{proof}
From~\Autoref{eq:BE_reward_concentration}, one can similarly modify the original proof of the UCB~\citep{auer2002finite}.

Then, putting $\epsilon = \sqrt{(8\mathrm{Var}[r_{i^{\star}}^{BE}]+4)\ln{t}}$ into~\Autoref{eq:BE_reward_concentration} make~\Autoref{eq:ucbproof_auer_concentration} becomes:
\begin{equation}
\begin{aligned}
& \mathbb{P}\left(\hat{\mu}_{i^{\star},s} \leq \mu_{i^{\star}}-\sqrt{\frac{(8\mathrm{Var}[r_{i^{\star}}^{BE}]+4)\ln{t}}{s}}\right)
\\ & \leq \exp(-4\ln{t})=t^{-4},
\\
& \mathbb{P}\left(\hat{\mu}_{i.n_i} \geq \mu_i+\sqrt{\frac{(8\mathrm{Var}[r_{i^{\star}}^{BE}]+4)\ln{t}}{n_i}}\right) \leq 
\\ & \exp(-4\ln{t})=t^{-4}.
\end{aligned}
\end{equation}
By choosing $l=\lceil\frac{(32\mathrm{Var}[r_{i^{\star}}^{BE}]+16)\ln{\tau(B)}}{(\Delta_i^{BE})^2}\rceil$, one can see for $n_i \geq l:$
\begin{equation*}
2\cdot \sqrt{\frac{(8\mathrm{Var}[r_{i^{\star}}^{BE}]+4)\ln{t}}{n_i}} \leq \Delta_i^{BE}.
\end{equation*}

Rest of the proof is similar as in~\Autoref{theorem:regret_upperbound_SD-UCB(general_ver)}.
\end{proof}

\paragraph{Regret comparison}
We restate the~\Autoref{corollary:sample_complexity_ratio} formally as follows:
\begin{corollary}
For any $n\in\mathcal{N}$, define functions $f_n,g_n,h_n$ on $(0,1)$ by $f_n(x)=\frac{x-x^{n+1}}{1-x}$, $g_n(x)=f_n'(x)=\sum_{s=1}^{n}sx^{s-1}$, and $h_n(x)=\sum_{s=1}^{n}s(x^{s-1}-x^{2n-s})$. If $h_n(\alpha_{i^{\star}})\geq \frac{g_n(\alpha_{i^{\star}})^2}{4n\alpha_{i^{\star}}}$ and
$\mathrm{Var}[r_i^{BE}]<\mathrm{Var}[r_{i^{\star}}^{BE}]$,
then the regret of our algorithm $\pi^{BE}$ with the BE reward feedback is upper bounded by some function $f(B)$, where $f(B)>\frac{8}{(N_{max})(\Delta(\alpha_i))^2}\ln{B}$.
\end{corollary}

\begin{proof}
One can observe:
\begin{equation*}
\begin{aligned}
\frac{(32\mathrm{Var}[r_i^{BE}]+16)}{(\Delta_i^{BE})^2} &\geq\frac{(32\mathrm{Var}[r_i^{BE}])}{(\Delta_i^{BE})^2} \\ & > \frac{16}{\Delta(\alpha_i)^2(N_{max})},
\end{aligned}    
\end{equation*}
where first inequality comes from~\Autoref{theorem:bandit_signal_ratio}.
Now, putting above result with~\Autoref{theorem:regret_upperbound_SD-UCB} and~\Autoref{theorem:BEreward_Regret_Bound}, we get the result.
\end{proof}

Note that the better regret upper bound does not always guarantee the better performance since sometimes it is a proof artifact. Since we take quite loose inequalities during the proof of~\Autoref{theorem:BEreward_Regret_Bound}, we can improve the constant factors for BE reward. Still, even with assuming we can use~\Autoref{lemma:Chernoff-Hoeffding_bound} inequality in BE reward (which has better guarantee then Bernstein's inequality), the result of~\Autoref{corollary:sample_complexity_ratio} still holds which shows the distinction between two reward designs in terms of regret as in~\Autoref{theorem:bandit_signal_ratio}.

\subsection{Assumption on model alignment}
\label{app:assumption_on_model_alignment}

Here, we formally define the assumption on the acceptance rate which is used throughout our analysis.

\begin{assumption}
\label{assumption}
Denote $\alpha_{i,t}$ as the acceptance rate for $t$-th token generated by $i$-th model. Then, for any instance of $x^{1:B}$ generated by the target model, $\alpha_{i,t}$'s are i.i.d. from a distribution $\nu_i$ with expectation $\alpha_i$. In other words, following holds for all drafter  $i\in[K]$.
\begin{equation}
\begin{aligned}
& \alpha_{i,t}=1-d_{TV}\left(p^{t}(\cdot|x^{1:t-1}),q_i^{t}(\cdot|x^{1:t-1})\right) \overset{\text{i.i.d.}}{\sim}  \nu_i, 
\\ & \mathbb{E}[{\alpha_{i,t}}] = \alpha_i.    
\end{aligned}
\end{equation}
\end{assumption}
Above assumption shows that the acceptance rate for each token only depends on the drafter index $i$. We empirically verify the validity of the assumption by observing the TV distance between a target model and a drafter is well concentrated (\ref{subsec:mab_settings}). 

\paragraph{Experimental evidence}
To further validate \Autoref{assumption}, we provide empirical observations of the BD rewards in actual speculative decoding scenarios. For the experiment, we follow the same experimental setup of multilingual translational task as in~\Autoref{sec4:exp}. We set the target task of Japanese to English translation, where we plot of the statistics when using individual specialized drafters used in~\Autoref{sec4:exp}. For each experiment, we generate 80 instances where we calculate the mean and the standard deviation of the BD rewards for each SD round number We take the value over the first 10 rounds considering the fact that few instances take more than 10 rounds of SD.

\Autoref{fig:bd-reward-plots} clearly supports the stationary assumption made in~\Autoref{assumption} where mean and the variance of the BD reward behaves similar along the speculative decoding round moves on. Here, $[\mu-s,\mu+s]$ is colorized from the mean $\mu$ and the standard deviation $s$.

\begin{figure*}[!ht]
    \centering
    \begin{subfigure}[b]{0.45\linewidth}
        \includegraphics[width=\linewidth]{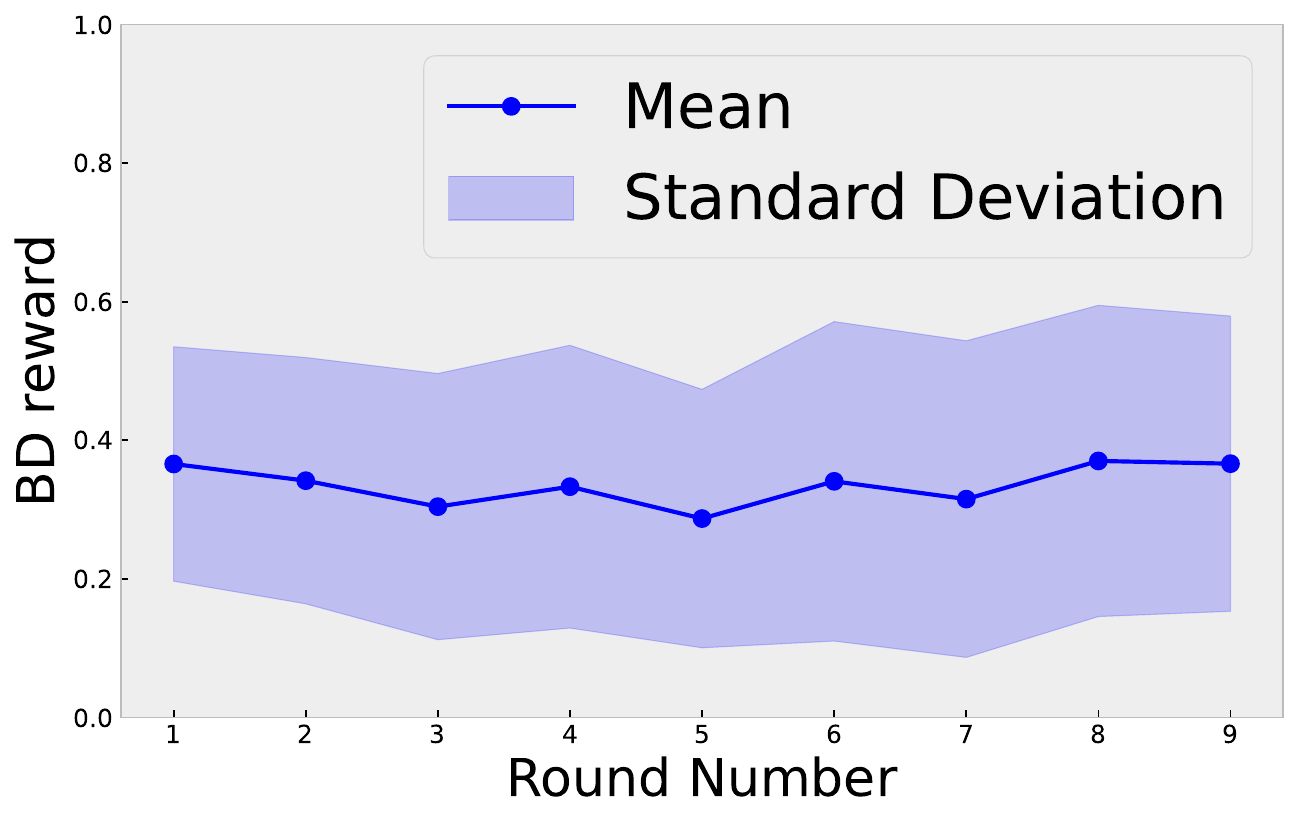}
        \caption{Drafter 1}
        \label{fig:plot1}
    \end{subfigure}
    \hfill 
    \begin{subfigure}[b]{0.45\linewidth}
        \includegraphics[width=\linewidth]{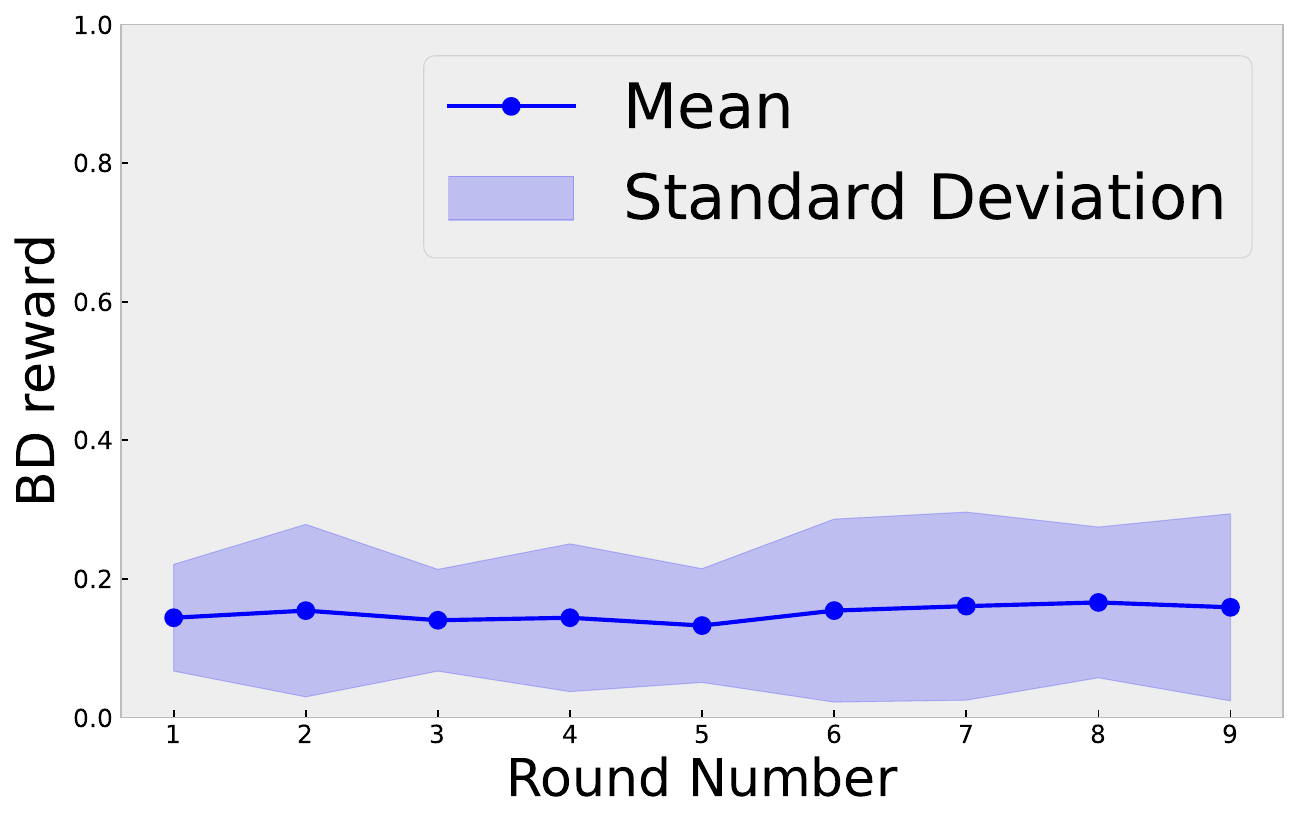}
        \caption{Drafter 2}
        \label{fig:plot2}
    \end{subfigure}

    \begin{subfigure}[b]{0.45\linewidth}
        \includegraphics[width=\linewidth]{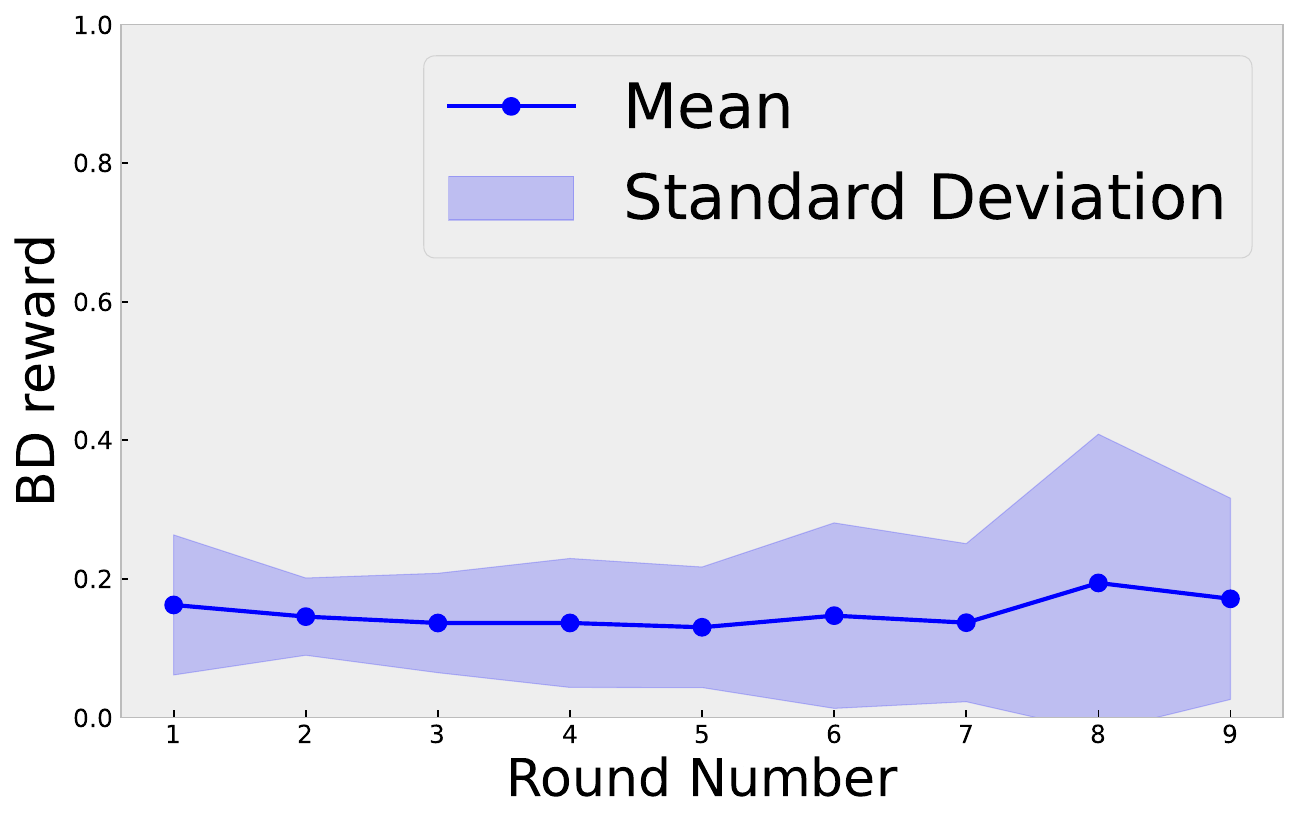}
        \caption{Drafter 3}
        \label{fig:plot3}
    \end{subfigure}
    \hfill
    \begin{subfigure}[b]{0.45\linewidth}
        \includegraphics[width=\linewidth]{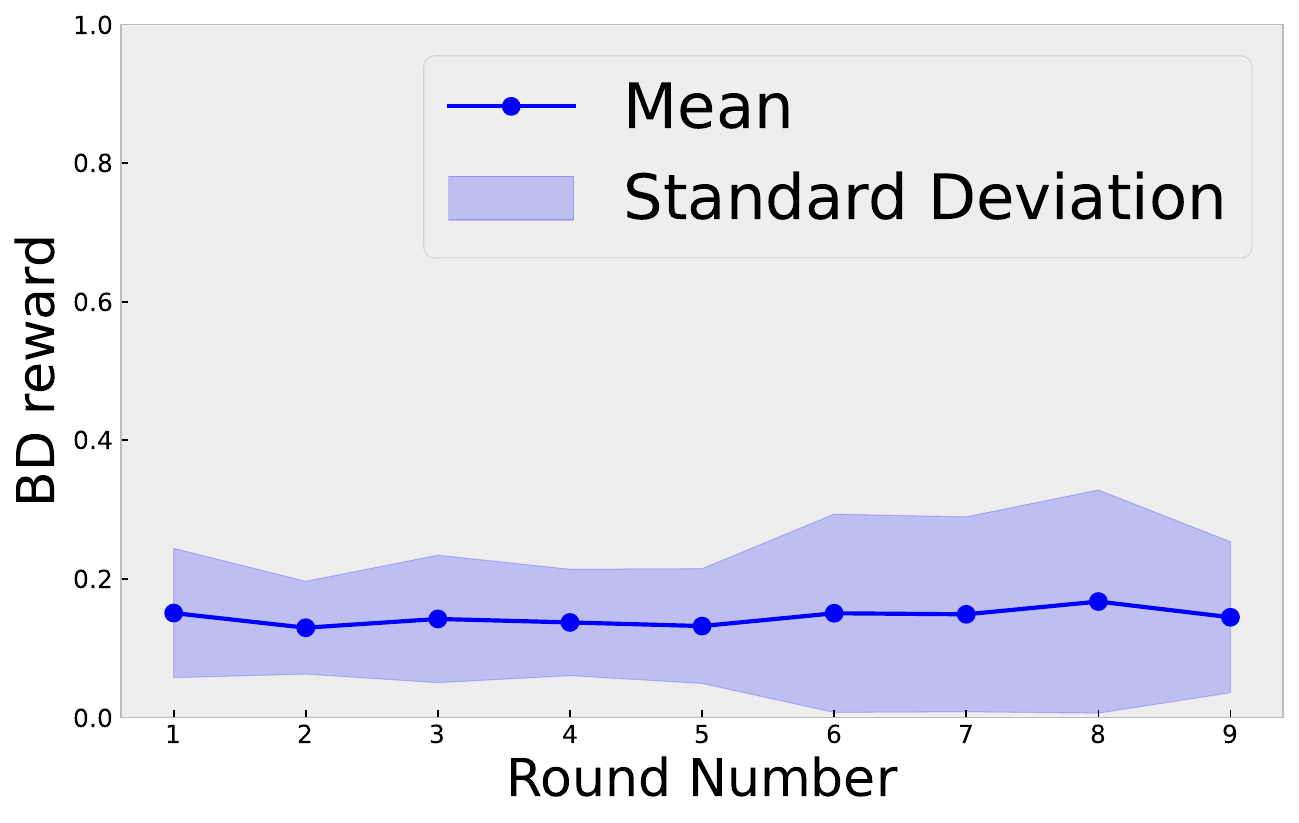}
        \caption{Drafter 4}
        \label{fig:plot4}
    \end{subfigure}
    \begin{subfigure}[b]{0.45\linewidth}
        \includegraphics[width=\linewidth]{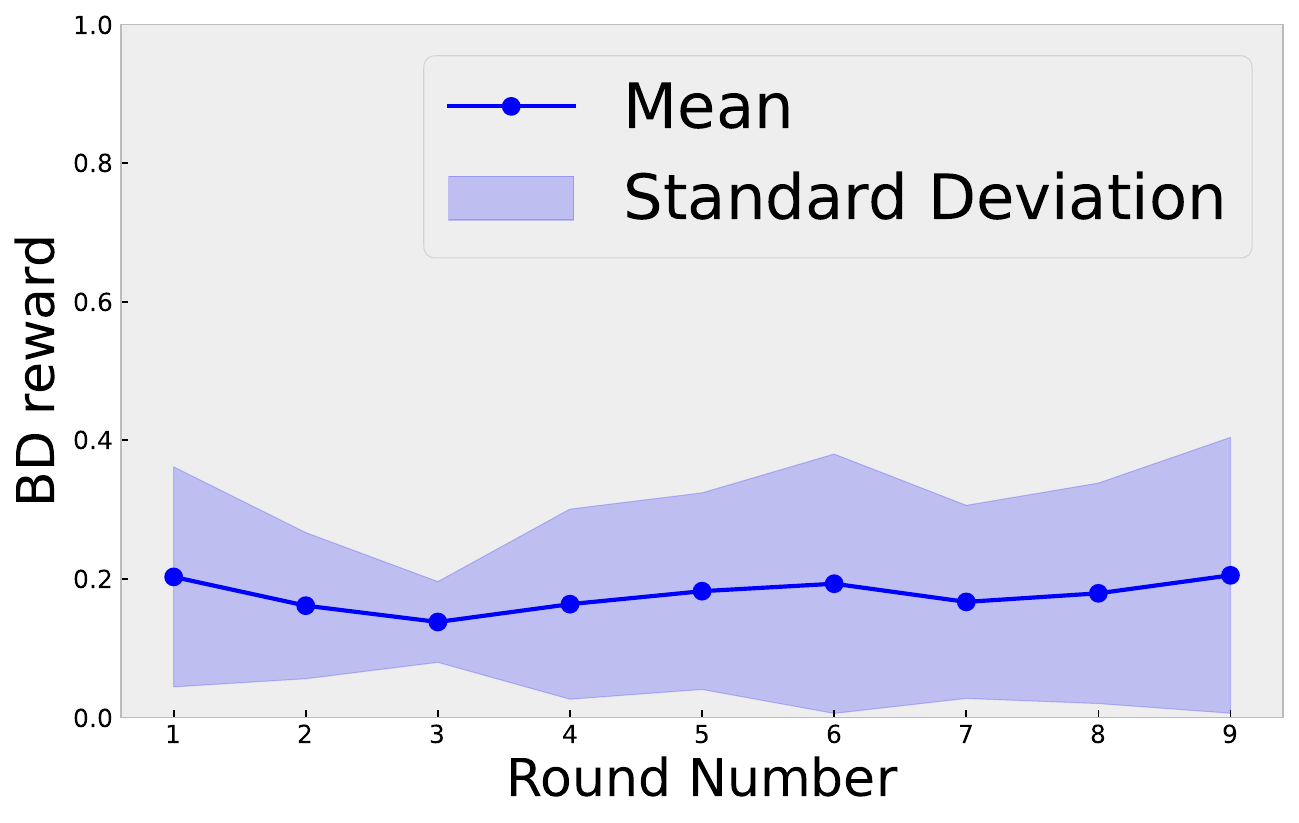}
        \caption{Drafter 5}
        \label{fig:plot5}
    \end{subfigure}

    \caption{Empirical measurement of BD reward statistics along speculation rounds in greedy decoding ($T=0$) scenario.}
    \label{fig:bd-reward-plots}
\end{figure*}

\Autoref{assumption} assumes i.i.d. of acceptance rate $\alpha_{i,t}$ in every instance and this might include the case where $\alpha_i$ can vary for every generation. However, this does not affect the analysis of~\Autoref{theorem:regret_upperbound_SD-UCB}  since our algorithm reset the bandit instance in every new generation.

\paragraph{Temperature sampling}
Note that we make~\Autoref{assumption} for any temperature $T$ which includes greedy decoding scenario with $T=0$. This implies analysis on the regret upper bound in \Autoref{appendix:UCB_general_reward} holds with any temperature $T=0$. This is empirically supported by our experimental results where MetaSD-UCB (\Algautoref{alg:ucb1}) works for both temperature sampling and greedy decoding scenarios. To further support the above claim. we conduct experiments to obtain BD reward statistics with temperature sampling in~\Autoref{fig:bd-reward-plots_temp1}. For the experiment, we set $T=1$ and follow the same experimental setup in~\Autoref{fig:bd-reward-plots}. One can observe the reward is concentrated well while maintaining stationarity along the round numbers. This shows \Autoref{assumption} works with any temperature, further validates our theoretical analysis.

\begin{figure*}[!ht]
    \centering
    \begin{subfigure}[b]{0.45\linewidth}
        \includegraphics[width=\linewidth]{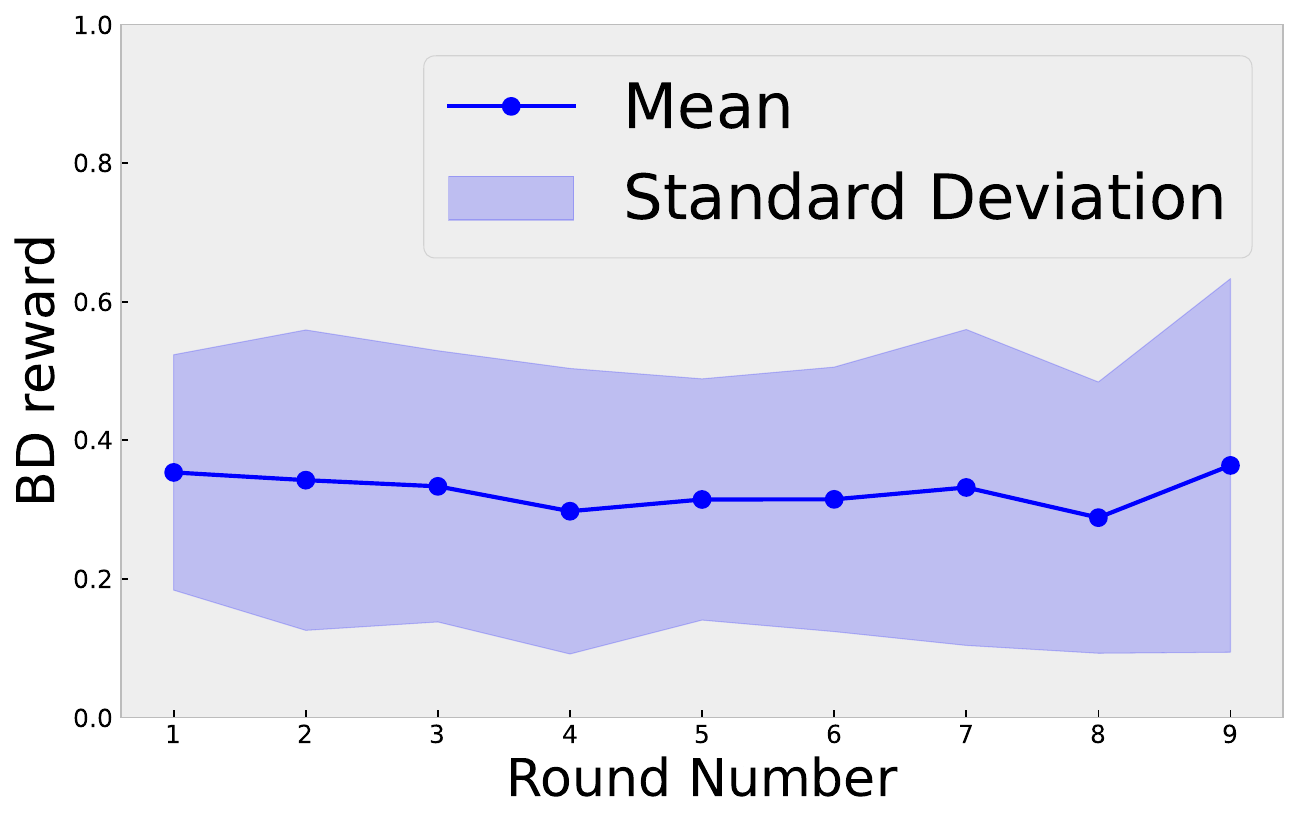}
        \caption{Drafter 1}
        \label{fig:plot1_temp1}
    \end{subfigure}
    \hfill 
    \begin{subfigure}[b]{0.45\linewidth}
        \includegraphics[width=\linewidth]{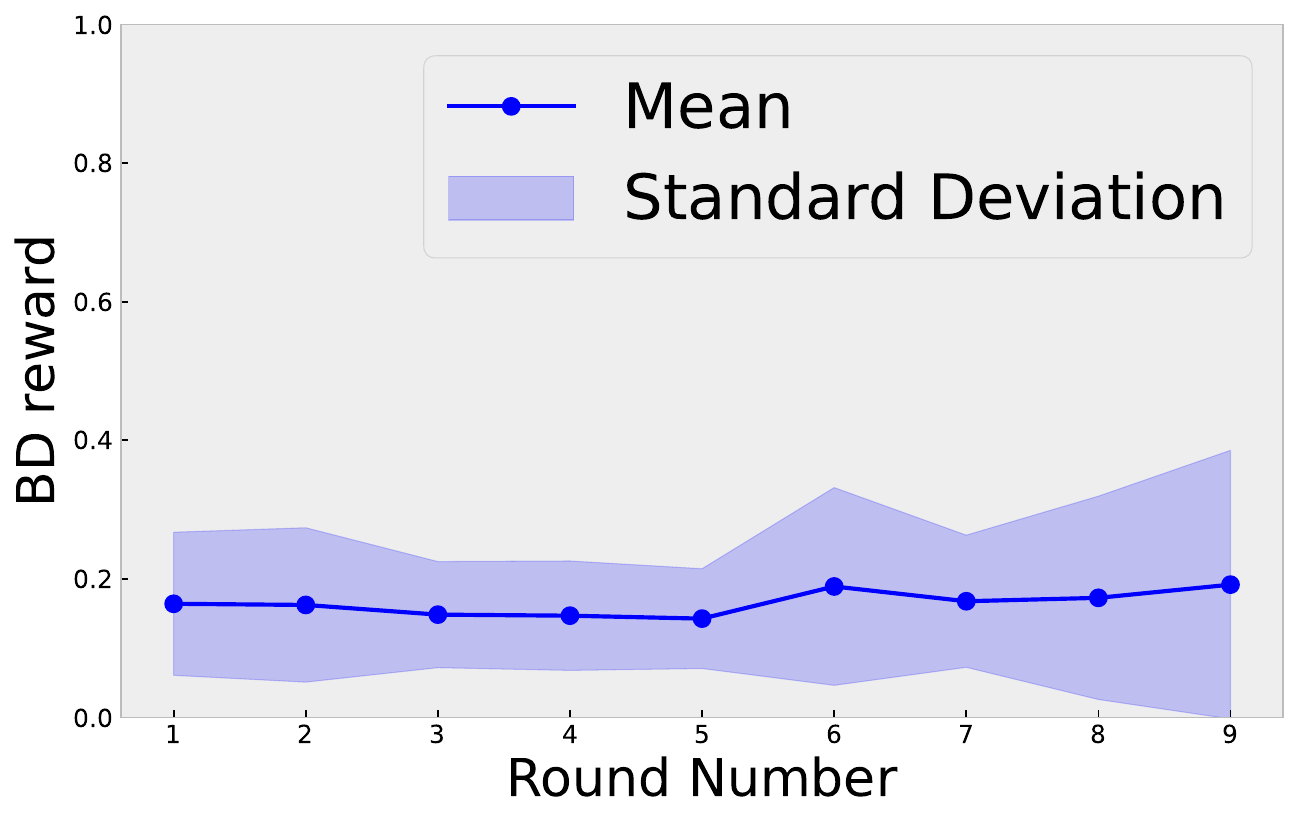}
        \caption{Drafter 2}
        \label{fig:plot2_temp1}
    \end{subfigure}

    \begin{subfigure}[b]{0.45\linewidth}
        \includegraphics[width=\linewidth]{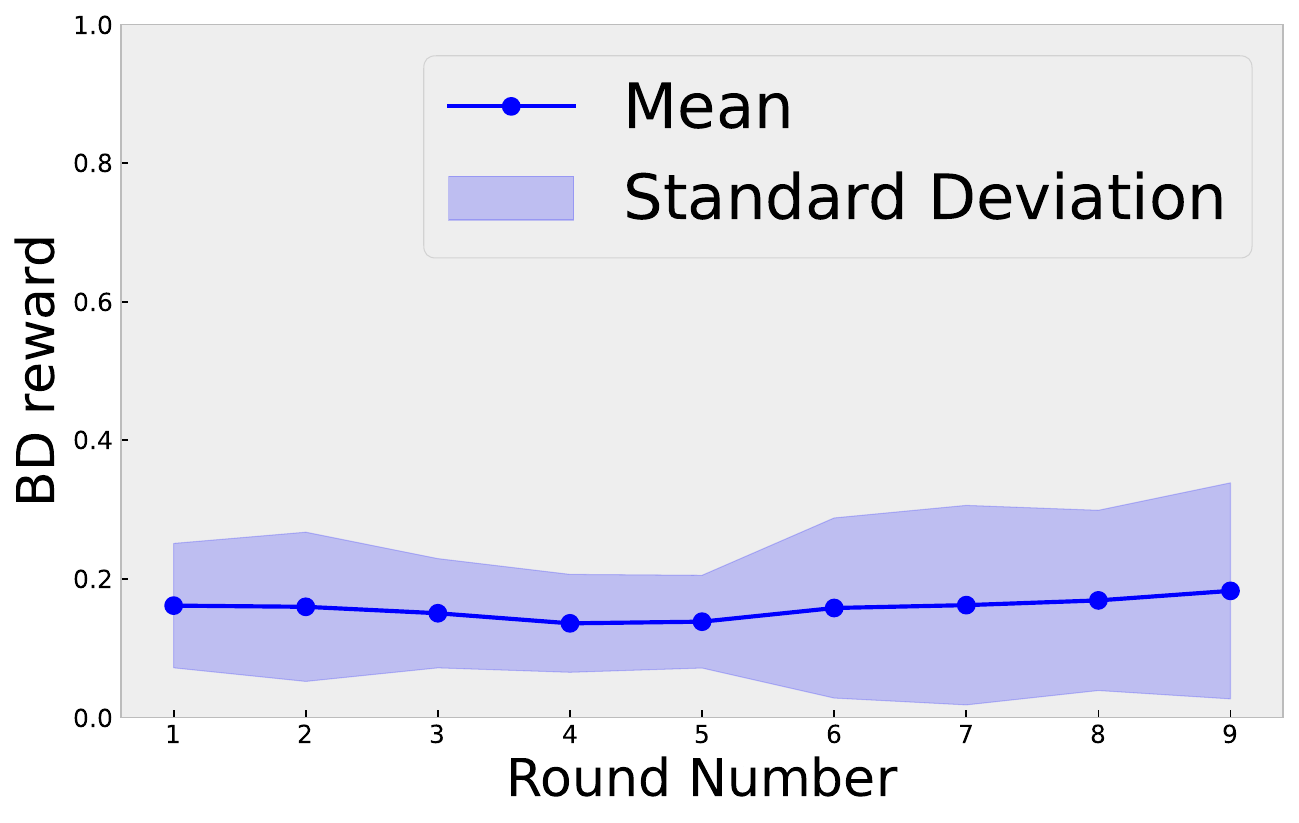}
        \caption{Drafter 3}
        \label{fig:plot3_temp1}
    \end{subfigure}
    \hfill
    \begin{subfigure}[b]{0.45\linewidth}
        \includegraphics[width=\linewidth]{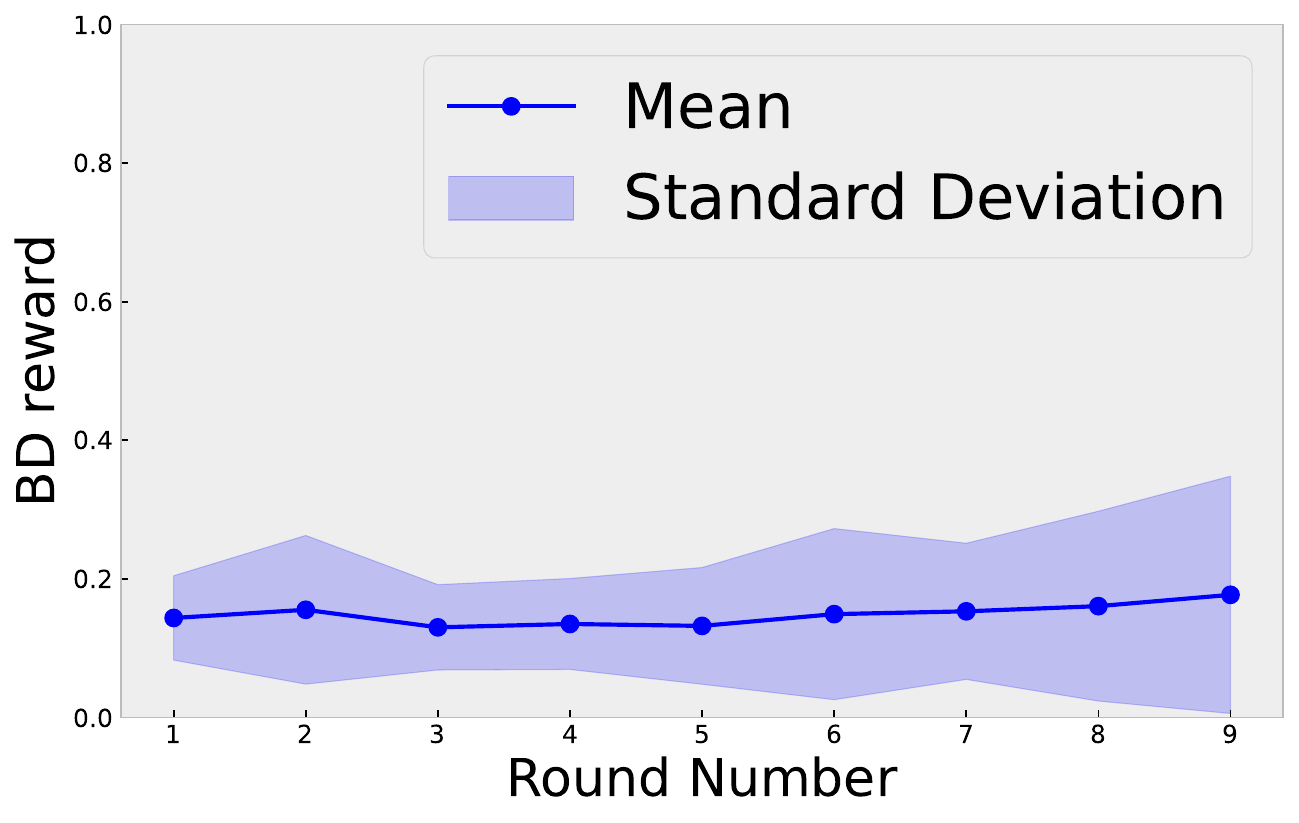}
        \caption{Drafter 4}
        \label{fig:plot4_temp1}
    \end{subfigure}

    \begin{subfigure}[b]{0.45\linewidth}
        \includegraphics[width=\linewidth]{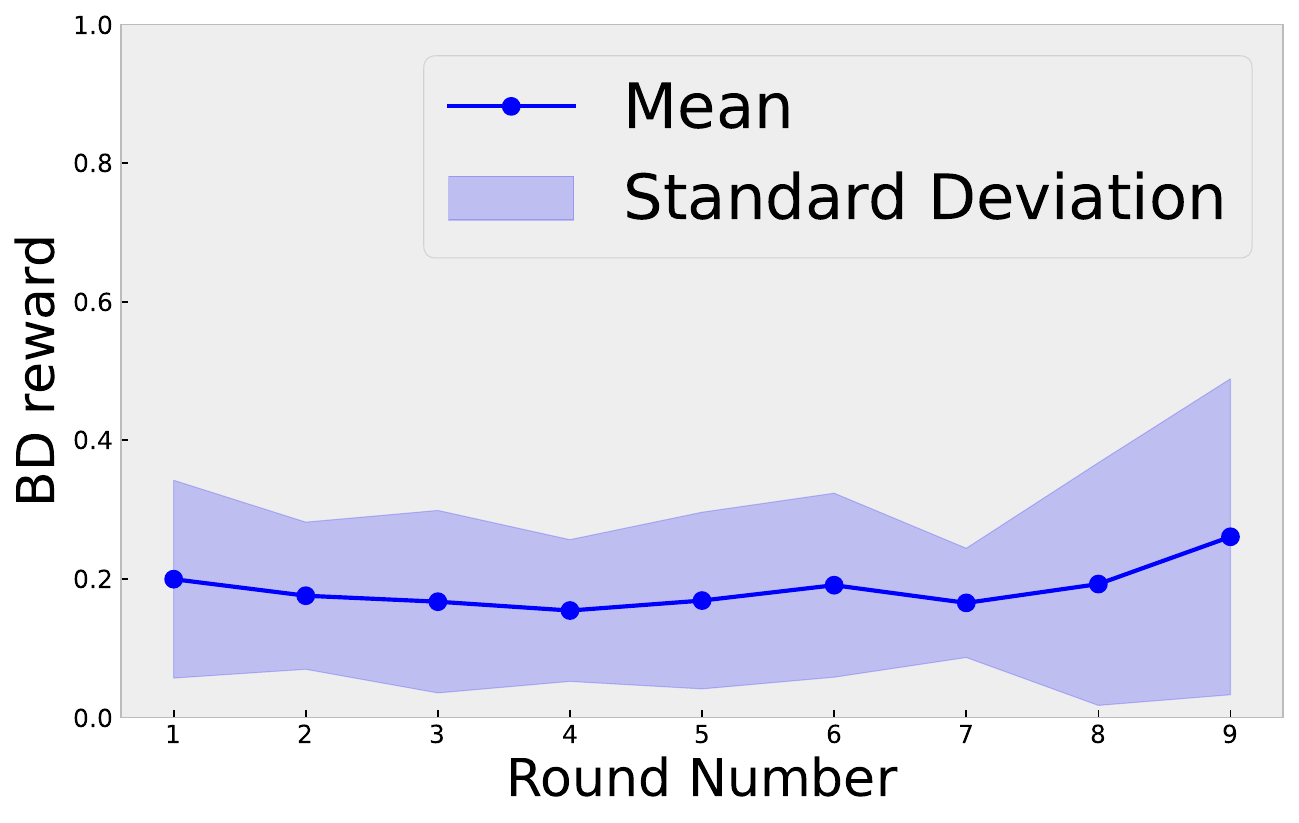}
        \caption{Drafter 5}
        \label{fig:plot5_temp1}
    \end{subfigure}

    \caption{Empirical measurement of BD reward statistics along speculation rounds in temperature sampling $T=1$.}
    \label{fig:bd-reward-plots_temp1}
\end{figure*}

\paragraph{Comparison with {\citep{speculative_decode, yin2024theoretical}}}
\citet{speculative_decode} assume fixed value of $\alpha_i$ when calculating expected number of generated tokens in each round. \Autoref{assumption} is more general than this where in our case, variance of the acceptance rate becomes critical factor to obtain a concentration bound as stated in~\Autoref{lemma:BEreward_stats} and~\Autoref{lemma:BD_stats_asymptotic} while this is impossible when assuming fixed acceptance rate. Note that \citet{yin2024theoretical} analyze SD in more general case where they provide the expected number of total rejected tokens as follows:
\begin{equation}
\begin{aligned}
& \mathbb{E}[N_{rej}] =
\\ &  \sum_{t=1}^{T}\mathbb{E}_{x_{1:t-1}\sim p^t}[d_{TV}(p^t(\cdot|x^{1:t-1}),q_i^t(\cdot|x^{1:t-1})]    
\end{aligned}
\end{equation}
This is general than~\Autoref{assumption} where we assume previous context $x^{1:t-1}$ does not affect the TV distance between target model and a drafter. Relaxing the assumption and considering context-dependent reward distribution will be related to a contextual bandit problem~\citep{li2010contextual} while we leave this as a future work.

\subsection{Randomness of the target sequence length B} 
\label{app:B_randomness}
We consider general scenario where we take all possible instances generated by a target model when using temperature sampling with $T>0$. In this scenario, we can define the expected regret over the probability space induced by the target model. 
To do so, we first provide a formal definition of a target sequence length $B$.
\begin{definition}[Target sequence length B]
\label{def:target_sequence_length}
Target sequence length $B$ is a stopping time which is defined as follows:
\begin{equation}
B=\min{\{t\in\mathbb{N}:x^t=EOS\}},
\end{equation}
where $x^t\sim p^t(\cdot|x^{1:t-1})$ with $p^t$ being a probability distribution from the target model given context $x^{1:t-1}$ and EOS refers to the end of sentence token.
\end{definition}
According to~\Autoref{def:target_sequence_length}, target sequence length is a random variable (a stopping time). With this, one can observe the following lemma holds:
\begin{lemma}
\label{lemma:target_sequence_length_expectation}
For $b\in\mathbb{N}$,
\begin{equation}
\begin{aligned}
& \mathbb{P}(B=b) = \mathbb{E}_{x^{1:b-1}\sim p}\prod_{t=1}^{b-1}(1-p^t(EOS|x^{1:t-1})
\\ & \cdot p^b(EOS|x^{1:b-1})
\end{aligned}
\end{equation}
Where, $p^t(\cdot|x^{1:t-1})$ refers to the conditional probability distribution from a target model for $t$-th token generation when given context $x^{1:t-1}$. Moreover, expectation of a target sequence length becomes:
\begin{equation}
\mathbb{E}[B] = \mathbb{E}_p(B) = \sum_{b=1}^{\infty}b\cdot \mathbb{P}(B=b).
\end{equation}
Here, $\mathbb{E}_p$ denotes the expectation taken over the probability distribution induced by the target model $p$.
\end{lemma}

Then the expected stopping time which includes every instance of token generation realization given a context can be analyzed with the following objective.
\begin{definition}[Expected stopping time regret] 
\begin{equation}
\label{eq:general_stoppingtime_regret}
\textsc{Reg}(\pi,B) =\mathbb{E}_{p,\pi}\left[\tau(\pi,B)]-\mathbb{E}_{p,\pi^{\star}}[\tau(\pi^{\star},B)\right],
\end{equation}
where, $\mathbb{E}_{p}$ denotes the expectation taken over from a probability space induced by the randomness of target model generation and $\mathbb{E}_{\pi}, \mathbb{E}_{\pi^{\star}}$ refers to the expectation taken over from the probability space generated by a bandit policy $\pi$ and the optimal policy $\pi^{\star}$ respectively.
\end{definition}

In order to analyze the expectation of the stopping time regret which includes the randomness of $B$, we need a more stronger assumption than in \Autoref{assumption:main_model_alignment} which is stated as follows.
\begin{assumption}
\label{assumption2}
Denote $\alpha_{i,t}$ as the acceptance rate for $t$-th token generated by $i$-th model. Then, for any instance $x^{1:B}$ generated by the target model, $\alpha_{i,t}$'s are i.i.d. from a distribution $\nu_i$ with expectation $\alpha_i$. In other words, following holds for all drafter  $i\in[K]$.
\begin{equation}
\begin{aligned}
& \alpha_{i,t}=1-d_{TV}\left(p^{t}(\cdot|x^{1:t-1}),q_i^{t}(\cdot|x^{1:t-1})\right) \overset{\text{i.i.d.}}{\sim}  \nu_i, 
\\ & \mathbb{E}[{\alpha_{i,t}}] = \alpha_i.    
\end{aligned}
\end{equation}
Moreover, $\alpha_i$ is independent of $B$ and its conditional expectation over the events with given $B$ is same for every $B$.
\end{assumption}

The above assumption implies acceptance rate for each drafter is i.i.d. from a stationary distribution of a given instance and its mean value is independent of $B$.
Now, with the generalized regret objective and~\Autoref{assumption2}, one can obtain regret upper bound in terms of expectation of total generated tokens.
\begin{theorem}[Expected stopping time regret] Under~\Autoref{assumption2}, following regret bound holds for Meta-UCB with general stopping time regret:
\begin{equation}
\label{eq:regret_upper_bound_generalver}
\begin{aligned}
& \textsc{Reg}(\pi,B) <
\sum_{i\neq i^{\star}}\frac{8}{(N_{max})\Delta(\alpha_i)^2}\cdot
\\ & \left(\ln\left({\mathbb{E}[B]}\right)+{\ln{(\ln(\sum_{i\neq i^{\star}}\frac{1}{\Delta(\alpha_i)^2}))}+C'}\right)+ C. 
\end{aligned}
\end{equation}
Here, $C,C'>0$ are again constants that are independent from $B$ and $\Delta(\alpha_i)$.
\end{theorem}
\begin{proof}
Since drafter selection from the policy $\pi$ is independent from B under~\Autoref{assumption2}, we can decouple~\Autoref{eq:general_stoppingtime_regret} as follows:
\begin{equation}
\label{eq:general_stoppingtime_regret_stochastic_B}
\textsc{Reg}(\pi,B) =\mathbb{E}_B[\mathbb{E}_{\pi}\left[\tau(\pi,B)]-\mathbb{E}_{\pi^{\star}}[\tau(\pi^{\star},B)\right]],
\end{equation}
where first expectation is taken over with respect to a probability distribution of $B$ generated from $p$. Using Jensen's inequality and combining with Theorem 2, we get the result. 
\end{proof}

\subsection{Further analysis on hyper-parameter $\beta$}
\label{appendix:beta_analysis}
Although original UCB-1 algorithm in~\citep{auer2002finite} is based on using fixed value of $\beta=1$, following works~\citep{audibert2009exploration,bubeck2010bandits} show the regret can indeed be dependent on the exploration parameter $\beta$. We provide a general results which includes a hyperparameter $\beta$ in MetaSD-UCB algorithm. In the following, we borrow the analysis of~\citep{bubeck2010bandits} for the general version of Theorem 2 that includes $\beta$.
\begin{theorem}[Regret upper bound containing $\beta$]
\label{theorem:general_upperbound_include_beta}
For $\beta>0.5$ and with \Autoref{assumption}, the regret upper bound in Theorem 2 can be generalized as follows:
\begin{equation}
\begin{aligned}
& \textsc{Reg}(\pi,B) <
\sum_{i\neq i^{\star}}\frac{8\beta^2}{(N_{max})\Delta(\alpha_i)^2}\cdot
\\ & \left(\ln{B}+\ln(\ln(\sum_{i\neq i^{\star}}\frac{1}{\Delta_i^2}))+C'\right)+ C.  
\end{aligned}   
\end{equation}
\end{theorem}
\begin{proof}
The proof is based on modifying~\Autoref{lemma:bounding_suboptimal_nums} to the equation (2.15) in~\citep{bubeck2010bandits} which is stated here for the completeness. 
\begin{equation}
\label{eq:subotpimal_armpick_withbeta}
\begin{aligned}
& \mathbb{E}[n_i(\tau(B))|\tau(B)]\leq 
\\ & \frac{8\beta^2\ln\tau(B)}{\Delta_i^2} + 1 + \frac{4}{\ln(2\beta^2+\frac{1}{2})}\left(\frac{2\beta^2+\frac{1}{2}}{2\beta^2-\frac{1}{2}}\right)^2
\end{aligned}
\end{equation}
Rest of the procedure is same with~\Autoref{theorem:regret_upperbound_SD-UCB}.
\end{proof}
Note that extra $\beta^2$ appears in the regret bound and supports and constant term can arbitrarily blow up when $\beta$ becomes close to the $\frac{1}{2}$ by the right term in~\Autoref{eq:subotpimal_armpick_withbeta}. We refer~\citep{bubeck2010bandits} for further details of the calculations. 

\section{Extended scenarios for the MetaSD framework}\label{app:extensions}
Our MetaSD framework is universal as it can incorporate various bandit algorithms tailored for different scenarios. However,  establishing optimality guarantees for existing algorithms in this framework requires careful analysis or one should look for the different algorithm designs. This is due to two key distinctions in our problem formulation: (i) stochastic stopping time, and (ii) a new regret objective defined in terms of this stopping time (\Autoref{def:stopping_time_regret}).

This section explores two distinct scenarios and introduces possible algorithms for each. First, we address a scenario when switching costs is not negligible anymore. In MetaSD framework, this happens when substantial computational or memory overhead is incurred when changing drafters. Second, we consider non-stationary environment where the characteristics of the context change within a one generation. Finally, we briefly discuss on other possible extensions of our framework.
\subsection{Switching costs}
\label{appendix:switching_costs}
\paragraph{Switching costs for multiple drafters}
In order to use multiple drafters in SD, one need to replace all missing key-value(KV) cache values for the model whenever switching one drafter to another. Reading and writing KV cache is one of the factor which can decrease the inference speed, and we define any decrease of inference speed by changing drafter as the switching cost. Formally, switching cost is defined as $\lambda(a_t,t) = \lambda\left(l(t)-l(\tau_i(t))\right)\cdot\mathbb{I}[a_{t-1}\neq a_{t}]$ where $l(t)$ is number of processed tokens by the target model in round $t$, $\tau_i(t)$ is the latest round where $i$-th drafter is selected before round $t$, $\mathbb{I}$ is an indicator function, and $\lambda$ is a constant. we first define the pseudo regret objective in the presence of switching costs.
\begin{definition}
With bandit policy $\pi$ and the given budget $B$, we define the regret as follows:
\label{def:regret_switching_cost}
\begin{equation}
\label{eq:def_pseudo-regret-switchingcost}
\begin{aligned}
& \textsc{Reg}_{switch}(\pi,B,\lambda) =
\\ & \mathbb{E}[\tau(\pi,B)] - \mathbb{E}[\tau(\pi^{\star},b)] + \sum_{t=2}^{\tau(B)}{\lambda_t\mathbb{P}({a_{t-1}\neq a_{t}}}).
\end{aligned}
\end{equation}
\end{definition}

To minimize the above regret, observe $\lambda(\pi,B) = 
\lambda\sum_{t=1}^{\tau(B)}{\lambda(a_t,t)} = \lambda\sum_{i=1}^{K}{B_i}$, where $B_i$'s are total number of tokens generated by the $i$-th drafter after the final round. Intuitively, this implies that total cost decreases when employing elimination-type of algorithms~\citep{Successive_rejects,Sequential_Halving}, which successively eliminate sub-optimal drafters and exclude those drafters from future selection. Consequently, the total regret $\textsc{REG}_{switch}(B,\lambda)$ can be reduced from early elimination of poor-performed drafters. However, regret can still increase if the best drafter is mistakenly eliminated early on. Therefore, it is essential to strike a balance between elimination-based algorithms and standard MAB algorithms. For this, we design a new algorithm \textbf{Pure Exploration-Then-Commit} (PETC) in~\Autoref{alg:Fast Arm Elimination} which effectively balances these two approaches.

PETC (\Autoref{alg:Fast Arm Elimination}) divides the MetaSD into two phases. In the first phase $l<B_0$, the algorithm tries to eliminate sub-optimal drafters as quickly as possible. In the bandit literature, this is related to the pure exploration (or best arm identification) problem~\citep{lattimore2020bandit} and we select using SH~\Autoref{alg:sequential_halving} for our analysis. After the exploration period for estimating the best drafter, the algorithm exclusively selects this drafter for the remaining rounds.
\begin{algorithm}[t]
\caption{Pure exploration-then-commit (PETC)}
\label{alg:Fast Arm Elimination}
\begin{algorithmic}[1]
\INPUT Drafter pool $[K]$, initial prompt sequence $x^{1:l}$, target sequence length $B$, exploration rounds $B_0$.
\FOR {$l=1,2,...,B_0$}
    \STATE Run SH algorithm with budget $B_0$ (in \Autoref{alg:sequential_halving})
\ENDFOR
\STATE $\hat{i}_{\star}$ be the survived index.
\WHILE {$l < B$}
    \STATE SD with a single drafter $\hat{i}_{\star}$.
\ENDWHILE
\end{algorithmic}
\end{algorithm}

Now, we provide how to find the optimal $B_0$ which by the following theorem:

\begin{theorem}[Regret upper bound on PETC]
\label{theorem:fast_arm_elimination_regret_bound}
By choosing $B_0=c\cdot\ln{B}$ for some constant $c>0$ and using~\Autoref{alg:sequential_halving} for the pure exploration in the for the first phase in~\Autoref{alg:Fast Arm Elimination}, $\textsc{REG}_{switch}(\pi,B,\lambda)\leq O(\ln{B})$ holds.    
\end{theorem}
\begin{proof}
First, we can decompose the regret as:
\begin{equation*}
\begin{aligned}
& \textsc{Reg}_{switch}(\pi,B,\lambda) = 
\\ & \sum_{t=1}^{\tau(B_0)}\textsc{Reg}(\pi,t) + \sum_{t=\tau(B_0)+1}^{\tau(B)}\textsc{Reg}(\pi,t) + S_T,
\end{aligned}
\end{equation*} 
where $\textsc{Reg}(\pi,t)$ denotes original regret objective~\Autoref{eq:Regret_Define_stoppingtime} for one round $t$ and $S_T$ denotes the total switching cost. First term can be bounded by the stopping time of selecting the worst drafter every round until $B_0$ which can be bounded by  $\tau(B_0)={O}(\ln{B})$ according to~\Autoref{lemma:single_process_stopping_time}. To bound the second term, we borrow Theorem 4.1 in~\cite{Sequential_Halving}, where they prove the probability of Sequential Halving algorithm to select the suboptimal arm after $B_0$ round can be bounded by $3\log_2{K}\cdot\exp(-\frac{B_0}{8H_2\log_2{K}})$, where $H_2:=\max_{i}\frac{i}{\Delta_i^2}$. Then we have
\begin{equation*}
\begin{aligned}
& \sum_{t=\tau(B_0)+1}^{\tau(B)}\textsc{Reg}(\pi,t) 
\leq \tau(\pi^{i_w},B)\cdot3\log_2{K}
\\ & \cdot\exp(-\frac{B_0}{8H_2\log_2{K}})  = O(\ln{B}), 
\end{aligned}
\end{equation*}
where $i_w$ denotes the worst drafter, $\tau(\pi^{i_w},B)$ denotes the stopping time for generating $B$ tokens using only the worst drafter. The last term is bounded by $\lambda B_0 = O(\ln{B})$ and this concludes the proof.
\end{proof}
Here, we can improve constant term in regret upper bound in~\Autoref{theorem:fast_arm_elimination_regret_bound} by controlling $c$ according to the switching cost $\lambda$ and given budget $B$ or we may use more advanced proof techniques in the best arm identification literature such as in~\cite{zhao2023revisiting}. We leave these as a future work.

\begin{algorithm}[t]
\DontPrintSemicolon
\caption{Sequential Halving (SH)~\citep{Sequential_Halving}}
\label{alg:sequential_halving}
\begin{algorithmic}[1]
\INPUT Total budget $T$, drafter pool $[K]$
\STATE \textbf{Initialize} $S_0 \gets [K]$
\FOR{$t = 0, 1, \ldots, \lfloor \log_2(K) \rfloor - 1$}
    \STATE Pull each drafter in $S_t$ for $n_t = \left\lfloor \frac{T}{|S_t| \lfloor \log_2(K) \rfloor} \right\rfloor$ additional times
    \STATE $R_{t}(i) \gets \sum_{j=1}^{n_t} r_{i,j}$ for $i\in S_t$
    \STATE Let $\sigma_t$ be a bijection on $S_k$ such that $R_{t}(\sigma_t(1)) \leq R_{t}(\sigma_t(2)) \leq \ldots \leq R_{t}(\sigma_t(|S_t|))$
    \STATE $S_{k+1} \gets [i \in S_k | R_{t}(\sigma_t(i)) \leq R_{t}(\sigma_t(\lceil |S_k|/2 \rceil))] $
\ENDFOR
\OUTPUT Singleton element of $S_{\lfloor \log_2(K) \rfloor}$
\end{algorithmic}
\end{algorithm}

\subsection{Non-stationary environment}
\label{appendix:non-stationary_environment}
In real-world scenarios, the reward distribution for each drafter may evolve over time and past information becomes less relevant for decision-making. This phenomenon, referred to as non-stationarity, challenges traditional MAB algorithms that operate under the assumption of stationary reward distributions. In SD, non-stationarity can stem from various factors. For example, during a long-form text generation task, the optimal drafter may change as the topic or style of the text evolves. Consider the prompt: `Please summarize and reason about the following article on climate change...'. Initially, a drafter specialized in summarization might be most effective. However, as the generation progresses towards the reasoning part, a drafter trained on logical reasoning tasks could become more suitable.

\paragraph{Non-stationary MetaSD}
Standard analyses of non-stationary bandits~\citep{auer2002nonstochastic,kocsis2006discounted,garivier2016optimal} often define $L$ to quantify the number of times the reward distributions change over $T$ rounds. Another line of work~\citep{slivkins2008adapting,besbes2014stochastic} quantifies the non-stationarity using $V$, the total variation of the means. In both cases, the regret (which is often called as dynamic regret) is defined as the cumulative expected difference between the rewards of the optimal arm and the selected arm at each round.
\begin{equation}
\label{eq:non_stationary_regret_original}
\textsc{Reg}(\pi,B,L) = \sum_{t=1}^{\tau(B)}(\max_{i\in[K]}\mu_{i,t}-\mathbb{E}[\mu_{a_t,t}])
\end{equation}
where, as before, $B$ is the number of total tokens we have to generate, $\mu_{i,t}$ is the mean reward of choosing drafter $i$ in $t$-th round, and $\tau(B)$ is the total round. However, the regret upper bound on~\Autoref{eq:non_stationary_regret_original} does not always guarantee the performance of the SD as we discussed in~\Autoref{subsec:algorithm 3.1}. Instead, we can use our original regret objective using stopping time in~\Autoref{def:stopping_time_regret} without any modification.

Here, we introduce two types of algorithms within our MetaSD framework: Discounted-UCB (D-UCB) algorithm~\citep{kocsis2006discounted} (\Autoref{algorithm:D-UCB}) and Sliding-window UCB~\citep{garivier2011upper} (\Autoref{algorithm:Sliding_window_UCB}). Discounted UCB-SD estimates mean reward by computing the mean of discounted cumulative rewards as shown in the line 9 of~\Autoref{algorithm:D-UCB}. By assigning less weight to the past observations, the algorithm finds a balance between  accumulating knowledge and adapting to the changing environment. Similarly, sliding-window UCB utilizes a fixed-length window to calculate mean reward as demonstrated in the line 9-10 of \Autoref{algorithm:Sliding_window_UCB}. By focusing only on recent information, it is also expected to achieve a balance with careful choose of the window size $\tau$~\cite{garivier2011upper}.

\begin{algorithm}[t]
\caption{Discounted UCB in MetaSD}
\label{algorithm:D-UCB}
\begin{algorithmic}[1]
\INPUT Drafter pool $[K]$, initial prompt sequence $x^{1:l}$, target sequence length $B$, exploration strength $\beta$, decaying parameter $\gamma$.
\STATE $t \leftarrow 0$ \\
\textcolor{darkgray}{/* Phase 1: Meta-draft each drafter in $[K]$ once and do one round of speculative decoding. */} \\
\FOR{$i \in [K]$}
    \STATE Do one round of SD with drafter $i$ and obtain $N_{acc}(i,t)$, $r_{i,t}$ (by \Autoref{eq:BDreward_definition})
    \STATE $\hat{\mu}_{i,t}, n_i, l,  t\leftarrow r_{i,t}, 1, l + N_{acc}(i,t) + 1,  t+1$ \\
\ENDFOR \\
\textcolor{darkgray}{/* Phase 2: Meta-draft the draft following the UCB bandit until target sequence length $B$ */} \\
\WHILE{$ l < B$}  
    \STATE $a_t \leftarrow \arg\max_{i \in [K]} \hat{\mu}_{i,t} + \beta \sqrt{\frac{2 \ln t}{n_i}}$ 

    \STATE Do one round of SD with drafter $a_t$ and obtain $N_{acc}(a_t,t)$, $r_{a_t,t}$ (by \Autoref{eq:BDreward_definition})
    \STATE $\hat{\mu}_{a_t,t}=\frac{1}{n_{a_t}}\sum_{s=1}^{t}\gamma^{t-s}r_{a_s,s}\mathbb{I}[a_s=a_t]$
    \STATE \small{$n_{a_t}, l, t \leftarrow  n_{a_t}+1 , l + N_{acc}(a_t,t) + 1,  t+1$}
\ENDWHILE
\end{algorithmic}
\end{algorithm}

\begin{algorithm}[t!]
\caption{MetaSD-EXP3~\citep{auer2002nonstochastic}}
\label{alg:EXP3}
\begin{algorithmic}[1]
\INPUT Drafter pool $[K]$, initial prompt sequence $x^{1:l}$, target sequence length $B$, $\gamma \in (0, 1]$
\STATE $t\leftarrow 0$, $w_t(i) \leftarrow 1$ for $i = 1, \dots, K$
\WHILE{$l<B$}
    \STATE 
    $ p_t(i) = (1 - \gamma) \frac{w_t(i)}{\sum_{i=1}^K w_t(i)} + \frac{\gamma}{K}$ 
    \\ for $ \quad i = 1, \dots, K$. 
    \STATE Draw $a_t$ randomly according to the probabilities $p_t(1), \dots, p_t(K)$.
    \STATE Do one round of SD with drafter $a_t$ and obtain $N_{acc}(a_t,t)$, $r_{a_t,t}$ (by \Autoref{eq:BDreward_definition})
    \FOR{$j = 1, \dots, K$}
        \STATE \[ \hat{r}_{j,t} = \begin{cases}
        r_{j,t} / p_t(j) & \text{if } j = a_t \\
        0 & \text{otherwise},
        \end{cases} \]
        \[ w_{t+1}(j) = w_t(j) \exp \left( \frac{\gamma\cdot \hat{r}_{j,t}}{K} \right)\] 
    \ENDFOR
    \STATE $l, t \leftarrow l + N_{acc}(a_t,t) + 1,  t+1$
\ENDWHILE
\end{algorithmic}
\end{algorithm}

One interesting point is that in the non-stationary MetaSD problem, the definition of non-stationarity $L$ does not fit naturally into our problem. The reason behind this is that under non-stationary context generations, number of distribution changes happen at the token level, not the round level. This can disrupt existing regret analysis because a single round might involve multiple reward distribution changes (e.g., one round of speculative decoding could have two changing points). Whether above algorithms maintain optimal regret bounds in our regret definition in this non-stationary setting presents an interesting direction for future theoretical analysis.

\paragraph{Experiments}

To further demonstrate the effectiveness of the MAB approach, we conduct experiments on a non-stationary translation task, where each query required translating two different languages (French and Chinese) into English. This presents a challenging task, as even well-trained router-based algorithms struggle to capture the shifting context during generation. The result is provided in \Autoref{tab:nonstationary_translation}.

\begin{table*}[ht!]
\centering
\caption{Results of the non-stationary translation task.}
\label{tab:nonstationary_translation}
\footnotesize 
\begin{tabular}{@{}lccccccc@{}}
\toprule
& Drafter1 & Drafter2 & Drafter3 & Drafter4 & Drafter5 & \multicolumn{1}{c}{Upper Bound} & MetaSpS-UCB \\
\midrule
Block Efficiency & 1.668 & 1.722 & 1.485 & 1.759 & 1.803 & 1.803 & 1.951 \\
Speedup Ratio & 1.429 & 1.492 & 1.289 & 1.539 & 1.581 & 1.581 & 1.722 \\
\bottomrule
\end{tabular}
\end{table*}

The result clearly demonstrate the effectiveness of the MAB approach, where MetaSD-UCB with $\alpha=0.1$ achieves a speedup ratio of 1.722, which even exceeds the performance of the optimal drafter with a speedup ratio 1.581. This improvement arises because MAB dynamically adapts to the best drafter in the current environment through a combination of exploration and exploitation.

For example, in a multilingual scenario, where the task is "Translate French and German to English, respectively," classification-based routing would be limited by the performance of the best single specialization (i.e., the speedup ratio of the optimal drafter as its upper bound). In contrast, MetaSD-UCB surpasses this upper bound by effectively adapting its policy dynamically, demonstrating a unique advantage over classification-based methods in non-stationary environments.

The experiments are performed using a single RTX 3090 GPU, and we follow the same experimental setup of MetaSpS-UCB in~\Autoref{sec4:exp}. In \Autoref{tab:nonstationary_translation}, Upper Bound refers to the performance achieved by an ideal router-based approach with oracle-provided labels.

\begin{algorithm*}[ht!]
\caption{Sliding-window UCB in MetaSD}
\label{algorithm:Sliding_window_UCB}
\begin{algorithmic}[1]
\INPUT Drafter pool $[K]$, initial prompt sequence $x^{1:l}$, target sequence length $B$, \\ exploration parameter $\beta$, window size $\tau$.
\STATE $t \leftarrow 0$ \\
\textcolor{darkgray}{/* Phase 1: Meta-draft each drafter in $[K]$ once and do one round of speculative decoding. */} \\
\FOR{$i \in [K]$}
    \STATE Do one round of SD with drafter $i$ and obtain $N_{acc}(i,t)$, $r_{i,t}$ (by \Autoref{eq:BDreward_definition})
    \STATE $\hat{\mu}_{i,t}, n_i, l,  t\leftarrow r_{i,t}, 1, l + N_{acc}(i,t) + 1,  t+1$ \\
\ENDFOR \\
\textcolor{darkgray}{/* Phase 2: Meta-draft the draft following the UCB bandit until target sequence length $B$ */} \\
\WHILE{$ l < B$}  
    \STATE $a_t \leftarrow \arg\max_{i \in [K]} \hat{\mu}_{i,t} + \beta \sqrt{\frac{2 \ln t}{n_i}}$ 
    \STATE Do one round of SD with drafter $a_t$ and obtain $N_{acc}(a_t,t)$, $r_{a_t,t}$ (by \Autoref{eq:BDreward_definition})
    \STATE $\hat{\mu}_{i,t}\leftarrow \frac{1}{n_i(t)}\sum_{s=t-\tau+1}^{t}r_{a_s,s}\mathbb{I}[a_s=i]$  $\forall i\in[K]$
    \STATE $n_i(t) \leftarrow \sum_{s=t-\tau+1}^{t}\mathbb{I}[a_s=i]$ $\forall i\in[K]$
    \STATE $l, t \leftarrow l + N_{acc}(a_t,t) + 1,  t+1$
\ENDWHILE
\end{algorithmic}
\end{algorithm*}
\subsection{Other possible scenarios}
\label{appendix:other_extensions}
\paragraph{Adversarial environment}
EXP3~\citep{auer2002nonstochastic} is designed to handle adversarial changes of reward distributions by continuously updating its estimates of the arm rewards and adjusting its exploration strategy accordingly. It achieves this by maintaining a probability distribution over the arms and exponentially weighting the rewards based on their recent performance. By incorporating EXP3 into our framework (\Autoref{alg:EXP3}), we can enable the system to adapt to evolving reward distributions and dynamically select the optimal drafter even in adversarial environments. We utilize this algorithm as a baseline in our experiments.

\section{Further discussion}\label{sec:discuss}

\subsection{Is scaling up drafter size always better?}

While increasing the drafter size might seem like a straightforward path to improved performance, it can be less efficient than our MetaSD approach, especially considering memory bandwidth constraints. Larger models demand more memory for storing weights and activations, increasing data movement between memory and processing units. This can become a bottleneck, particularly in high-performance computing where memory bandwidth is often a limiting factor. It is also discussed in \cite{yi2024towards} in SD scenarios. Moreover, this phenomenon is well-illustrated by the roofline model, which highlights the trade-off between computational intensity and memory bandwidth \citep{medusa}. As model size increases, computational intensity might improve, but the memory bandwidth demands can quickly limit overall speedup.

In contrast, MetaSD utilizes multiple smaller drafters with lower individual memory requirements. By efficiently switching between these drafters, MetaSD can achieve comparable or superior performance to a single large drafter while mitigating the memory bandwidth bottleneck. This is because, despite having multiple drafters, MetaSD only utilizes one drafter for computation at any given time. Thus, the memory bandwidth requirement does not scale with the combined size of all drafters, but rather with the size of the individual drafter being used.  Provided sufficient GPU DRAM, this approach does not have any bottleneck compared to the single drafter SD. Furthermore, MetaSD offers the flexibility to incorporate diverse drafters with specialized capabilities. This specialization can be more effective than simply increasing the size of a single general-purpose drafter, particularly for tasks demanding domain-specific knowledge. 

\subsection{Out-of-domain (OOD)}
We evaluate MetaSD on Alpaca-Finance \citep{financealpaca} and RAG \citep{xia2024unlocking}, which fall outside the training domains of specialized drafters. As shown in \Autoref{tab:ood_results} (in \Autoref{app:ood}), MetaSD outperforms OFA and most specialized drafters, demonstrating its ability to generalize without relying on predefined domain similarities. Unlike similarity-based selection, which incurs high inference costs for long inputs and is prone to misclassification, MetaSD dynamically adapts at the token level, ensuring efficient and robust routing. In heterogeneous drafter settings, where training descriptions are incomplete or unavailable, MetaSD remains effective, highlighting its scalability and adaptability in real-world scenarios.

\subsection{Computational overhead analysis}

\paragraph{Training overhead} While specialization may require additional training efforts compared to an OFA (One-size-Fits-All) drafter, we emphasize that our approach is designed to handle real-world scenarios where heterogeneous drafters already exist in public repositories. MetaSD focuses on optimizing the utilization of such heterogeneous drafters, dynamically selecting the most suitable drafter during inference. This shifts the problem from retraining models to developing an effective strategy for utilizing pre-existing resources. Therefore, while training specialized drafters may involve additional costs in certain cases, the broader applicability and versatility of MetaSD provide substantial practical value. Additionally, the cost of training drafters is a general challenge shared across the speculative decoding research domain, not limited to our work.

\paragraph{Inference memory-bandwidth efficiency}  
The inference memory-bandwidth efficiency of MetaSD remains comparable to single-drafter methods. Although MetaSD employs multiple drafters, the additional memory requirements are minimal. Specifically, MetaSD increases DRAM usage by only 2 GB (from 17 GB to 19 GB), as the drafters’ weights are preloaded into DRAM. However, this does not affect VRAM bandwidth, as only the active drafter interacts with VRAM during inference. As a result, the VRAM bandwidth demands remain identical to those of single-drafter methods. This efficient memory management ensures that MetaSD maintains competitive performance without introducing significant overhead.

By ensuring that only the active drafter interacts with the VRAM, MetaSD maintains parity with single-drafter approaches in terms of VRAM bandwidth demands.

\paragraph{Serving complexity}
Using multiple drafters in MetaSD does not inherently increase serving complexity. Modern distributed systems already employ model parallelism techniques to allocate workloads across multiple GPUs effectively. In MetaSD, drafters are evenly distributed across GPUs, with each GPU independently handling its assigned drafter without added coordination costs. This design ensures the following:  
\begin{itemize}
    \item Load balancing: Drafters are distributed across GPUs based on their assigned tasks, maintaining equivalent complexity to single-drafter systems.
    \item Minimal communication overhead: MetaSD requires no additional inter-GPU communication beyond standard model parallelism setups.  
\end{itemize}

\paragraph{Justification of overhead}
The modest increase in DRAM memory usage (+2 GB) and marginal training cost for specialized drafters is justified by the significant performance gains achieved through adaptive optimization. MetaSD dynamically selects the most suitable drafter for each task, consistently outperforming single-drafter methods across diverse scenarios, as highlighted in our experimental results. Furthermore, MetaSD addresses an important real-world challenge: effectively utilizing publicly available, pre-trained heterogeneous drafters. By providing a generalizable strategy for optimizing these resources, MetaSD adds practical value beyond specialized retraining, supporting diverse and evolving task requirements.

\subsection{Regret upper bound for MetaSD-UCB}

\Autoref{theorem:regret_upperbound_SD-UCB} provides a regret upper bound for MetaSD-UCB, demonstrating that the number of rounds required to identify the optimal drafter is inversely proportional to the predefined draft length $N_{max}$. This aligns with the intuition that longer drafts provide more information about the relative performance of each drafter, leading to faster convergence towards the optimal choice. The logarithmic dependence on the target sequence length $B$ further highlights the efficiency of MetaSD-UCB in minimizing regret. These theoretical guarantees are supported by our empirical observations, where MetaSD-UCB consistently demonstrates strong performance and rapid convergence towards the best-performing drafter.

\end{document}